\documentclass[preprint,3p,times]{elsarticle}
\usepackage{amsmath}
\usepackage{mathrsfs}
\usepackage{tikz}
\usepackage{hyperref}
\usepackage{listings}
\usepackage{multirow}
\usepackage{amssymb}
\usepackage{amsthm}
\usepackage{epstopdf}
\usepackage{algorithm}
\usepackage{algpseudocode}
\usepackage{subfigure}
\usepackage{tabularx,booktabs}
\usepackage{caption}
\usepackage{extarrows}
\usepackage{comment}
\usepackage{xcolor}

\biboptions{sort&compress}
\theoremstyle{plain}
  \newtheorem{theorem}{Theorem}[section]
  \newtheorem{lemma}{Lemma}[section]
  \newtheorem{proposition}{Proposition}[section]
  \newtheorem{corollary}{Corollary}[section]
\theoremstyle{definition}
  \newtheorem{definition}{Definition}[section]
  
  \newtheorem{example}{Example}[section]
  \newtheorem{remark}{Remark}[section]

\makeatletter
\newif\if@borderstar
\def\bordermatrix{\@ifnextchar*{%
\@borderstartrue\@bordermatrix@i}{\@borderstarfalse\@bordermatrix@i*}%
}
\def\@bordermatrix@i*{\@ifnextchar[{\@bordermatrix@ii}{\@bordermatrix@ii[()]}}
\def\@bordermatrix@ii[#1]#2{%
\begingroup
\m@th\@tempdima8.75\p@\setbox\z@\vbox{%
\def\cr{\crcr\noalign{\kern 2\p@\global\let\cr\endline }}%
\ialign {$##$\hfil\kern 2\p@\kern\@tempdima &\thinspace %
\hfil $##$\hfil &&\quad\hfil $##$\hfil\crcr\omit\strut %
\hfil\crcr\noalign{\kern -\baselineskip}#2\crcr\omit %
\strut\cr}}%
\setbox\tw@\vbox{\unvcopy\z@\global\setbox\@ne\lastbox}%
\setbox\tw@\hbox{\unhbox\@ne\unskip\global\setbox\@ne\lastbox}%
\setbox\tw@\hbox{%
$\kern\wd\@ne\kern -\@tempdima\left\@firstoftwo#1%
\if@borderstar\kern2pt\else\kern -\wd\@ne\fi%
\global\setbox\@ne\vbox{\box\@ne\if@borderstar\else\kern 2\p@\fi}%
\vcenter{\if@borderstar\else\kern -\ht\@ne\fi%
\unvbox\z@\kern-\if@borderstar2\fi\baselineskip}%
\if@borderstar\kern-2\@tempdima\kern2\p@\else\,\fi\right\@secondoftwo#1$%
}\null \;\vbox{\kern\ht\@ne\box\tw@}%
\endgroup
}
\makeatother
\allowdisplaybreaks
\begin{document}
\begin{frontmatter}

\title{On three types of $L$-fuzzy $\beta$-covering-based rough sets}
\author{Wei Li\fnref{label1}}
\ead{lxylw@nwafu.edu.cn}
\author{Bin Yang\fnref{label1}\corref{cor1}}
\ead{binyang0906@nwsuaf.edu.cn}
\author{Junsheng Qiao\fnref{label2}\corref{cor1}}
\ead{jsqiao@nwnu.edu.cn}
 \address[label1]{College of Science, Northwest A \& F University, Yangling, Shaanxi 712100, PR China}
 \address[label2]{College of Mathematics and Statistics, Northwest Normal University, Lanzhou 730070, PR China}
\cortext[cor1]{Corresponding author.}
\begin{abstract}
    In this paper, we mainly construct three types of $L$-fuzzy $\beta$-covering-based rough set models and study the axiom sets,  matrix representations and interdependency of these three pairs of $L$-fuzzy $\beta$-covering-based rough approximation operators.
    Firstly, we propose three pairs of $L$-fuzzy $\beta$-covering-based rough approximation operators by introducing the concepts such as $\beta$-degree of intersection and $\beta$-subsethood degree, which are generalizations of degree of intersection and subsethood degree, respectively. And then, the axiom set for each of these $L$-fuzzy $\beta$-covering-based rough approximation operator is investigated. Thirdly, we give the matrix representations of three types of $L$-fuzzy $\beta$-covering-based rough approximation operators, which make it valid to calculate the $L$-fuzzy $\beta$-covering-based lower and upper rough approximation operators through operations on matrices.
    Finally, the interdependency of the three pairs
    of rough approximation operators based on $L$-fuzzy $\beta$-covering is studied by using the notion of reducible elements and independent elements. In other words, we present the necessary and sufficient conditions under which two $L$-fuzzy $\beta$-coverings can generate the same lower and upper rough approximation operations.
\end{abstract}


\end{frontmatter}

\section{Introduction}\label{section1}
    Rough set theory was initiated by Pawlak \cite{Pawlak1982Rough} as a fruitful mathematical approach to tackle the vagueness, incompleteness and granularity of knowledge, which seems to be of fundamental importance to artificial intelligence and cognitive analysis, such as machine learning, decision analysis and pattern recognition. Basic notions in rough set theory are the lower and upper rough approximation operators, which can approximate an arbitrary subset. In addition, Pawlak's rough set model was explored based on equivalence relation (a binary relation with reflexivity, symmetry and transitivity). However, the equivalence relation appears to be a stringent condition which may impose restrictions on the applicability of Pawlak's rough set model\cite{wei2012comparative,liu2010closures,dai2012approximations}.
    Therefore, these restrictions had been relaxed in recent years by replacing equivalence relations or notions of partition such as binary relations \cite{greco2002rough, slowinski2000generalized, yao1998constructive, skowron1996tolerance}, neighborhood systems and Boolean algebras \cite{boixader2000upper, wu2002neighborhood}, and covering of the universe of discourse \cite{pomykala1987approximation, bonikowski1998extensions, degang2007new, pomykala1988definability}. Further, Pomykala \cite{pomykala1987approximation, pomykala1988definability} proposed two pairs of dual rough approximation operators over the concept of covering and Yao \cite{yao1998relational} examined the connections between neighborhood and approximation operators.
    In the past 40 years, rough set theory has received wide attention in both theoretical research and practical applications, mainly due to the following advantages:
     \begin{enumerate}[(1)]
       \item
       No preliminary or additional information about the data is required.
       \item
       Valid algorithms for discovering hidden patterns in data are provided.
       \item
       Minimal sets of data can be obtained.
     \end{enumerate}

     However, the problem is that rough sets mainly handle qualitative (discrete) data, and it faces restrictions with
     real-valued data sets due to the values of the attributes in the databases could be both symbolic and real-valued \cite{jensen2004fuzzy}. Fuzzy set theory \cite{zadeh1996fuzzy} can effectively solve this problem, by setting a membership function between $0$ and $1$ for each object in the set. So, it is more natural to try to combine rough set and fuzzy set, rather than to prove that one is more general, or, more useful than the other. Considering the lower and upper rough approximation in fuzzy status, Dubois and Prade \cite{dubois1990rough}, and Chakrabarty et al. \cite{chakrabarty2000fuzziness} proposed rough fuzzy sets and fuzzy rough sets. Many scholars try to use different methods to generalize fuzzy rough sets \cite{hu2014generalized, hu2015generalized,mi2004axiomatic, radzikowska2004fuzzy, wang2015granular}. The most common fuzzy rough sets can be constructed by replacing the crisp relations and crisp sets with fuzzy relations and fuzzy sets.
     In addition, the concept of fuzzy covering \cite{deng2007novel, li2007fuzzy} is also an important way for the construction of fuzzy rough sets.
     De Cock et.al \cite{de2004fuzzy} further explored fuzzy rough sets with the help of the true essence of fuzzy set theory that an element can belong (to some degree) to multiple sets at the same time, which was an important step in the fuzzification process. Similarly, Deng et al. \cite{deng2007novel} given the new model of fuzzy rough
     sets based on the concepts of both fuzzy covering and binary fuzzy logical operators. In 2007, Li and Ma \cite{li2007fuzzy} constructed two types of approximation operators, where the fuzzy covering-based fuzzy rough approximation operators employed two special logical operators i.e., the standard min operator and Kleene-Dienes implicator. Thus, it is necessary to construct more general fuzzy rough set models based on fuzzy covering. Such deficiency has stimulated more research in this area \cite{d2017fuzzy, ma2016two, yang2016fuzzy, yang2017some}. In addition, Ma in \cite{ma2016two} generalized the notion of fuzzy covering and proposed fuzzy $\beta$-covering, which can build a bridge between covering-based rough set theory and fuzzy set theory.

      In the meantime, fuzzy rough sets can be obtained by extending the basic structure $[0, 1]$ to the abstract algebraic structure. Thus, some lattice structure were introduced  to replace the interval $[0, 1]$ as the truth table for membership degrees, among which residuated lattices play a significant role. For instance, Radzikowska and Kerre \cite{radzikowska2004fuzzy} constructed the $L$-fuzzy rough sets by residuated lattice, and Deng et al. in 2007 \cite{deng2007novel} presented a fuzzy rough set by the use of complete lattice (a complete lattice means a partial ordering set with every subset possessing infimum and supremum).
      In general, we discussed residuated lattice-valued fuzzy approximation operators from two directions, called, $L$-fuzzy relation-based rough approximation operators and $L$-fuzzy $\beta$-covering-based approximation operators. In \cite{deng2007novel}, Deng et al. explored a pair of $L$-fuzzy $\beta$-covering-based rough approximation operators. In particular, Li et al. \cite{li2008generalized} introduced another two pairs of $L$-fuzzy $\beta$-covering-based approximation operators when $L$ = [0, 1]. Furthermore, Jin and Li \cite{jin2013second} given two $L$-fuzzy $\beta$-covering-based rough approximation operators with the condition of completely distributive complete lattice.

      There are at least two approaches to characterize the lower and upper rough approximation operators, called, the constructive and axiomatic approaches. The axiomatic approaches consider the abstract lower and upper rough approximation operators, by a set of axioms to depict rough approximation operators which produced using the constructive approach, see \cite{liu2013relationship, zhang2010axiomatic, zhang2011minimization, li2017theaxiomatic}.
      In 2017, Li et al. \cite{li2017theaxiomatic} studied three pairs of $L$-fuzzy $\beta$-covering-based rough approximation operators with axiomatic approach, which are presented in \cite{deng2007novel} and \cite{li2008generalized}, and analysed the differences between the axiom sets of $L$-fuzzy $\beta$-covering-based approximation operators and their crisp counterparts. Furthermore, they proved that any $L$-fuzzy $\beta$-covering-based rough approximation operators can be represented by a special $L$-fuzzy relation. On the other hand, Yang and Hu \cite{yang2018matrix} studied the Matrix representations and interdependency of three pairs of  $L$-fuzzy $\beta$-covering-based rough approximation operators. The calculation of lower and upper rough approximations can be converted into operations on matrices which greatly facilitate the computation. In addition, the research of interdependency is necessary, which show a necessary and sufficient condition under which two $L$-fuzzy coverings can generate the same $L$-fuzzy $\beta$-covering-based rough approximation operators by introducing the notion of reduct and core for a fuzzy covering.

     In this paper, we further study $L$-fuzzy $\beta$-covering-based lower and upper rough
     approximation operators.
     Following the idea of \cite{deng2007novel} and \cite{li2008generalized}, we give the definition of $L$-fuzzy $\beta$-covering-based lower and upper rough
     approximation operators by introducing the concepts such as $\beta$-degree of intersection and $\beta$-subsethood degree, which are generalizations of  degree of intersection and subsethood degree, respectively. After that, starting from three aspects of axiomatic characterizations, matrix representations and interdependency, explore the properties of lower and upper rough approximation operators based on $L$-fuzzy $\beta$-covering.

     To be more specific, the main general objectives of this paper are as follows.
      \begin{itemize}
        \item To discuss the axiomatic characterization of the defined $L$-fuzzy $\beta$-covering-based lower and upper rough approximation operators, which can give us an insight into the structure of them.
        \item To investigate the matrix representations of the defined $L$-fuzzy $\beta$-covering-based lower and upper rough approximation operators, which make it valid to calculate the $L$-fuzzy $\beta$-covering-based lower and upper rough approximation operators though operations on matrices.
        \item To focus the interdependency of the defined $L$-fuzzy $\beta$-covering-based lower and upper rough approximations, which allows us to have a clear resolution of the relationships among them.
      \end{itemize}

 In addition, it should be pointed out that the obtained results in this work not only can be a supplement of the research topic of rough set theory from the view point of mathematic, also, they provide theoretical
basis and more possibilities for the applications of rough set theory in real application problems, such as in expert systems, machine learning, decision analysis, pattern recognition and so on.

      The reminder of this paper is organized as follows. In Section~\ref{section2}, we review some preliminary definitions and results about residuated lattice, $L$-fuzzy set and so on.
      In Section~\ref{section3}, we propose the concept about $\beta$-degree of intersection and $\beta$-subsethood degree, and then the expression of three pairs of $L$-fuzzy $\beta$-covering-based rough approximation operators are given. In Section~\ref{section4}, the axiomatic characterizations on three pairs of $L$-fuzzy $\beta$-covering-based rough approximation operators are established. In Section~\ref{section5}, the matrix representations on three pairs of $L$-fuzzy $\beta$-covering-based rough approximation operators are proposed. In Section~\ref{section6}, we propose the necessary and sufficient conditions under which two $L$-fuzzy coverings can generate the same $L$-fuzzy $\beta$-covering-based rough approximation operations. In Section~\ref{section7}, we conclude this paper.

\section{Preliminaries}\label{section2}
    Throughout this paper, let the universe of discourse $U$ to be an arbitrary non-empty set.
    The class of all subsets of $U$ will be denoted by $\mathscr{P}(U)$.

    A \emph{complete residuated lattice} \cite{Ward39Residuatedlattices} is a pair $L=(L, \otimes)$ subject to the following conditions:
    \begin{enumerate}[(1)]
      \item $L$ is a complete lattice with a top element $1$ and a bottom element $0$;
      \item $(L, \otimes, 1)$ is a commutative monoid;
      \item $a\otimes \underset{j\in J}{\bigvee}b_{j}=\underset{j\in J}{\bigvee}(a\otimes b_{j})$
      for all $a\in L$ and $\{b_{j}: j\in J\}\subseteq L$.
    \end{enumerate}
    The binary operation $\otimes$ induces another binary operation $\rightarrow$ on $L$ via the adjoint property:
    \begin{align*}
    &a\otimes b\leq c\Longleftrightarrow b\leq a\rightarrow c.
    \end{align*}

    In this paper, if not otherwise specified, $L=(L, \bigwedge, \bigvee, \otimes, \rightarrow, 0, 1)$ is always a complete residuated lattice.
    In addition, a function $A: U\longrightarrow L$ is an $L$-fuzzy set in $U$. We use $L^{U}$ to denote the set of all $L$-fuzzy sets in $U$ and call it the $L$-fuzzy power set on $U$. The operators $\bigvee, \bigwedge, \otimes, \rightarrow$ on $L$ can
    be translated onto $L^{U}$ in a pointed wise. That is, for any $A, B, A_{t}~(t\in T)\in L^{U}$,
    \begin{align*}
    A\leq B&\Longleftrightarrow A(x)\leq B(x)~for~any~x\in U,\\
    (\underset{t\in T}{\bigwedge}A_{t})(x)&=\underset{t\in T}{\bigwedge}A_{t}(x),\ (\underset{t\in T}{\bigvee}A_{t})(x)=\underset{t\in T}{\bigvee}A_{t}(x),\\
    (A\otimes B)(x)&=A(x)\otimes B(x),\ (A\rightarrow B)(x)=A(x)\rightarrow B(x).
    \end{align*}
    For the convenience of readers, we provide some basic properties of the operations on complete residuated lattices in the following theorem.

    \begin{theorem}\label{t2-1}(\cite{Blount03thestructure,Hajek98Matamathematics,hohle2012mathematics,Ward39Residuatedlattices})
    Suppose that $(L, \otimes, \rightarrow, 0, 1)$  is a complete
    residuated lattice. Then, for any $\alpha, \beta\in L$ and $\{\alpha_{t}: t\in T\}, \{\beta_{t}: t\in T\}\subseteq L$, the following statements hold.
    \begin{enumerate}
      \item[(I1)] $\alpha\otimes\beta=\beta\otimes\alpha$, $\alpha\rightarrow\beta=1\Longleftrightarrow\alpha\leq\beta$, $1\rightarrow\alpha=\alpha$;
      \item[(I2)] $\alpha\otimes(\alpha\rightarrow\beta)\leq\beta$;
      \item[(I3)] $\alpha\rightarrow(\beta\rightarrow\gamma)
          =(\alpha\otimes\beta)\rightarrow\gamma=\beta\rightarrow(\alpha\rightarrow\gamma)$;
      \item[(I4)] $\alpha\otimes(\underset{t\in T}{\bigvee}\beta_{t})=\underset{t\in T}{\bigvee}(\alpha\otimes\beta_{t})$,\\
          (I4) implies that (I$4'$)~$\beta\leq\gamma\Longrightarrow\alpha\otimes\beta\leq\alpha\otimes\gamma$;
      \item[(I5)] $(\underset{t\in T}{\bigvee}\alpha_{t})\rightarrow\beta=\underset{t\in T}{\bigwedge}(\alpha_{t}\rightarrow\beta)$,\\
          (I5) implies that (I$5'$)~$\alpha\leq\beta \Longrightarrow\alpha\rightarrow\gamma\geq\beta\rightarrow\gamma$;
      \item[(I6)] $\alpha\rightarrow(\underset{t\in T}{\bigwedge}\beta_{t})=\underset{t\in T}{\bigwedge}(\alpha\rightarrow\beta_{t})$,\\
          (I6) implies that (I$6'$)~$\beta\leq\gamma\Longrightarrow\alpha\rightarrow\beta\leq\alpha\rightarrow\gamma$;
      \item[(I7)] $\alpha\otimes\beta\leq\gamma\Longleftrightarrow\alpha\leq\beta\rightarrow\gamma$;
      \item[(I8)] $\beta\leq\alpha\rightarrow(\alpha\otimes\beta)$.
    \end{enumerate}
   \end{theorem}
    A complete residuated lattice is regular if it satisfies the double negation law: $$(\alpha\rightarrow 0)\rightarrow 0=\alpha,$$
  for all $\alpha \in L$. In the following, we use $\neg\alpha$ to denote $\alpha\rightarrow 0$. A complete regular residuated satisfies the following conditions.
   \begin{enumerate}
     \item[(I9)] $\alpha\rightarrow\beta=\neg\beta\rightarrow \neg\alpha$;
     \item[(I10)] $\alpha\rightarrow\beta=\neg(\alpha\otimes \neg\beta)$;
     \item[(I11)] $\neg(\bigwedge_{t\in T}\alpha_{t})=\bigvee_{t\in T}(\neg\alpha_{t})$.
   \end{enumerate}

    Further, a complete residuated lattice is called a complete Heyting algebra or a frame if the
    binary operator $\otimes$ on $L$ with $\otimes=\bigwedge$.

    Also,we make no difference between a constant function and its value while no confusion
    will arise. Then, $(L^{U}, \otimes, \rightarrow, \bigwedge, \bigvee, 0, 1)$
    forms a complete residuated lattice.

    For a crisp subset $A\subseteq U$, let $1_{A}$ be the characteristic function, i.e., $1_{A}(x)=1$ if $x\in A$
    and $1_{A}(x)=0$ if $x\notin A$.
    Clearly, the characteristic function $1_{A}$ of a subset $A\subseteq U$ can be
    regarded as an $L$-fuzzy set in $U$. Thus, when $L=\{0, 1\}$, the set $L^{U}$
    degenerates into $\mathscr{P}(U)$
    if we make no difference between a subset of $U$ and its characteristic function.
\begin{definition}\label{d2-1}
    Let $A, B$ be $L$-fuzzy sets in $U$.
    \begin{enumerate}[(1)]
      \item In~\cite{GeorgescuPopescu2004Non-Dual}, the degree of intersection of $A, B$, denoted by
      $N(A, B)$, is defined by
      \begin{align*}
      N(A, B)&=\underset{x\in U}{\bigvee}[A(x)\otimes B(x)].
      \end{align*}
      \item In~\cite{bvelohlavek2002fuzzy}, the subsethood degree of $A, B$, denoted by $S(A, B)$, is defined by
          \begin{align*}
      S(A, B)&=\underset{x\in U}{\bigwedge}[A(x)\rightarrow B(x)].
      \end{align*}
    \end{enumerate}
\end{definition}

    The following lemma collects some properties of the above two notions. These properties appear in the literature such as~\cite{bvelohlavek2002fuzzy,chen2007construction,jinming2010stratified,
    GeorgescuPopescu2004Non-Dual,hohle2012mathematics,li2012stratified,zhang2007enriched}.

\begin{lemma}\label{l2-1}
    Let $A, B, A_{t}~(t\in T)\in L^{U}, \alpha\in L$. Then,
    \begin{enumerate}
      \item [(S1)] $S(A, B)=1\Longleftrightarrow A\leq B$;
      \item [(S2)] $S(A, \underset{t\in T}{\bigwedge}A_{t})=\underset{t\in T}{\bigwedge} S(A, A_{t})$;
      \item [(S3)] $S(A, \alpha\rightarrow B)=\alpha\rightarrow S(A, B)$;
      \item [(S4)] $S(A, \alpha\otimes B)\geq\alpha\otimes S(A, B)$;
      \item [(N1)] $N(A, B)=N(B, A)$;
      \item [(N2)] $N(A, \underset{t\in T}{\bigvee}A_{t})=\underset{t\in T}{\bigvee}N(A, A_{t})$;
      \item [(N3)] $N(A, \alpha\otimes B)=\alpha\otimes N(A, B)$;
      \item [(N4)] $N(A, \alpha\rightarrow B)\leq\alpha\rightarrow N(A, B)$;
      \item [(NS)] $N(A, B)\rightarrow\alpha=S(A, B\rightarrow\alpha)$.
    \end{enumerate}
\end{lemma}

\section{Three types of $L$-fuzzy $\beta$-covering-based rough sets}\label{section3}
    In this section, we propose three types of $L$-fuzzy $\beta$-covering-based rough sets by introducing some new concepts such as $\beta$-degree of intersection and $\beta$-subsethood degree. Furthermore, some basic properties of these $L$-fuzzy $\beta$-covering-based rough approximation operators are studied.

\subsection{$\beta$-degree of intersection and $\beta$-subsethood degree}\label{section3-1}

    In this subsection, we generalize the concepts of degree of intersection
    and subsethood degree by using a parameter $\beta$, where $\beta\in L$ and $\beta>0$. Some properties of $\beta$-degree of intersection and $\beta$-subsethood degree are also studied.

\begin{definition}\label{d3-1}
    Let $A, B$ be $L$-fuzzy sets in $U$, $\beta\in L$ and $\beta>0$.
    \begin{enumerate}[(1)]
      \item The $\beta$-degree of intersection of $A, B$, denoted by
      $N^{\beta}(A, B)$, is defined by
      \begin{align*}
      N^{\beta}(A, B)&=N(A, B)\otimes \beta=\left(\underset{x\in U}{\bigvee}\left[A(x)\otimes B(x)\right]\right)\otimes \beta.
      \end{align*}
      \item The $\beta$-subsethood degree of $A, B$, denoted by $S^{\beta}(A, B)$, is defined by
          \begin{align*}
      S^{\beta}(A, B)&=\beta\rightarrow S(A, B)\\&=\beta\rightarrow\left(\underset{x\in U}{\bigwedge}[A(x)\rightarrow B(x)]\right)
      \\&=\underset{x\in U}{\bigwedge}\left(\beta\rightarrow[A(x)\rightarrow B(x)]\right)
      \\&=\underset{x\in U}{\bigwedge}\left(\left(\beta\otimes A(x)\right)\rightarrow B(x)\right).
      \end{align*}
    \end{enumerate}
\end{definition}

\begin{remark}\label{r3-1}
    Let $A, B$ be $L$-fuzzy sets in $U$, $\beta\in L$ and $\beta>0$.
    If $\beta=1$, then $N^{1}(A, B)=N(A, B)\otimes 1=N(A, B)$
    and $S^{1}(A, B)=1\rightarrow S(A, B)=S(A, B)$.
\end{remark}

 \begin{definition}(\cite{d2017})\label{d3-2}
    A fuzzy subsethood measure is a mapping $\sigma: \mathscr{P}(U)\times\mathscr{P}(U)\longrightarrow L$ satisfying the following properties, for every $A, B, C \in \mathscr{P}(U)$:
    \begin{enumerate}[(SM1)]
      \item If $\sigma(A, B)=1$, then $A\leq B$;
      \item $\sigma(A, B)=0 \Longleftrightarrow A=1_U$ and $B=0_U$;
      \item (i) If $A\leq B\leq C$, then $\sigma(C, A)\leq \sigma(B, A)$, and (ii) if $A\leq B$, then $\sigma(C, A)\leq \sigma(C, B)$.
    \end{enumerate}
\end{definition}

If we remove item (SM2) in Definition \ref{d3-2} and replace $\leq$ item (SM1) in Definition \ref{d3-2} with $\leq_{\beta}$, and call such $\sigma$ a fuzzy pseudo-subsethood measure, then $S^{\beta}$ is a fuzzy pseudo-subsethood measure.

\begin{lemma}\label{l3-1}
    Let $A, B, A_{t}~(t\in T)\in L^{U}, \alpha, \beta\in L$ and $\beta>0$. Then,
    \begin{enumerate}
      \item [(\textbf{S}1)] $S^{\beta}(A, B)=1\Longleftrightarrow A\leq_{\beta} B$, where $A\leq_{\beta} B\Longleftrightarrow S(A, B)\geq\beta$;
      \item [(\textbf{S}2)] $S^{\beta}(A, \underset{t\in T}{\bigwedge}A_{t})=\underset{t\in T}{\bigwedge} S^{\beta}(A, A_{t})$;
      \item [(\textbf{S}3)] $S^{\beta}(A, \alpha\rightarrow B)=\alpha\rightarrow S^{\beta}(A, B)$;
      \item [(\textbf{S}4)] $S^{\beta}(A, \alpha\otimes B)\geq\alpha\otimes S^{\beta}(A, B)$;
      \item [(\textbf{N}1)] $N^{\beta}(A, B)=N^{\beta}(B, A)$;
      \item [(\textbf{N}2)] $N^{\beta}(A, \underset{t\in T}{\bigvee}A_{t})=\underset{t\in T}{\bigvee}N^{\beta}(A, A_{t})$;
      \item [(\textbf{N}3)] $N^{\beta}(A, \alpha\otimes B)=\alpha\otimes N^{\beta}(A, B)$;
      \item [(\textbf{N}4)] $N^{\beta}(A, \alpha\rightarrow B)\leq\alpha\rightarrow N^{\beta}(A, B)$;
      \item [(\textbf{NS})] $N^{\beta}(A, B)\rightarrow\alpha=S^{\beta}(A, B\rightarrow\alpha)$.
    \end{enumerate}
\end{lemma}
\begin{proof}
    (\textbf{S}1) By (I1) of Theorem~\ref{t2-1}, we have that $$A\leq_{\beta} B\Longleftrightarrow S(A, B)\geq\beta\Longleftrightarrow S^{\beta}(A, B)=\beta\rightarrow S(A, B)=1.$$

    (\textbf{S}2) We have that
    \begin{align*}
    S^{\beta}(A, \underset{t\in T}{\bigwedge}A_{t})&=\beta\rightarrow S(A, \underset{t\in T}{\bigwedge}A_{t})\\
    &=\beta\rightarrow\underset{t\in T}{\bigwedge}S(A, A_{t})~(by~(S2)~of ~Lemma~\ref{l2-1})\\
    &=\underset{t\in T}{\bigwedge}\left(\beta\rightarrow S(A, A_{t})\right)~(by~(I6)~of~Theorem~\ref{t2-1})\\
    &=\underset{t\in T}{\bigwedge}S^{\beta}(A, A_{t}).
    \end{align*}

    (\textbf{S}3) We have that
    \begin{align*}
    S^{\beta}(A, \alpha\rightarrow B)&=\beta\rightarrow S(A, \alpha\rightarrow B)\\
    &=\beta\rightarrow\left(\alpha\rightarrow S(A, B)\right)~(by~(S3)~of ~Lemma~\ref{l2-1})\\
    &=\alpha\rightarrow\left(\beta\rightarrow S(A, B)\right)~(by~(I3)~of~Theorem~\ref{t2-1})\\
    &=\alpha\rightarrow S^{\beta}(A, B).
    \end{align*}

    (\textbf{S}4) We have that
    \begin{align*}
    S^{\beta}(A, \alpha\otimes B)&=\beta\rightarrow S(A, \alpha\otimes B)\\
    &\geq \beta\rightarrow (\alpha\otimes S(A, B))~(by~(S4)~of ~Lemma~\ref{l2-1})\\
    &\geq \alpha\otimes(\beta\rightarrow S(A, B))\\
    &=\alpha\otimes S^{\beta}(A, B).
    \end{align*}

    (\textbf{N}1) By (N1) of Lemma~\ref{l2-1}, we have that
    $$N^{\beta}(A, B)=N(A, B)\otimes \beta=N(B, A)\otimes \beta=N^{\beta}(B, A).$$

    (\textbf{N}2) We have that
    \begin{align*}
    N^{\beta}(A, \underset{t\in T}{\bigvee}A_{t})&=N(A, \underset{t\in T}{\bigvee}A_{t})\otimes \beta\\
    &=\left(\underset{t\in T}{\bigvee}N(A, A_{t})\right)\otimes \beta~(by~(N2)~of ~Lemma~\ref{l2-1})\\
    &=\underset{t\in T}{\bigvee}\left(N(A, A_{t})\otimes\beta\right)~(by~(I4)~of~Theorem~\ref{t2-1})\\
    &=\underset{t\in T}{\bigvee}N^{\beta}(A, A_{t}).
    \end{align*}

    (\textbf{N}3) By (N3) of Lemma~\ref{l2-1}, we have that
    $$N^{\beta}(A, \alpha\otimes B)=N(A, \alpha\otimes B)\otimes\beta=(\alpha\otimes N(A, B))\otimes\beta=\alpha\otimes(N(A, B)\otimes\beta)=\alpha\otimes N^{\beta}(A, B).$$

    (\textbf{N}4) It is easy to see that
    $$N^{\beta}(A, \alpha\rightarrow B)\leq\alpha\rightarrow N^{\beta}(A, B)\Longleftrightarrow \alpha\otimes N^{\beta}(A, \alpha\rightarrow B)\leq N^{\beta}(A, B).$$
    Since
    \begin{align*}
    \alpha\otimes N^{\beta}(A, \alpha\rightarrow B)&=\alpha\otimes \left(N(A, \alpha\rightarrow B)\otimes\beta\right)\\
    &\leq\alpha\otimes \left(\alpha\rightarrow N(A, B)\right)\otimes \beta~(by~(N4)~of ~Lemma~\ref{l2-1})\\
    &\leq N(A, B)\otimes \beta~(by~(I8)~of~Theorem~\ref{t2-1})\\
    &=N^{\beta}(A, B),
    \end{align*}
    we have that $N^{\beta}(A, \alpha\rightarrow B)\leq\alpha\rightarrow N^{\beta}(A, B)$.

     (\textbf{NS}) We have that
     \begin{align*}
     N^{\beta}(A, B)\rightarrow\alpha&=(N(A, B)\otimes\beta)\rightarrow\alpha\\
     &=\beta\rightarrow(N(A, B)\rightarrow\alpha)~(by~(I3)~of~Theorem~\ref{t2-1})\\
     &=\beta\rightarrow S(A, B\rightarrow\alpha)~(by~(NS)~of ~Lemma~\ref{l2-1})\\
     &=S^{\beta}(A, B\rightarrow\alpha).
     \end{align*}
\end{proof}

\begin{lemma}\label{l3-2}
    Let $A, B, C\in L^{U}$, $\beta\in L$ and $\beta>0$. If $A\leq B$, then
    $S^{\beta}(A, C)\geq S^{\beta}(B, C)$ and $N^{\beta}(A, C)\leq N^{\beta}(B, C)$.
\end{lemma}
\begin{proof}
    By (I5$'$) of Theorem~\ref{t2-1}, we have that
    $$\underset{x\in U}{\bigwedge}[A(x)\rightarrow C(x)]\geq\underset{x\in U}{\bigwedge}[B(x)\rightarrow C(x)].$$
    Then, it follows Definition~\ref{d3-1} and (I6$'$) of Theorem~\ref{t2-1} that
    \begin{align*}
    S^{\beta}(A, C)&=\beta\rightarrow\left(\underset{x\in U}{\bigwedge}[A(x)\rightarrow C(x)]\right)\\
    &\geq \beta\rightarrow\left(\underset{x\in U}{\bigwedge}[B(x)\rightarrow C(x)]\right)\\
    &=S^{\beta}(B, C).
    \end{align*}

    By (I4$'$) of Theorem~\ref{t2-1}, we have that
    $$\underset{x\in U}{\bigvee}[A(x)\otimes C(x)]\leq\underset{x\in U}{\bigvee}[B(x)\otimes C(x)].$$
     Then it follows Definition~\ref{d3-1} that
    \begin{align*}
    N^{\beta}(A, C)&=\left(\underset{x\in U}{\bigvee}[A(x)\otimes C(x)]\right)\otimes\beta\\
    &\leq\left(\underset{x\in U}{\bigvee}[B(x)\otimes C(x)]\right)\otimes\beta\\
    &=N^{\beta}(B, C).
    \end{align*}
   Therefore, $S^{\beta}(A, C)\geq S^{\beta}(B, C)$ and $N^{\beta}(A, C)\leq N^{\beta}(B, C)$ can be followed.
\end{proof}

\begin{lemma}\label{l3-3}
    Let $A, B\in L^{U}$, $(A_{i})_{i\in\{1, 2, \ldots, r\}}$, $\beta\in L$ and $\beta>0$. If $A=\underset{i=1}{\overset{r}\bigvee}A_{i}$, then
    $S^{\beta}(A, B)=\underset{i=1}{\overset{r}\bigwedge}S^{\beta}(A_{i}, B)$
    and $N^{\beta}(A, B)=\underset{i=1}{\overset{r}\bigvee}N^{\beta}(A_{i}, B)$.
\end{lemma}
\begin{proof}
    By Definition~\ref{d3-1}, we have that
    \begin{align*}
    S^{\beta}(A, B)&=\beta\rightarrow\left(\underset{x\in U}{\bigwedge}[A(x)\rightarrow B(x)]\right)\\
    &=\beta\rightarrow\left(\underset{x\in U}{\bigwedge}[\underset{i=1}{\overset{r}\bigvee}A_{i}(x)\rightarrow B(x)]\right)\\
    &=\beta\rightarrow\left(\underset{x\in U}{\bigwedge}\underset{i=1}{\overset{r}\bigwedge}[A_{i}(x)\rightarrow B(x)]\right)~(by~(I5)~of~Theorem~\ref{t2-1})\\
    &=\beta\rightarrow\left(\underset{i=1}{\overset{r}\bigwedge}\underset{x\in U}{\bigwedge}[A_{i}(x)\rightarrow B(x)]\right)\\
    &=\beta\rightarrow\left(\underset{i=1}{\overset{r}\bigwedge}S(A_{i}, B)\right)\\
    &=\underset{i=1}{\overset{r}\bigwedge}[\beta\rightarrow S(A_{i}, B)]~(by~(I6)~of~Theorem~\ref{t2-1})\\
    &=\underset{i=1}{\overset{r}\bigwedge}S^{\beta}(A_{i}, B),
    \end{align*}
    and
    \begin{align*}
    N^{\beta}(A, B)&=\left(\underset{x\in U}{\bigvee}[A(x)\otimes B(x)]\right)\otimes\beta\\
    &=\left(\underset{x\in U}{\bigvee}[\underset{i=1}{\overset{r}{\bigvee}}A_{i}(x)\otimes B(x)]\right)\otimes\beta\\
    &=\left(\underset{x\in U}{\bigvee}\underset{i=1}{\overset{r}{\bigvee}}[A_{i}(x)\otimes B(x)]\right)\otimes\beta~(by~(I4)~of~Theorem~\ref{t2-1})\\
    &=\underset{i=1}{\overset{r}{\bigvee}}\left(\underset{x\in U}{\bigvee}[A_{i}(x)\otimes B(x)]\right)\otimes\beta\\
    &=\left(\underset{i=1}{\overset{r}{\bigvee}} N(A_{i}, B)\right)\otimes\beta\\
    &=\underset{i=1}{\overset{r}{\bigvee}}[N(A_{i}, B)\otimes\beta]~(by~(I4)~of~Theorem~\ref{t2-1})\\
    &=\underset{i=1}{\overset{r}{\bigvee}}N^{\beta}(A_{i}, B).
    \end{align*}
    Then
    $S^{\beta}(A, B)=\underset{i=1}{\overset{r}\bigwedge}S^{\beta}(A_{i}, B)$
    and $N^{\beta}(A, B)=\underset{i=1}{\overset{r}\bigvee}N^{\beta}(A_{i}, B)$ can be followed.
\end{proof}

In the following, we shall give the concept of three pairs of $L$-fuzzy $\beta$-covering-based lower and upper rough approximation operators.

\subsection{The first pair of $L$-fuzzy $\beta$-covering-based rough approximation operators}\label{section3-2}

\begin{definition}\cite{ma2016two}\label{d-covering}
	Let $\mathscr{C}$ be a family of subsets of $L^U$. For each $\beta\in L$, we call $\mathscr{C}$ a $L$-fuzzy $\beta$-covering of $U$,
	if $\left(\bigcup_{C\in \mathscr{C}} C\right)(x) \geq \beta$ for any $x\in U$, and the ordered pair $(U,\mathscr{C})$ is called a fuzzy $\beta$-covering approximation space.	
\end{definition}

\begin{definition}\label{d3-4}
    Let $\mathscr{C}$ be an $L$-fuzzy $\beta$-covering on $U$.
    Then, the first pair of $L$-fuzzy $\beta$-covering-based lower and upper rough approximation operators $\underline{\mathscr{C}}_{1}, \overline{\mathscr{C}}_{1}: L^{U}\longrightarrow L^{U}$
    are defined as follows: for any $X\in L^{U}$,
    \begin{align*}
    \underline{\mathscr{C}}_{1}(X)&=\underset{C\in \mathscr{C}}{\bigvee}[C\otimes S^{\beta}(C, X)],\\
     \overline{\mathscr{C}}_{1}(X)&=\underset{C\in \mathscr{C}}{\bigwedge}[C\rightarrow N^{\beta}(C, X)].
    \end{align*}
\end{definition}
\begin{remark}
    In~\cite{li2017theaxiomatic}, Li et al. generalized the first pair of covering-based rough approximation operators to lattice-valued, and proposed the first pair of $L$-fuzzy covering-based rough approximation operators $\underline{\mathcal{C}}_{1}(X)$ and $\overline{\mathcal{C}}_{1}(X)$. Here, the concept of the first pair of $L$-fuzzy $\beta$-covering-based rough approximation operators is given again based on $L$-fuzzy $\beta$-covering.
    Indeed, when $\beta = 1$, one concludes that
     \begin{align*}
    \underline{\mathscr{C}}_{1}(X)
    &=\underset{C\in \mathscr{C}}{\bigvee}[C\otimes S^{\beta}(C, X)]\\
    &=\underset{C\in \mathscr{C}}{\bigvee}[C\otimes S(C, X)]\\
    &=\underline{\mathcal{C}}_{1}(X),
    \end{align*}
 and
   \begin{align*}
   \overline{\mathscr{C}}_{1}(X)
   &=\underset{C\in \mathscr{C}}{\bigwedge}[C\rightarrow N^{\beta}(C, X)]\\
   &=\underset{C\in \mathscr{C}}{\bigwedge}[C\rightarrow N(C, X)]\\
   &=\overline{\mathcal{C}}_{1}(X).
    \end{align*}
\end{remark}

\subsection{The second pair of $L$-fuzzy $\beta$-covering-based rough approximation operators}\label{section3-3}

\begin{definition}\label{d3-5}
    Let $\mathscr{C}$ be an $L$-fuzzy $\beta$-covering on $U$.
    Then, the second pair of $L$-fuzzy $\beta$-covering-based lower
    and upper rough approximation operators $\underline{\mathscr{C}}_{2}, \overline{\mathscr{C}}_{2}: L^{U}\longrightarrow L^{U}$
    are defined as follows: for any $X\in L^{U}$,
    \begin{align*}
    \underline{\mathscr{C}}_{2}(X)&=\underset{C\in \mathscr{C}}{\bigwedge}[C\rightarrow S^{\beta}(C, X)],\\
     \overline{\mathscr{C}}_{2}(X)&=\underset{C\in \mathscr{C}}{\bigvee}[C\otimes N^{\beta}(C, X)].
    \end{align*}
\end{definition}

\begin{remark}
	In~\cite{li2017theaxiomatic}, Li et al. generalized the second pair of covering-based rough approximation operators to lattice-valued, and proposed the second pair of $L$-fuzzy covering-based rough approximation operators $\underline{\mathcal{C}}_{2}(X)$ and $\overline{\mathcal{C}}_{2}(X)$. Here, the concept of the second pair of $L$-fuzzy $\beta$-covering-based rough approximation operators is given again based on $L$-fuzzy $\beta$-covering.
	Indeed, when $\beta = 1$, one concludes that
	\begin{align*}
	\underline{\mathscr{C}}_{2}(X)
	&=\underset{C\in \mathscr{C}}{\bigwedge}[C\rightarrow S^{\beta}(C, X)]\\
	&=\underset{C\in \mathscr{C}}{\bigwedge}[C\rightarrow S(C, X)]\\
	&=\underline{\mathcal{C}}_{2}(X),
	\end{align*}
	and
	\begin{align*}
	 \overline{\mathscr{C}}_{2}(X)
	 &=\underset{C\in \mathscr{C}}{\bigvee}[C\otimes N^{\beta}(C, X)]\\
	 &=\underset{C\in \mathscr{C}}{\bigvee}[C\otimes N(C, X)]\\
	&=\overline{\mathcal{C}}_{2}(X).
	\end{align*}	
\end{remark}

\subsection{The third pair of $L$-fuzzy $\beta$-covering-based rough approximation operators}\label{section3-4}
Let $\mathscr{C}$ be an $L$-fuzzy $\beta$-covering on $U$. $R_{\mathscr{C}}(x, y): U\times U\longrightarrow L$ is defined as follows:
$$R_{\mathscr{C}}(x, y)=\underset{C\in \mathscr{C}}{\bigwedge}[C(x)\rightarrow C(y)].$$
It is easy to see that $R_{\mathscr{C}}$ is a reflexive and
transitive $L$-fuzzy relation.

\begin{definition}\label{d3-6}
    Let $\mathscr{C}$ be an $L$-fuzzy $\beta$-covering on $U$.
    Then, the third pair of $L$-fuzzy $\beta$-covering-based lower and upper rough approximation operators $\underline{\mathscr{C}}_{3}, \overline{\mathscr{C}}_{3}: L^{U}\longrightarrow L^{U}$
    are defined as follows: for any $X\in L^{U}$,
    \begin{align*}
    \underline{\mathscr{C}}_{3}(X)(x)&=S^{\beta}(R_{\mathscr{C}}(-, x), X),\\
     \overline{\mathscr{C}}_{3}(X)(x)&=N^{\beta}(R_{\mathscr{C}}(-, x), X),
    \end{align*}
    where $R_{\mathscr{C}}(-, x)\in L^{U}$ and $R_{\mathscr{C}}(-, x)(y)=R_{\mathscr{C}}(y, x)$ for any $y\in U$.
\end{definition}

\begin{remark}
	In~\cite{li2017theaxiomatic}, Li et al. generalized the third pair of covering-based rough approximation operators to lattice-valued, and proposed the third pair of $L$-fuzzy $\beta$-covering-based rough approximation operators $\underline{\mathcal{C}}_{3}(X)$ and $\overline{\mathcal{C}}_{3}(X)$. Here, the concept of the third pair of $L$-fuzzy $\beta$-covering-based rough approximation operators is given again based on $L$-fuzzy $\beta$-covering.
	Indeed, when $\beta = 1$, one concludes that
	\begin{align*}
	\underline{\mathscr{C}}_{3}(X)(x)
	&=S^{\beta}(R_{\mathscr{C}}(-, x), X),\\
	&=S(R_{\mathscr{C}}(-, x), X),\\
	&=\underline{\mathcal{C}}_{3}(X),
	\end{align*}
	and
	\begin{align*}
	\overline{\mathscr{C}}_{3}(X)(x)
	&=N^{\beta}(R_{\mathscr{C}}(-, x), X)\\
	&=N(R_{\mathscr{C}}(-, x), X)\\
	&=\overline{\mathcal{C}}_{3}(X).
	\end{align*}	
\end{remark}

%
\section{The axiomatic characterizations on three pairs of $L$-fuzzy $\beta$-covering-based rough approximation operators}\label{section4}

\subsection{Some operators induced by $L$-fuzzy relation}
In this section, we shall find that all $L$-fuzzy $\beta$-covering-based rough approximation operators discussed in this paper can be characterized by special $L$-fuzzy relations. An $L$-fuzzy relation between a set $X$ and a set $Y$ is a mapping $R: X\times Y\longrightarrow L$. Next, two pairs of  operators are given based on the fuzzy relationship.

\begin{definition}\label{d4-1}
 For any $A \in L^{X}, B \in L^{Y}, x \in X$ and $y \in Y$, $\uparrow_{R}: L^{X} \longrightarrow L^{Y}, \downarrow_{R}: L^{Y} \longrightarrow L^{X}$, $\Uparrow_{R}: L^{X} \longrightarrow L^{Y}, \Downarrow_{R}: L^{Y} \longrightarrow L^{X}:$
\begin{align*}
\begin{array}{l}
\uparrow_{R} A(y)=N^\beta(A, R(-, y)), \quad \downarrow_{R} B(x)=S^\beta(R(x,-), B); \\
\Uparrow_{R} A(y)=S^\beta(R(-, y), A), \quad \Downarrow_{R} B(x)=N^\beta(B, R(x,-)).
\end{array}
\end{align*}	
\end{definition}

\begin{definition}\label{d4-2}\cite{GeorgescuPopescu2004Non-Dual}
	Let $X, Y$ be two universes. A pair of functions $(\uparrow, \downarrow)$, where $\uparrow: L^{X} \longrightarrow L^{Y}$ and $\downarrow: L^{Y} \longrightarrow L^{X}$ is called
	$L$-isotone Galois connection  between $X$ and $Y$ if
	\begin{align*}
  S\left(A, {\downarrow }B\right)=S\left({\uparrow} A, B \right).
	\end{align*}
\end{definition}

According to Definition~\ref{d4-2}, it is easy to verify that the pair $\left(\uparrow_{R}, \downarrow_{R}\right)$ and $(\Downarrow_{R}, \Uparrow_{R})$ form $L$-isotone Galois connection between $X$ and $Y$.

\subsection{The axiomatic characterizations on
	the first pair of $L$-fuzzy $\beta$-covering-based rough approximation operators}\label{section4-3}
\subsubsection{On lower rough approximation operator $\underline{\mathscr{C}}_{1}$}
With the help of Definition~\ref{d4-1}, we can give the connection between $\underline{\mathscr{C}}_{1}$ and the special fuzzy relation.

\begin{proposition}\label{p4-1}
  Let $\mathscr{C}$ be an $L$-fuzzy $\beta$-covering on $U$ and $R_{\mathscr{C}}: U \times \mathscr{C} \longrightarrow L$ be an $L$-fuzzy relation defined as $R_{\mathscr{C}}(x, C)=C(x),$ where $(x, C) \in U \times \mathscr{C}$. Then, $\underline{\mathscr{C}}_{1}\otimes\beta=\Downarrow_{R_{\mathscr{C}}} \circ \Uparrow_{R_{\mathscr{C}}}.$
\end{proposition}

\begin{proof}
 For any $X \in L^{U}, x \in U$, one concludes that
\begin{align*}
\left(\Downarrow_{R_{\mathscr{C}}} \circ \Uparrow_ {R_{\mathscr{C}}}\right) X(x)
&=N^\beta\left(\Uparrow_{R_{\mathscr{C}}} X, R_{\mathscr{C}}(x,-)\right)\\
&=\left(\underset{C \in \mathscr{C}}{\bigvee}\left(\Uparrow_{R_{\mathscr{C}}} X(C) \otimes R_{\mathscr{C}}(x, C)\right)\right)\otimes\beta \\
&=\left(\bigvee_{C \in \mathscr{C}}\left(S^\beta\left(R_{\mathscr{C}}(-, C), X\right) \otimes C(x)\right)\right)\otimes\beta\\
&=\left(\underset{C \in \mathscr{C}}{\bigvee}(C(x) \otimes S^\beta(C, X))\right)\otimes\beta\\
&=\underline{\mathscr{C}}_{1} X(x)\otimes\beta.
\end{align*}	
\end{proof}

Note that the pair $\left(\Downarrow_{R}, \Uparrow_{R}\right)$ forms an $L$-isotone Galois connection between $X$ and $Y$. Then, it follows from Proposition \ref{p4-1} that the operator $\underline{\mathscr{C}}_{1}\otimes\beta=\Downarrow_{R_{\mathscr{C}}} \circ\Uparrow_{R_{\mathscr{C}}}$ is an $L$-interior operator on $U$, that is, there exists a function $f: L^{U} \longrightarrow L^{U}$ satisfying the following conditions:
\begin{enumerate}	
     \item [$(L2')$]
	$A \leq B \Longrightarrow f(A) \leq f(B)$;	
     \item	[$(L3')$]
	 $f(A) \leq A$;
     \item [$(L4')$]
	 $f(A) \leq ff(A)$;
     \item[$(L5')$]
	$\alpha\otimes f(A) \leq f(\alpha \otimes A)$.	
\end{enumerate}
where $A, B \in L^{U}, \alpha \in L$.
Further, it is proposed in \cite{Belohlavek2004FuzzyInterior} that $f=\underline{\mathscr{C}}_{1}\otimes\beta=g\otimes\beta$ when $\mathscr{C}=\left\{f(X)=X \mid X \in L^{U}\right\}$.

In the following, we give the axiom set to characterize the approximation operator $\underline{\mathscr{C}}_{1}$ and then verify they are uncorrelated with each other by means of examples.
\begin{theorem}
	Let $g: L^U\longrightarrow L^U$ be an operator and $\beta\in L$. Then there exists an $L$-fuzzy $\beta$-covering $\mathscr{C}$ on $U$ such that $g=\underline{\mathscr{C}}_{1}$ if and only if it satisfies the following propoties:
	
	\begin{enumerate}
		\item [(L1)]
		$g(1_U)\ge\beta;$
		
		\item [(L2)]
		$A\le B\Longrightarrow g(A)\le g(B)$;
		
		\item [(L3)]
		$g(X)\le \beta\rightarrow X;$
		
		\item [(L4)]
		$g(X)\le g(\beta \otimes g(X));$
		
		\item [(L5)]
		$\alpha\otimes g(X)\le g(\alpha \otimes X).$
	\end{enumerate}
\end{theorem}

\begin{proof}
	Since $f$ satisfies (L2$'$)-(L5$'$) and $f=g\otimes\beta$, $g$ obviously satisfies (L2)-(L5). Moreover, we verify the relationship between (L1) and the coverage property. Suppose that  $g: L^{U} \longrightarrow L^{U}$ be an operator satisfying (L1)-(L5) and $\mathscr{C}=\left\{\beta\otimes g(X)=X \mid X \in L^{U}\right\}.$
	Then, it follows (L1) that $U\in\mathscr{C}$, $g$ is an $L$-fuzzy $\beta$-covering on $U$. In turn, let $g=\underline{\mathscr{C}}_{1}$ for an $L$-fuzzy $\beta$-covering $\mathscr{C}$ on $U$. Then,  $\underline{\mathscr{C}}_{1}(1)=$ $\underset{{C \in \mathscr{C}}}{\bigvee}(C \otimes S^\beta(C, 1))=\underset{{C \in \mathscr{C}}}{\bigvee} C\ge \beta$, that is,  $g$ satisfies the condition (L1).
\end{proof}
  \begin{example}\label{e4-1}	
  	Let $L=\{0,1\}$, $U=\{x, y\}$ and $\alpha, \beta\in L$. $(\otimes,\rightarrow)$ is an  G\"{o}del adjoint pair, which is defined as
    \begin{align*}
      a\otimes b =\text{min}\{a, b\}\quad and\quad a\rightarrow b =\left\{\begin{array}{ll}
   		1, & a\leq b,\\
   		b, & a>b.
   		\end{array}\right.
    \end{align*}
    where $a,b\in L.$
  	\begin{enumerate}[(1)]
  		\item
  		Let $g: L^{U} \longrightarrow L^{U}$ as $g \equiv 0_{U}$. It is easy to verify that $g$ satisfies (L2)-(L5) but does not satisfy (L1).
  		\item
  		Let $g: L^{U} \longrightarrow L^{U}$ as
  		$$
  		g(X)=\left\{\begin{array}{ll}
  		0_{U}, & X=0_{U}, \\
  		1_{U}, & \text { otherwise. }
  		\end{array}\right.
  		$$
  		It is easy to verify that $g$ satisfies (L1), (L2), (L4) and (L5).
  		But for any $\beta\in L$, $g$ does not satisfy (L3) for $X=1_{\{y\}}$,
  		\begin{align*}
  			g(1_{\{y\}})(x)= 1_U(x)=1\nleq 0=1\rightarrow 0 = 1\rightarrow 1_{\{y\}}(x).
  		\end{align*}
  	\end{enumerate}
  \end{example}

   \begin{example}\label{e4-2}
   		\begin{table}
   		\renewcommand{\tabcolsep}{4.5pc}
   		\caption{ Residuated.}
   		{\begin{tabular}{cccc} \toprule
   				$\rightarrow$ & 0 & $a$ & $1$ \\ \midrule
   				$0$ & $1$ & $1$  & $1$\\
   				$a$ & $0$ & $1$ & $1$ \\
   				$1$ & $0$ & $a$ & $1$ \\ \bottomrule
   		\end{tabular}}
   		\label{table1}
   	\end{table}

    Let $U=\{x, y\}$ and $L=\{0, a, 1\}$ be a complete lattice with $0<a<1.$ Let $\otimes=\bigwedge$ and $\rightarrow$ be defined by Table \ref{table1}.
   	It is easy to verify that $L=(L, \otimes, \rightarrow, \bigwedge, \bigvee, 0,1)$ is a complete residuated lattice. In particular, for any $\alpha\in\ L$, we take $\beta=1$ and $g: L^{U} \longrightarrow L^{U}$ as
   	\begin{align*}
   		g(X)=\left\{\begin{array}{ll}
   		1_{U}, & X= 1_{U},\\
   		0_{U}, & \text { otherwise. }
   		\end{array}\right.
   	\end{align*}
   	
   	It is easy to verify that $g$ satisfies (L1)-(L4). But $g$ does not satisfy (L5) since, for $a=\alpha\in L$,
   	\begin{align*}
   	a\otimes g(1_U)(x)=a\otimes 1_U(x)=a\nleq 0 =g(a_U)(x)= g(a\otimes 1_U)(x).
   	\end{align*}
   \end{example}

\begin{example}\label{e4-3}
	Let $U=\{x, y, z\}, L=\{0,1\}$, $\alpha,\beta\in L$ and $(\otimes,\rightarrow)$ be an  G\"{o}del adjoint pair as Example~\ref{e4-1}.
	\begin{enumerate}
	   \item [(1)]
	    Let $g: L^{U} \longrightarrow L^{U}$ as
	   \begin{align*}
	   	g(X)=\left\{\begin{array}{ll}
	   	1_{U}, & X=1_{U}, \\
	   	0_{U}, & X=1_{\{x\}}, 1_{\{y\}}, 1_{\{z\}}, 0_{U}, \\
	   	1_{\{x\}}, & X=1_{\{x, y\}}, \\
	   	1_{\{y\}}, & X=1_{\{y, z\}}, \\
	   	1_{\{z\}}, & X=1_{\{z, x\}}.
	   	\end{array}\right.
	   	\end{align*}	   	
	   	It is easy to verify that $g$ satisfies (L1)-(L3) and (L5). But $g$ does not satisfy (L4) since  	
	   	\begin{align*}
	   		g\left(1_{\{x, y\}}\right)=1_{\{x\}} \nleq
	   		0_{U}=g(\beta\otimes 1_{\{x\}})=g\left(\beta\otimes g\left(1_{\{x, y\}}\right)\right).
	   	\end{align*}
	\end{enumerate}

	(2) Let $g: L^{U} \longrightarrow L^{U}$ as
	\begin{align*}
		g(X)=\left\{\begin{array}{ll}
		X, & X=1_{U}, 1_{\{x\}}, 1_{\{y\}}, 1_{\{z\}}, 0_{U},\\
		0_{U}, & X=1_{\{x, y\}}, 1_{\{y, z\}}, 1_{\{z, x\}}.
		\end{array}\right.
	\end{align*}
	In particular, we can verify that $g$ satisfies (L1) and (L3)-(L5) with $\beta=1$. But $g$ does not satisfy (L2) since
	\begin{align*}
			g\left(1_{\{y\}}\right)=1_{\{y\}} \nleq 0_{U}=g\left(1_{\{x, y\}}\right).
	\end{align*}
\end{example}

\subsubsection{On upper rough approximation operator $\overline{\mathscr{C}}_{1}$}

With the help of Definition~\ref{d4-1}, we can give the connection between $\overline{\mathscr{C}}_{1}$ and the special fuzzy relation.

\begin{proposition}\label{p4-2}
	Let $\mathscr{C}$ be an L-fuzzy $\beta$-covering on $U$ and $R_{\mathscr{C}}: U \times \mathscr{C} \rightarrow L$ be an $L$-fuzzy relation defined as $R_{\mathscr{C}}(x, C)=C(x),$ where $(x, C) \in U \times \mathscr{C}$.  Then, $\beta\rightarrow\overline{\mathscr{C}}_{1}=\downarrow_{R_{\mathscr{C}}} \circ \uparrow_{R_{\mathscr{C}}}.$
\end{proposition}

\begin{proof}
	For any $X \in L^{U}, x \in U$, one concludes that
	\begin{align*}
	\left(\downarrow_{R_{\mathscr{C}}} \circ \uparrow_ {R_{\mathscr{C}}}\right) X(x)
	&=S^\beta\left(R_{\mathscr{C}}(x,-), \uparrow_{R_{\mathscr{C}}} X\right)\\
	&=\underset{C \in \mathscr{C}}{\bigwedge}\left((\beta\otimes R_{\mathscr{C}}(x, C))\rightarrow\uparrow_{R_{\mathscr{C}}} X(C) \right)\\
	&=\underset{C \in \mathscr{C}}{\bigwedge}\left((\beta\otimes C(x))\rightarrow N^\beta(R_{\mathscr{C}}(-,C), X) \right)\\	
	&=\underset{C \in \mathscr{C}}{\bigwedge}\left((\beta\otimes C(x))\rightarrow N^\beta(C, X) \right)\\
	&=\underset{C \in \mathscr{C}}{\bigwedge}\left(\beta\rightarrow( C(x)\rightarrow N^\beta(C, X) )\right)\\
	&=\beta\rightarrow\underset{C \in \mathscr{C}}{\bigwedge}\left( C(x)\rightarrow N^\beta(C, X)\right)\\	
	&=\beta\rightarrow\overline{\mathscr{C}}_{1} X(x).
	\end{align*}	
\end{proof}

Note that the pair $\left(\uparrow_{R}, \downarrow_{R}\right)$ forms an $L$-isotone Galois connection between $X$ and $Y$. Then, it follows from Proposition \ref{p4-2} that the operator $\beta\rightarrow\overline{\mathscr{C}}_{1}=\downarrow_{R_{\mathscr{C}}} \circ \uparrow_{R_{\mathscr{C}}}$ is an $L$-closure operator on $U$, that is, there exists a function $f: L^{U} \longrightarrow L^{U}$ satisfying the following conditions:
\begin{enumerate}	
     \item [$(U1')$]
	$f(0) =0;$	
     \item	[$(U2')$]
	 $A \leq B \Longrightarrow f(A) \leq f(B)$;
     \item [$(U3')$]
	 $f(A) \ge A$;
     \item[$(U4')$]
	$f(A) \ge ff(A)$;	
     \item[$(U5')$]
	$\alpha\rightarrow f(A) \ge f(\alpha \rightarrow A)$,	
\end{enumerate}
where $A, B \in L^{U}, \alpha \in L$.
Further, there exists a function $g: L^{U} \longrightarrow L^{U}$ such that  $f=\beta\rightarrow\overline{\mathscr{C}}_{1}=\beta\rightarrow g$ holds. It is easy to verify that $g$ satisfies the following conditions:
\begin{enumerate}	
     \item [$(U1)$]
	$g(0)=0;$	
     \item	[$(U2)$]
	 $A \leq B \Longrightarrow g(A) \leq g(B)$;
     \item [$(U3)$]
	 $g(A) \ge \beta\otimes A$;
     \item[$(U4)$]
	$g(A) \ge g(\beta\rightarrow g(A))$;	
     \item[$(U5)$]
	$ \alpha\rightarrow g(A) \ge g(\alpha \rightarrow A)$,	
\end{enumerate}
where $A, B \in L^{U}, \alpha,\beta \in L$.
	
	In 2001, B{\v{e}}lohl{\'a}vek introduced the axiomatic characterization of $L$-closure operator in \cite{Belohlavek2001FuzzyClosure}. For an $L$-closure operator $f: L^{U} \longrightarrow L^{U}$, there exists an $L$-closure system $\mathscr{D}=\left\{f(A)=A \mid A \in L^{U}\right\}$ such that $f=\overline{\mathscr{D}}^{\dagger}$, where $\overline{\mathscr{D}}^{\dagger}$ is defined by
	\begin{align*}
	\overline{\mathscr{D}}^{\dagger}(X)=\underset{C \in \mathscr{D}}{\bigwedge}(S(X, C) \rightarrow C) ,
	\end{align*}
	for any $A \in L^{U}$. Furthermore, $f=\beta\rightarrow g$, then
	there exists an $L$-closure system $\mathscr{D}=\left\{\beta\rightarrow g(A)=A \mid A \in L^{U}\right\}$ such that $g=\overline{\mathscr{D}}^{\ddag}$, where $\overline{\mathscr{D}}^{\ddag}$ is defined by
	\begin{align*}
	\overline{\mathscr{D}}^{\ddagger}(X)=\underset{C \in \mathscr{D}}{\bigwedge}(S^\beta(X, C) \rightarrow C),
	\end{align*}
	for any $A \in L^{U}$.

Next, we consider the axiomatic characterization on approximation operator $\overline{\mathscr{C}}_{1}.$

\begin{theorem}\label{t4-2}
	Let $L$ be a regular complete residuated lattice. Then, there exists an $L$-fuzzy $\beta$-covering $\mathscr{C}$ on $U$ such that $g=\overline{\mathscr{C}}_{1}$ if and only if $g$ satisfies $(U1)$-$(U5)$.
\end{theorem}

\begin{proof}
	$\Longrightarrow)$: According to the above analysis, $g$ satisfies the properties (U1)-(U5) when $g=\overline{\mathscr{C}}_{1}$ for an $L$-fuzzy $\beta$-covering $\mathscr{C}$ on $U$.
	
	$\Longleftarrow)$: Let $\mathscr{C}=\left\{A=\neg g(\neg A) \mid A \in L^{U}\right\}.$ It follows (U1) and (U3) that $U \in \mathscr{C}$ and $g(U)(x)\ge\beta\otimes U(x)=\beta$,
	so $\mathscr{C}$ is an $L$-fuzzy $\beta$-covering on $U.$
	
	 In addition, it is easy to see that $\mathscr{C}=\neg \mathscr{D}=\{\neg C \mid C \in \mathscr{D}\}.$ Next, we shall prove that $\overline{\mathscr{D}}^{\dagger}=\overline{\mathscr{C}}_{1}$, and then by $g=\overline{\mathscr{D}}^{\dagger}$, we obtain that $g=\overline{\mathscr{C}}_{1}.$ For any $X \in L^{U}$,
	\begin{align*}
		\overline{\mathscr{D}}^{\dagger}(X)
    &=\underset{{C \in \mathscr{D}}}{\bigwedge}(S^\beta(X, C) \rightarrow C) \\
    &= \underset{{C \in \mathscr{D}}}{\bigwedge}(\neg C \rightarrow \neg S^\beta(X, C))\quad (by~~(I9))\\
	&=\underset{{C \in \mathscr{D}}}{\bigwedge}(\neg C \rightarrow N^\beta(\neg C, X))\quad (by~~(I10)~~and~~(I11) )\\
    &=\underset{{H \in \mathscr{D}}}{\bigwedge}(H \rightarrow N^\beta(H, X))\\
    &=\overline{\mathscr{C}}_{1}(X) .
	\end{align*}
Hence, $g=	\overline{\mathscr{D}}^{\dagger}(X)=\overline{\mathscr{C}}_{1}(X)$.
\end{proof}
In a similar way, it can be proven that the axiom set (U1)-(U5) is uncorrelated.

\begin{remark}\label{r4-1}
	Similar to the discussion of \cite{li2017theaxiomatic}, a weakened double negation law is given. Consider that in all complete residuated lattices, the ordinary complement $\neg a$ can be replaced by a new double negation law: $$a=\underset{b\in L}{\bigwedge}(a\rightarrow b)\rightarrow b.$$ By this idea, we can derive the duality between the upper and lower rough approximation operators $\overline{\mathscr{C}}_{1}$ and $\underline{\mathscr{C}}_{1}$.
\end{remark}

\begin{theorem}
    Let $\mathscr{C}$ be an $L$-fuzzy $\beta$-covering on $U$. Then, for any $X \in L^{U}$, it holds that $$ \overline{\mathscr{C}}_{1} (X)=\underset{b\in L}{\bigwedge}\left(\underline{\mathscr{C}_{1}}(X \rightarrow b) \rightarrow b\right).$$
\end{theorem}
    \begin{proof}
    	For any $X \in L^{U}$, we have that
    	\begin{align*}
    		\bigwedge_{b \in L}\left(\underline{\mathscr{C}_{1}}(X \rightarrow b) \rightarrow b\right)
    		=&\bigwedge_{b \in L}\left(\left(\underset{C \in \mathscr{C}}{\bigvee}(C \otimes S^\beta(C, X \rightarrow b))\right) \rightarrow b\right)\\
    	    \stackrel{\text { (NS) }}{=}
    		& \bigwedge_{b \in L}\left(\left(\underset{C \in \mathscr{C}}{\bigvee}(C \otimes( N^\beta(C, X) \rightarrow b))\right) \rightarrow b\right)\\
    		& =\bigwedge_{b \in L}\underset{C \in \mathscr{C}}{\bigwedge}\left([C \otimes( N^\beta(C, X) \rightarrow b)] \rightarrow b\right)\\
    		& =\bigwedge_{b \in L}\underset{C \in \mathscr{C}}{\bigwedge}\left(C \rightarrow[( N^\beta(C, X) \rightarrow b) \rightarrow b]\right)\\
    		& =\underset{C \in \mathscr{C}}{\bigwedge}\left(C \rightarrow\bigwedge_{b \in L}[( N^\beta(C, X) \rightarrow b) \rightarrow b]\right)\\
    		& =\underset{C \in \mathscr{C}}{\bigwedge}\left(C \rightarrow N^\beta(C, X)\right)\\
    		&=\overline{\mathscr{C}}_{1} (X) .
    	\end{align*}

    \end{proof}

\subsection{The axiomatic characterizations on
	the second pair of $L$-fuzzy $\beta$-covering-based rough approximation operators}\label{section4-2}

At first, we can define $\underline{\mathscr{C}}_{2}$ and $\overline{\mathscr{C}}_{2}$ through special $L$-fuzzy relation.

Based on the definition of $L$-fuzzy $\beta$-covering-based rough approximation operators, we shall show the independent axiom set to characterize the rough approximation operators.

\begin{lemma}\label{l4-1}\cite{li2017theaxiomatic}
Let $\mathscr{C}$ be an L-fuzzy covering on $U.$ Then, the function $R_{\mathscr{C}}: U \times U \longrightarrow L$ defined by
\begin{align*}
R_{\mathscr{C}}(x, y)=\underset{C \in \mathscr{C}}{\bigvee}(C(x) \otimes C(y)),
\end{align*}
is a symmetric $L$-fuzzy relation on $U$. In particular, $R$ is reflexive when $L$ is a complete Heyting algebra.
\end{lemma}

\begin{proposition}
 Let $\mathscr{C}$ be an $L$-fuzzy $\beta$-covering on $U$. Then the following statements hold for any $X \in L^{U}, x \in U$,
\begin{align*}
\overline{\mathscr{C}}_{2} (X)=N^\beta\left(X, R_{\mathscr{C}}(-, x)\right),\\ \underline{\mathscr{C}}_{2} (X)=S^\beta\left(R_{\mathscr{C}}(-, x), X\right).
\end{align*}
\end{proposition}

\begin{proof}	
For any $X \in L^{U}, x \in U$, it follows from  Theorem~\ref{t2-1} (I1), (I3), (I4) and (I6) that
\begin{align*}
     \overline{\mathscr{C}}_{2} (X)(x)
     &=\underset{C \in \mathscr{C}}{\bigvee}[C(x)\otimes\left(\bigvee_{y \in U}(C(y)\otimes X(y))\right)\otimes\beta]\\
     &=\underset{y \in U}{\bigvee}[X(y)\otimes\left(\underset{C\in \mathscr{C}}{\bigvee}
    (C(x)\otimes C(y))\right)\otimes\beta]\\
     &=\underset{y \in U}{\bigvee}[X(y)\otimes R_{\mathscr{C}}(y,x) \otimes\beta]\\
     &=\left(\underset{y \in U}{\bigvee}[X(y)\otimes R_{\mathscr{C}}(y,x)]\right) \otimes\beta\\
     &=N^\beta\left(X, R_{\mathscr{C}}(-, x)\right),
\end{align*}
and
\begin{align*}
\underline{\mathscr{C}}_{2} X(x)
&=\underset{C \in \mathscr{C}}{\bigwedge}[C(x)\rightarrow\left(\bigwedge_{y \in U}(\left(C(y)\otimes\beta\right)\rightarrow X(y))\right)]\\
&=\underset{y \in U}{\bigwedge}\underset{C\in \mathscr{C}}{\bigwedge}[C(x)\rightarrow\left(C(y)\otimes\beta\right)\rightarrow X(y)]\\
&=\underset{y \in U}{\bigwedge}\underset{C\in \mathscr{C}}{\bigwedge}[C(x)\otimes\left(C(y)\otimes\beta\right)\rightarrow X(y)]\\
&=\underset{y \in U}{\bigwedge}\left(\underset{C\in \mathscr{C}}{\bigvee}[\left(C(x)\otimes C(y)\otimes\beta\right)\rightarrow X(y)]\right)\\
&=\underset{y \in U}{\bigwedge}\left(\beta\rightarrow\left(\left(\underset{C\in \mathscr{C}}{\bigvee}C(x)\otimes C(y)\right)\rightarrow X(y)\right)\right)\\
&=\underset{y \in U}{\bigwedge}\left[\beta\rightarrow\left(R_{\mathscr{C}}(y, x)\rightarrow X(y)\right)\right]\\
&=\beta\rightarrow\underset{y \in U}{\bigwedge}\left[R_{\mathscr{C}}(y, x)\rightarrow X(y)\right]\\
&=S^\beta\left(R_{\mathscr{C}}(-, x),X\right).
\end{align*}
\end{proof}

\subsubsection{On upper rough approximation operator $\overline{\mathscr{C}}_{2}$}

 In the following, we give the axiom set to characterize the rough approximation operator $\overline{\mathscr{C}}_{2}$ and then verify they are uncorrelated with each other by means of examples.

 \begin{theorem}
 	Let $L$ be a complete Heyting algebra. Then, there exists an $L$-fuzzy $\beta$-covering $\mathscr{C}$ on $U$ such that $g=\overline{\mathscr{C}}_{2}$ if and only if it satisfies the following propoties:

 	\begin{enumerate}
 		\item [(U3)]
 		$g(X)\ge\beta\otimes X;$
 			
 		\item [(U6)]
 		$\alpha \otimes g(X)=g(\alpha\otimes X);$
 				
 		\item [(U7) ]
 		$g(\underset{t \in T}{\bigvee}X_t)= \underset{t \in T}{\bigvee}g(X_t);$
 		
 		\item [(U8)]
 		$g\left(1_{\{x\}}\right)(y)=g\left(1_{\{y\}}\right)(x)\le\beta;$
 	\end{enumerate}
 	where $X \in L^{U}$ and $\left\{X_{t} \mid t \in T\right\} \subseteq$ $L^{U}, \alpha,\beta \in L, x, y \in U$.	
 \end{theorem}

\begin{proof}
$\Longrightarrow)$: Let $g=\overline{\mathscr{C}}_{2}$ for an $L$-fuzzy $\beta$-covering $\mathscr{C}$ on $U$. Then it follows from Theorem~\ref{t2-1} and Definition~\ref{d3-5} that $g$ satisfies the properties  of (U3) and (U6)-(U8). In particular, $g\left(1_{\{x\}}\right)(y)=g\left(1_{\{y\}}\right)(x)$ and it follows that
\begin{align*}
g\left(1_{\{x\}}\right)(y)
&=\underset{C \in \mathscr{C}}{\bigvee}[C(y) \otimes N^\beta(C,1_{\{x\}})]\\
&=\underset{C \in \mathscr{C}}{\bigvee}[C(y) \otimes \left(\underset{z\in U}{\bigvee}C(z)\otimes 1_{\{x\}}(z)\right)\otimes\beta]\\
&=\left(\underset{C \in \mathscr{C}}{\bigvee}(C(y) \otimes C(x))\right)\otimes\beta\\
&\le 1\otimes \beta\\
&=\beta.	
\end{align*}

$\Longleftarrow)$: Suppose $g$ satisfies conditions (U3) and (U6)-(U8). Then, we construct a subset $\mathscr{C}$ of $L^{U}$ as
$$
\mathscr{C}=\left\{C_{x y} \mid x, y \in U\right\}, \text { where } C_{x y}(z)=\left\{\begin{array}{ll}
g\left(1_{\{x\}}\right)(y), & z \in\{x, y\}, \\
0, & \text { otherwise. }
\end{array}\right.
$$
\begin{enumerate}[(1)]
	\item
	Firstly, $\mathscr{C}$ is an $L$-fuzzy $\beta$-covering on $U$. Indeed, it follows from (U3) that $C_{xx} \in \mathscr{C}$ for any $x\in U$, and thus,  $g\left(1_{\{x\}}\right)(x) \geq \beta\otimes
	1_{\{x\}}(x)=\beta$, that is, $C_{xx}(x)\geq \beta$.
	\item
	Secondly, $\overline{\mathscr{C}}_{2}\left(1_{\{x\}}\right)=g\left(1_{\{x\}}\right)$ holds for each $x \in U$. In general, for any $y \in U$,
	\begin{align*}
	\overline{\mathscr{C}}_{2}\left(1_{\{x\}}\right)(y)
	&=\underset{C \in \mathscr{C}}{\bigvee}[C(y) \otimes N^\beta(C,1_{\{x\}})]\\
	&=\underset{C \in \mathscr{C}}{\bigvee}[C(y) \otimes \left(\underset{z\in U}{\bigvee}C(z)\otimes 1_{\{x\}}(z)\right)\otimes\beta]\\
	&=\left(\underset{C \in \mathscr{C}}{\bigvee}(C(y) \otimes C(x))\right)\otimes\beta\\
	&=\left(\underset{C \in \mathscr{C}}{\bigvee}(C(y) \bigwedge C(x))\right)\bigwedge\beta.	
	\end{align*}	
	Note that if $C \neq C_{x y}\left(C_{x y}=C_{y x}\right.$ by $(U8)$), then $C(x)=0$ or $C(y)=0$. Thus,
	\begin{align*}
	\overline{\mathscr{C}}_{2}\left(1_{\{x\}}\right)(y)
	&=\left(\underset{C \in \mathscr{C}}{\bigvee}(C(y) \bigwedge C(x))\right)\bigwedge\beta\\
	&=C_{x y}(y) \bigwedge C_{x y}(x) \bigwedge\beta\\
	&=g\left(1_{\{x\}}\right)(y) \bigwedge g\left(1_{\{x\}}\right)(y)\bigwedge\beta\\
	&=g\left(1_{\{x\}}\right)(y).	
	\end{align*}
	\item
	Finally, $g=\overline{\mathscr{C}_{2}}$ holds. Indeed, for any $X \in L^{U}$, we have $X=\underset{x \in U}{\bigvee}\left(X(x) \bigwedge 1_{\{x\}}\right)$, then
	\begin{align*}
	\overline{\mathscr{C}}_{2} (X) &=\overline{\mathscr{C}}_{2}\left(\underset{x \in U}{\bigvee}\left(X(x) \bigwedge 1_{\{x\}}\right)\right) \\
&=\underset{x \in U}{\bigvee}\left(X(x) \bigwedge \overline{\mathscr{C}}_{2}\left(1_{\{x\}}\right)\right)\quad(by~~(U6)~~and~~(U7)) \\
	&=\bigvee_{x \in U}\left(X(x) \bigwedge g\left(1_{\{x\}}\right)\right)\\
&=g\left(\underset{x \in U}{\bigvee}\left(X(x) \bigwedge 1_{\{x\}}\right)\right) \quad(by~~(U6)~~and~~(U7)) \\
&=g(X).
	\end{align*}	
\end{enumerate}
\end{proof}

Next, we verify that (U3) and (U6)-(U8) are independent of each other.

\begin{example}\label{e4-4}
Let $U=\{x, y\}, L=\{0,1\}$ and $\beta\in L$.

(1) We define $g: L^{U} \longrightarrow L^{U}$
as
\begin{equation*}
g(X)=\left\{
\begin{aligned}
0_{U}, \quad X=0_{U}, 1_{\{x\}}, \\
1_{\{y\}}, \quad X=1_{U}, 1_{\{y\}}.\\
\end{aligned}
\right
.
\end{equation*}
It is easy to verify that $g$ satisfies (U6)-(U8). However, 
\begin{align*}
	g(1_{\{x\}})=0_U\ngeq \beta\otimes 1_{\{x\}}.
\end{align*}
Thus, $g$ does not satisfy (U3).

(2) We define $g: L^{U} \longrightarrow L^{U}$ as
\begin{equation*}
g(X)=\left\{
\begin{aligned}
0_{U}, \quad &X=0_{U}, \\
1_{U}\otimes\beta, \quad &X=1_{\{y\}}, 1_{U},\\
1_{\{x\}}, \quad &X=1_{\{x\}}.\\
\end{aligned}
\right
.
\end{equation*}
It is easy to verify that $g$ satisfies (U3),(U6) and (U7). However,
\begin{align*}
	g\left(1_{\{x\}}\right)(y)=1_{\{x\}}(y)=0\neq\beta=1_{U}\otimes\beta = g\left(1_{\{y\}}\right)(x).
\end{align*}
Thus, $g$ does not satisfy (U8).	
\end{example}

\begin{example}\label{e4-5}
	Let $U=\{x, y\}$, $\beta\in L$ and $L=\{0, b, 1\}$ be defined by Table~\ref{table1}. $g:$ $L^{U} \longrightarrow L^{U}$ is a mapping defined as
$$
g(X)=\left\{\begin{array}{ll}
0_{U}, & X=0_{U}, \\
1_{U}\otimes\beta, & \text { otherwise. }
\end{array}\right.
$$
It is easy to verify that $g$ satisfies (U3),(U7) and (U8). However, if $\alpha= b$ and for any $X\neq 0_{U}$ and $X\in L^U$,
\begin{align*}
b\otimes g(X)=b\otimes (1_{U}\otimes\beta)=b_{U}\otimes\beta\neq 1_{U}\otimes\beta=g(b\otimes X).
\end{align*}
Thus, $g$ does not satisfy (U6).
\end{example}

\begin{example}\label{e4-6}
Let $U=\{x, y, z\}, \beta\in (0,1]$ and $L=\{0,1\}$. $g: L^{U} \longrightarrow L^{U}$ is a mapping defined as
\begin{equation*}
g(X)=\left\{
\begin{aligned}
0_{U}, \quad &X=0_{U}, \\
1_{\{x\}},\quad &X=1_{\{x\}},\\
1_{\{y\}}, \quad &X=1_{\{y\}},\\
1_{\{z\}}, \quad &X=1_{\{z\}},\\
1_U, \quad &otherwise.\\
\end{aligned}
\right
.
\end{equation*}
It is easy to verify that $g$ satisfies (U3),(U6) and (U8). However,
\begin{align*}
g(1_{\{x\}}\bigvee 1_{\{y\}})=g(1_{\{x,y\}})=1_U\neq 1_{\{x,y\}}= g(1_{\{x\}})\bigvee g(1_{\{y\}}).
\end{align*}
Thus, $g$ does not satisfy (U7).
\end{example}

\subsubsection{On lower rough approximation operator $\underline{\mathscr{C}}_{2}$}
 In the following, we give the axiom set to characterize the rough approximation operator $\underline{\mathscr{C}}_{2}$ and then show their independence from each other.

\begin{theorem}
	Let $L$ be a regular complete Heyting algebra. Then, there exists an $L$-fuzzy $\beta$-covering $\mathscr{C}$ on $U$ such that $\underline{\mathscr{C}}_{2}$ if and only if it satisfies the following propoties:
	
	\begin{enumerate}
		\item [(L3)]
		$g(X)(x)\le\beta\rightarrow X(x);$
		
		\item [(L6)]
		$\alpha \rightarrow g(X)(x)=g(\alpha\rightarrow X)(x);$
		
		\item [(L7) ]
		$g(\underset{t \in T}{\bigwedge}X_t)(x)= \underset{t \in T}{\bigwedge}g(X_t)(x);$
		
		\item [(L8)] $g\left(1_{U-\{x\}}\right)(y)=g\left(1_{U-\{y\}}\right)(x)\ge\neg \beta;$
	\end{enumerate}
	where $X \in L^{U}$ and $\left\{X_{t} \mid t \in T\right\} \subseteq$ $L^{U}, \alpha,\beta \in L, x, y \in U$.	
\end{theorem}

\begin{proof}
	$\Longrightarrow)$: Let $g=\underline{\mathscr{C}}_{2}$ for an $L$-fuzzy $\beta$-covering $\mathscr{C}$ on $U$, then  it follows from Theorem~\ref{t2-1} and Definition~\ref{d3-5} that $g$ satisfies the properties of (L3) and (L6)-(L8).
	In particular, $g\left(1_{U-\{x\}}\right)(y)=g\left(1_{U-\{y\}}\right)(x)$ and it follows that
	\begin{align*}
	g\left(1_{U-\{x\}}\right)(y)
	&=\underset{C \in \mathscr{C}}{\bigwedge}[C(y) \rightarrow S^\beta(C,1_{U-\{x\}})]\\
	&=\underset{C \in \mathscr{C}}{\bigwedge}[C(y) \rightarrow \left((C(x)\otimes\beta)\rightarrow 0 \right)]\\	
	&=\underset{C \in \mathscr{C}}{\bigwedge}[(C(y) \otimes C(x)\otimes\beta)\rightarrow 0 ]\\
	&\ge\neg \beta.
	\end{align*}
	
	$\Longleftarrow)$: Let $g$ satisfy the conditions (L3) and (L6)-(L8). Then, we construct a subset $\mathscr{C}$ of $L^{U}$ as
	$$
	\mathscr{C}=\left\{C_{x y} \mid x, y \in U\right\}, \text { where } C_{x y}(z)=\left\{\begin{array}{ll}
	\neg g\left(1_{U-\{x\}}\right)(y), & z \in\{x, y\}, \\
	0, & \text { otherwise. }
	\end{array}\right.
	$$
	\begin{enumerate}[(1)]
		\item
		Firstly, $\mathscr{C}$ is an $L$-fuzzy $\beta$-covering on $U$.
		It follows from (L3) that
		\begin{align*}
			g\left(1_{U-\{x\}}\right)(x)\le\beta\rightarrow (1_{U-\{x\}})(x)=\beta\rightarrow 0=\neg\beta.
		\end{align*}
		Since $L$ is regular complete residuated lattice, it satisfies the double negation law $\neg\neg\alpha = \alpha$. And thus,
		\begin{align*}
			C_{xx}(x)=\neg g\left(1_{U-\{x\}}\right)(x)\ge\neg\neg\beta=\beta,
		\end{align*}
		that is, $C_{xx} \in \mathscr{C}$ for any $x\in U$.
		\item
		Secondly, $\underline{\mathscr{C}}_{2}\left(1_{U-\{x\}}\right)=g\left(1_{U-\{x\}}\right)$ holds for each $x \in U$. In general, for any $y \in U$,
		\begin{align*}
		\underline{\mathscr{C}}_{2}\left(1_{U-\{x\}}\right)(y)
		&=\underset{C \in \mathscr{C}}{\bigwedge}[C(y) \rightarrow S^\beta(C,1_{U-\{x\}})]\\
		&=\underset{C \in \mathscr{C}}{\bigwedge}[C(y) \rightarrow \left(\underset{z\in U}{\bigwedge}(C(z)\otimes\beta)\rightarrow 1_{U-\{x\}}(z) \right)]\\
		&=\underset{C \in \mathscr{C}}{\bigwedge}[C(y) \rightarrow \left((C(x)\otimes\beta)\rightarrow 0 \right)]\\	
		&=\underset{C \in \mathscr{C}}{\bigwedge}[(C(y) \otimes C(x)\otimes\beta)\rightarrow 0 ]\ (by\ (I3)\ of \ Theorem~\ref{t2-1})\\
		&=\neg\left(\underset{C \in \mathscr{C}}{\bigvee}(C(y) \otimes C(x)\otimes\beta) \right)\\
		&=\neg\left(\underset{C \in \mathscr{C}}{\bigvee}(C(y) \bigwedge C(x)\bigwedge\beta) \right).
		\end{align*}	
		Note that if $C \neq C_{x y}\left(C_{x y}=C_{y x}\right.$ by (L8)), then $C(x)=0$ or $C(y)=0$. Thus,
		\begin{align*}
		\overline{\mathscr{C}}_{2}\left(1_{U-\{x\}}\right)(y)
		&=\neg\left(C_{xy}(y) \bigwedge C_{xy}(x)\bigwedge\beta\right)\\
		&=\neg\left(\neg g\left(1_{\{x\}}\right)(y) \bigwedge \neg g\left(1_{\{x\}}\right)(y)\bigwedge\beta\right)\\	
		&=\neg \neg g\left(1_{\{x\}}\right)(y)\\
		&=g\left(1_{\{x\}}\right)(y).	
		\end{align*}
		\item
		Finally, $g=\overline{\mathscr{C}_{2}}$ holds. Indeed, for any $X \in L^{U}$, $X=\bigwedge_{x \in U}\left(\neg X(x) \rightarrow 1_{X-\{x\}}\right)$, then
		\begin{align*}
		\overline{\mathscr{C}}_{2} (X) &=\overline{\mathscr{C}}_{2}\left(\bigwedge_{x \in U}\left(\neg X(x) \rightarrow 1_{X-\{x\}}\right)\right) \stackrel{(\mathrm{L} 6, \mathrm{L} 7)}{=} \underset{x \in U}{\bigwedge}\left(\neg X(x) \bigwedge \overline{\mathscr{C}}_{2}\left(1_{U-\{x\}}\right)\right) \\
		&=\bigwedge_{x \in U}\left(\neg X(x) \bigwedge g\left(1_{\{x\}}\right)\right) \stackrel{(\mathrm{L} 6, \mathrm{L} 7)}{=} g\left(\underset{x \in U}{\bigwedge}\left(\neg X(x) \bigwedge 1_{U-\{x\}}\right)\right)=g(X).
		\end{align*}	
	\end{enumerate}
\end{proof}

In a similar way, it can be proven that the axiom set (L3) and (L6)-(L8) is independent.

\begin{remark}\label{r4-2}
By the new double negation law as Remark~\ref{r4-1}, we can derive the duality between the upper and lower rough approximation operators $\overline{\mathscr{C}}_{2}$ and $\underline{\mathscr{C}}_{2}$.
\end{remark}

\begin{theorem}
    Let $\mathscr{C}$ be an $L$-fuzzy $\beta$-covering on $U$. Then, for any $X \in L^{U}$, it holds that $$ \underline{\mathscr{C}}_{2} (X)=\underset{b\in L}{\bigwedge}\left(\overline{\mathscr{C}_{2}}(X \rightarrow b) \rightarrow b\right).$$
\end{theorem}

\subsection{The axiomatic characterizations on
	the third pair of $L$-fuzzy $\beta$-covering-based rough approximation operators}\label{section4-3}
\subsubsection{On upper rough approximation operator $\overline{\mathscr{C}}_{3}$}

 In the following, we give the axiom set to characterize the rough approximation operator $\overline{\mathscr{C}}_{3}$ and then verify they are uncorrelated with each other by means of examples.

\begin{theorem}
	Let $L$ be a complete Heyting algebra. Then, there exists an $L$-fuzzy $\beta$-covering $\mathscr{C}$ on $U$ such that $g=\overline{\mathscr{C}}_{3}$ if and only if it satisfies the following properties:
	
	\begin{enumerate}
		\item [(U3)]
		$g(X)>\beta\otimes X;$
		\item [(U6) ]
		$\alpha \otimes g(X)=g(\alpha\otimes X);$
		\item [(U7)]
		$g(\underset{t \in T}{\bigvee}X_t)= \underset{t \in T}{\bigvee}g(X_t);$
		\item [(U9)]
		$gg(X)\le\beta\otimes g(X);$
	\end{enumerate}
	where $X \in L^{U}$ and $\left\{X_{t} \mid t \in T\right\} \subseteq$ $L^{U}, \alpha,\beta \in L, x, y \in U$.	
\end{theorem}
$\Longrightarrow)$: Let $g=\overline{\mathscr{C}}_{3}$ for an $L$-fuzzy $\beta$-covering $\mathscr{C}$ on $U$, then  it follows from Theorem~\ref{t2-1} and Definition~\ref{d3-6} that $g$ satisfies the properties of (U3), (U6) and (U7). Further, for any $X\in L^U$, we have that
\begin{align*}
	gg(X)(x) &= \left(\underset{y \in U}{\bigvee}(R_{\mathscr{C}}(y,x)\otimes g(X)(y))\right)\otimes\beta\\
	&= \left(\underset{y \in U}{\bigvee}(R_{\mathscr{C}}(y,x)\otimes \left(\left(\underset{z \in U}{\bigvee}(R_{\mathscr{C}}(z,y)\otimes X(z))\right)\otimes\beta\right)\right)\otimes\beta\\
		&= \underset{y \in U}{\bigvee}\underset{z \in U}{\bigvee}(R_{\mathscr{C}}(y,x)\otimes R_{\mathscr{C}}(z,y)\otimes X(z))\otimes\beta\otimes\beta\\
		&\le \left(\underset{z \in U}{\bigvee}R_{\mathscr{C}}(z,x)\otimes X(z)\right)\otimes\beta\otimes\beta \\
		&= N^\beta(R_{\mathscr{C}}(-,x),X)\otimes\beta\\
		&= \overline{\mathscr{C}}_{3}\otimes\beta\\
		&= g(X)(x)\otimes\beta.		
 \end{align*}
In particular, combining (U3) and (U9) shows that $gg(X)(x)=g(X)(x)\otimes\beta$ for any $X\in L^U$ and $x\in U$.

$\Longleftarrow)$: Let $g$ satisfy the conditions (U3) and (U6)-(U8). Then, we construct a subset $\mathscr{C}$ of $L^{U}$ as
$\mathscr{C}=\{g(1_{\{x\}})|x\in X\}$.
\begin{enumerate}[(1)]
	\item
	Firstly, $\mathscr{C}$ is an $L$-fuzzy $\beta$-covering on $U$.
	It follows from (U3) that $g\left(1_{\{x\}}\right)\in \mathscr{C}$ and
	\begin{align*}
	g\left(1_{\{x\}}\right)(x)\ge\beta\otimes (1_{\{x\}})(x)=\beta.
	\end{align*}
	\item
	Secondly, $\overline{\mathscr{C}}_{3}\left(1_{\{x\}}\right)=g\left(1_{\{x\}}\right)$ holds for each $x \in U$.
	\begin{align*}
		\overline{\mathscr{C}}_{3}\left(1_{\{x\}}\right)(y)
		=\left(\bigvee_{z \in X}\left(R_{\mathscr{C}}(z, y)\otimes 1_{\{x\}}(z)\right)\right)\otimes\beta
		=R_{\mathscr{C}}(x, y)\otimes\beta
		=\left(\underset{{C \in \mathscr{C}}}{\bigwedge}(C(x) \rightarrow C(y))\right)\otimes\beta.
	\end{align*}
	
	Note that for any $X \in L^{U}$, we have $X=\underset{{w \in U}}{\bigvee}\left(X(w)\otimes 1_{\{w\}}\right)$, then it holds that for any $z \in U$,
	\begin{align*}
	\underset{x \in U}{\bigvee}\left(g\left(1_{\{z\}}\right)(x)\otimes g\left(1_{\{x\}}\right)\right)
&= \underset{x \in U}{\bigvee}\left(g\left(g\left(1_{\{z\}}\right)(x)\otimes 1_{\{x\}}\right)\right)\quad(by~~(U6)) \\
	& = g\left(\bigvee_{x \in U}\left(g\left(1_{\{z\}}\right)(x)\otimes 1_{\{x\}}\right)\right)\quad(by~~(U7)) \\
	&=g g\left(1_{\{z\}}\right)\\
 &\le (g\left(1_{\{z\}} \right))\otimes\beta \quad(by~~(U9)).
	\end{align*}
	
	Hence, it holds that $g\left(1_{\{z\}}\right)(x)\otimes g\left(1_{\{x\}}\right)(y) \leq g\left(1_{\{z\}}\right)(y)\otimes\beta$ for any $z \in U$. Further, it follows from (I7) that
	\begin{align*}
			g\left(1_{\{x\}}\right)(y)
			&\leq \underset{{z \in U}}{\bigwedge}\left(g\left(1_{\{z\}}\right)(x) \rightarrow (g\left(1_{\{z\}}\right)(y)\right)\otimes\beta)\\
			&\leq \left(\underset{{z \in U}}{\bigwedge}\left(g\left(1_{\{z\}}\right)(x) \rightarrow g\left(1_{\{z\}}\right)(y)\right)\right)\bigwedge \left(\underset{{z \in U}}{\bigwedge}(g\left(1_{\{z\}}\right)(x) \rightarrow\beta)\right)\\
			&= \left(\underset{{z \in U}}{\bigwedge}\left(g\left(1_{\{z\}}\right)(x) \rightarrow g\left(1_{\{z\}}\right)(y)\right)\right)\bigwedge \left(\underset{{z \in U}}{\bigvee}g\left(1_{\{z\}}\right)(x) \rightarrow\beta\right)\\
			&= \left(\underset{{z \in U}}{\bigwedge}\left(g\left(1_{\{z\}}\right)(x) \rightarrow g\left(1_{\{z\}}\right)(y)\right)\right)\bigwedge \beta\\
			&= \left(\underset{{z \in U}}{\bigwedge}\left(g\left(1_{\{z\}}\right)(x) \rightarrow g\left(1_{\{z\}}\right)(y)\right)\right)\otimes \beta\\
			&{\leq} \left(g\left(1_{\{x\}}\right)(x) \rightarrow g\left(1_{\{x\}}\right)(y)\right)\otimes\beta\quad(by~~taking~~z=x)\\
			&\le \beta\otimes\left(\beta \rightarrow g\left(1_{\{x\}}\right)(y)\right)\\
			&\le g\left(1_{\{x\}}\right)(y).
	\end{align*}
In conclusion,
\begin{align*}
	g\left(1_{\{x\}}\right)(y)= \left(\underset{{z \in U}}{\bigwedge}\left(g\left(1_{\{z\}}\right)(x) \rightarrow g\left(1_{\{z\}}\right)(y)\right)\right)\otimes \beta=\left(\underset{C\in\mathscr{C} }{\bigwedge}(C(x) \rightarrow C(y))\right)\otimes\beta=\overline{\mathscr{C}}_{3}\left(1_{\{x\}}\right)(y).
\end{align*}

	Therefore, we obtain that $\overline{\mathscr{C}}_{3}\left(1_{\{x\}}\right)=g\left(1_{\{x\}}\right)$.
	
	\item
	Finally, $g=\overline{\mathscr{C}_{3}}$. For any $X \in L^{U}$, it holds that
	\begin{align*}
	\overline{\mathscr{C}}_{3} (X)
	&=\overline{\mathscr{C}}_{3}\left(\underset{x \in U}{\bigvee}\left(X(x)\otimes 1_{\{x\}}\right)\right) \\
&=\underset{x \in U}{\bigvee}\left(X(x)\otimes \overline{\mathscr{C}}_{3}\left(1_{\{x\}}\right)\right)\quad(by~~(U6)~~and~~(U7)\\
	&=\bigvee_{x \in U}\left(X(x)\otimes g\left(1_{\{x\}}\right)\right) \\
 &=g\left(\bigvee_{x \in U}\left(X(x)\otimes 1_{\{x\}}\right)\right)\\
 &=g(X)\quad(by~~(U6)~~and~~(U7).
	\end{align*}
	
	\begin{example}	
		\begin{enumerate}[(1)]
			\item
			 The operator $g$ in Example~\ref{e4-4} (1) satisfies the conditions (U6), (U7) and (U9) with $\beta=1$, but it does not satisfy the condition (U3).
			\item
			The operator $g$ in Example~\ref{e4-5} satisfies the conditions (U3), (U7) and (U9) with $\beta=1$, but it does not satisfy the condition (U6).
			\item
			The operator $g$ in Example~\ref{e4-6} satisfies the conditions (U3), (U6) and (U9) with $\beta=1$, but it does not satisfy the condition (U7).
		\end{enumerate}
	\end{example}
	
    \begin{example}
    	Let $U=\{x, y, z\}, L=\{0,1\}$ and $\alpha,\beta\in L$. We define $g: L^{U} \longrightarrow L^{U}$ as
    	$$
    	g(X)=\left\{\begin{array}{ll}
    	0_{U}, & X=0_{U}, \\
    	1_{U}, & X=1_{\{y, z\}}, 1_{\{z, x\}}, 1_{\{x, y\}}, 1_{U}, \\
    	1_{\{y, z\}}, & X=1_{\{z\}} ,\\
    	1_{\{z, x\}}, & X=1_{\{x\}}, \\
    	1_{\{x, y\}}, & X=1_{\{y\}}.
    	\end{array}\right.
    	$$
    	It is easily seen that $g$ satisfies the conditions (U3), (U6) and (U7). The following inequality shows that $g$ does not satisfy the condition (U9),
  	    \begin{align*}
  	    	gg(1_{\{x\}})=g(1_{\{z,x\}})=1_{U}\nleq\beta\otimes 1_{\{z,x\}}=\beta\otimes g(1_{\{x\}}).
  	    \end{align*}
    \end{example}	 	
\end{enumerate}

\subsubsection{On lower rough approximation operator $\underline{\mathscr{C}}_{3}$}

In the following, we give the axiom set to characterize the rough approximation operator $\underline{\mathscr{C}}_{3}$ and then verify they are uncorrelated with each other.

\begin{theorem}
	 Let $L$ be a regular complete Heyting algebra. Then, there exists an $L$-fuzzy $\beta$-covering $\mathscr{C}$ on $U$ such that $g=\underline{\mathscr{C}}_{3}$ if and only if it satisfies the following properties:
	
		\begin{enumerate}
		\item [(L3)]
		$g(X)<\beta\rightarrow X;$
		
		\item [(L6)]
		$\alpha \rightarrow g(X)=g(\alpha\rightarrow X);$
		
		\item [(L7)]
		$g(\underset{t \in T}{\bigwedge}X_t)= \underset{t \in T}{\bigwedge}g(X_t);$
		
		\item [(L9)]
		$gg(X)\ge\beta\rightarrow g(X);$
	\end{enumerate}
	where $X \in L^{U}$ and $\left\{X_{t} \mid t \in T\right\} \subseteq$ $L^{U}, \alpha,\beta \in L, x, y \in U$.	
\end{theorem}

\begin{proof}
	$\Longrightarrow)$: Let $g=\underline{\mathscr{C}}_{3}$ for an $L$-fuzzy $\beta$-covering $\mathscr{C}$ on $U$, then  it follows from Theorem~\ref{t2-1} and Definition~\ref{d3-6} that $g$ satisfies the properties of (L3), (L6) and (L7). Further, for any $X\in L^U$, we have that
	\begin{align*}
	gg(X)(x) &= \underset{y \in U}{\bigwedge}(\beta\otimes R_{\mathscr{C}}(y,x)\rightarrow\underline{\mathscr{C}}_{3}(X)(y) )\\
	&= \underset{y \in U}{\bigwedge}[\beta\otimes R_{\mathscr{C}}(y,x)\rightarrow \left(\underset{z \in U}{\bigwedge}(\beta\otimes R_{\mathscr{C}}(z,y))\rightarrow X(z)\right) ]\\		
	&= \underset{y \in U}{\bigwedge}\underset{z \in U}{\bigwedge} [(\beta\otimes\beta\otimes R_{\mathscr{C}}(y,x)\otimes R_{\mathscr{C}}(z,y))\rightarrow X(z) ]\\
	&\ge \underset{z \in U}{\bigwedge} [(\beta\otimes\beta\otimes R_{\mathscr{C}}(z,x))\rightarrow X(z) ]\\
	&= \beta\rightarrow\left(\underset{z \in U}{\bigwedge} ((\beta\otimes R_{\mathscr{C}}(z,x))\rightarrow X(z)) \right)\\
	&= \beta\rightarrow \underline{\mathscr{C}}_{3}(X)(x)\\	
	&= \beta\rightarrow g(X)(x).		
	\end{align*}
	In particular, combining (L3) and (L9), we conclude that $gg(X)(x)=\beta\rightarrow g(X)(x)$ for any $X\in L^U$ and $x\in U$.

$\Longleftarrow)$: Let $g$ satisfy the conditions (L3), (L6), (L7) and (L9). Then, a subset $\mathscr{C}$ of $L^{U}$ is constructed as
\begin{align*}
	\mathscr{C}=\left\{\neg g\left(1_{{U}-\{x\}}\right) \mid x \in U\right\}.
\end{align*}

\begin{enumerate}[(1)]
	\item
	Firstly, $\mathscr{C}$ is an $L$-fuzzy $\beta$-covering on $U.$ It follows from (L3) that, for any $x \in U$,
	\begin{align*}
		\neg g\left(1_{U-\{x\}}\right)(x) \geq\neg(\beta\rightarrow 1_{U-\{x\}}(x))=\beta\otimes\neg (1_{U-\{x\}}(x))=\beta\otimes\neg 0=\beta.
	\end{align*}
	\item
	Secondly, $\underline{\mathscr{C}}_{3}\left(1_{U-\{x\}}\right)=g\left(1_{U-\{x\}}\right)$ holds for each $x \in U$.
	\begin{align*}
		\underline{\mathscr{C}}_{3}\left(1_{U-\{x\}}\right)(y)
		&=\underset{z \in U}{\bigwedge}\left((\beta\otimes R_{\mathscr{C}}(z, y)) \rightarrow 1_{U-\{x\}}(z)\right)\\
		&=\neg (\beta\otimes R_{\mathscr{C}}(x, y))\\
&=\neg\left(\beta\otimes[\underset{C \in \mathscr{C}}{\bigwedge}(C(x) \rightarrow C(y))]\right).
	\end{align*}
	Note that, for any $X \in L^{U}$, we have $X=\bigwedge_{w \in U}\left(\neg X(w) \rightarrow 1_{U-\{w\}}\right)$, then it holds that, for any $z \in U$,
	\begin{align*}
	\underset{x \in U}{\bigwedge}\left(\neg g\left(1_{U-\{z\}}\right)(x) \rightarrow g\left(1_{U-\{x\}}\right)\right) &=\underset{x \in U}{\bigwedge}\left(g\left(\neg g\left(1_{U-\{z\}}\right)(x) \rightarrow 1_{U-\{x\}}\right)\right)\quad(by~~(L6)) \\
	& =g\left(\underset{x \in U}{\bigwedge}\left(\neg g\left(1_{U-\{z\}}\right)(x) \rightarrow 1_{U-\{x\}}\right)\right) \quad(by~~(L7))\\
	&=g g\left(1_{U-\{z\}}\right) \geq \beta\rightarrow g\left(1_{U-\{z\}}\right)\quad(by~~(L9)) .
	\end{align*}
	
	Hence, it holds that $\neg g\left(1_{U-\{z\}}\right)(x) \rightarrow g\left(1_{U-\{x\}}\right)(y) \geq \beta \rightarrow g\left(1_{U-\{z\}}\right)(y)$ for any $z\in U$. Then it follows from Theorem~\ref{t2-1} (I7) that
	\begin{align*}
	(\beta \rightarrow g\left(1_{U-\{z\}}\right)(y)) \rightarrow g\left(1_{U-\{x\}}\right)(y) &\geq \neg g\left(1_{U-\{z\}}\right)(x), \\
	\neg g\left(1_{U-\{x\}}\right)(y) \rightarrow \neg (\beta \rightarrow g\left(1_{U-\{z\}}\right)(y) )&\geq \neg g\left(1_{U-\{z\}}\right)(x), \\
	\neg g\left(1_{U-\{z\}}\right)(x) \rightarrow \neg (\beta \rightarrow g\left(1_{U-\{z\}}\right)(y)) &\geq \neg g\left(1_{U-\{x\}}\right)(y).
	\end{align*}
	
	 In addition, from (U3), one has that
		$g(1_{U-\{x\}})(x)<\beta\rightarrow 1_{U-\{x\}}(x)= \beta\rightarrow 0=\neg \beta$, and then 	$\neg g(1_{U-\{x\}})(x)> \beta$.
	
	Further, we can obtain that
	\begin{align*}
	   \neg g\left(1_{U-\{x\}}\right)(y)
	   &\leq \underset{z \in U}{\bigwedge}\left(\neg g\left(1_{U-\{z\}}\right)(x) \rightarrow \neg (\beta\rightarrow g\left(1_{U-\{z\}}\right)(y)) \right)\\ &=\underset{z \in U}{\bigwedge}\left(\neg g\left(1_{U-\{z\}}\right)(x) \rightarrow  (\beta\otimes \neg g\left(1_{U-\{z\}}\right)(y)) \right)\\
	   &\leq\left(\underset{z \in U}{\bigwedge}\left(\neg g\left(1_{U-\{z\}}\right)(x) \rightarrow  ( \neg g\left(1_{U-\{z\}}\right)(y)) \right)\right)\bigwedge\left( \underset{z \in U}{\bigwedge}\left(\neg g\left(1_{U-\{z\}}\right)(x) \rightarrow \beta \right)\right)\\
	   &=\left(\underset{z \in U}{\bigwedge}\left(\neg g\left(1_{U-\{z\}}\right)(x) \rightarrow  ( \neg g\left(1_{U-\{z\}}\right)(y)) \right)\right)\bigwedge \left(\left(\underset{z \in U}{\bigvee}\neg g\left(1_{U-\{z\}}\right)(x)\right) \rightarrow \beta \right)\\
	   &=\left(\underset{z \in U}{\bigwedge}\left(\neg g\left(1_{U-\{z\}}\right)(x) \rightarrow  ( \neg g\left(1_{U-\{z\}}\right)(y)) \right)\right)\otimes\beta\\
	   &\le\left(\left(\neg g\left(1_{U-\{x\}}\right)(x) \rightarrow  ( \neg g\left(1_{U-\{x\}}\right)(y)) \right)\right)\otimes\beta\quad(by~~ taking ~~z=x)\\
	   &\le\left(\beta\rightarrow  ( \neg g\left(1_{U-\{x\}}\right)(y)) \right)\otimes\beta\\
	   &\le \neg g\left(1_{U-\{x\}}\right)(y).\\
	\end{align*}
	Then, it holds that
	\begin{align*}
    g\left(1_{U-\{x\}}\right)(y)
    &=\neg \left(\left(\underset{z \in U}{\bigwedge}\left(\neg g\left(1_{U-\{z\}}\right)(x) \rightarrow  ( \neg g\left(1_{U-\{z\}}\right)(y)) \right)\right)\otimes\beta\right)\\
    &=\neg \left(\left(\underset{C \in \mathscr{C}}{\bigwedge}\left(C(x) \rightarrow   C(y) \right)\right)\otimes\beta\right)\\	&=\underline{\mathscr{C}}_{3}\left(1_{U-\{x\}}\right)(y).
	\end{align*}
	
	Therefore, we get that $\underline{\mathscr{C}}_{3}\left(1_{U-\{x\}}\right)=g\left(1_{U-\{x\}}\right)$.
	\item
	 Finally, $g=\underline{\mathscr{C}}_{3}.$ For any $X \in L^{U}$, it holds that	
	\begin{align*}
	\underline{\mathscr{C}}_{3} (X) &=\underline{\mathscr{C}}_{3}\left(\bigwedge_{w \in U}\left(\neg X(w) \rightarrow 1_{U-\{w\}}\right)\right) \\
&= \bigwedge_{w \in U}\left(\neg X(w) \rightarrow \underline{\mathscr{C}}_{3}\left(1_{U-\{w\}}\right)\right)\quad (by~~(L6)~~and~~(L7) \\
	&=\bigwedge_{w \in U}\left(\neg X(w) \rightarrow g\left(1_{U-\{w\}}\right)\right)\\
&=g\left(\underset{w \in U}{\bigwedge}\left(\neg X(w) \rightarrow 1_{U-\{w\}}\right)\right)\quad (by~~(L6)~~and~~(L7) \\
&=g(X).
	\end{align*}
\end{enumerate}
\end{proof}

In a similar way, it can be proven that the axiom set (L3), (L6) ,(L7) and (L9) are independent.

\begin{remark}\label{r4-2}
By the new double negation law as Remark~\ref{r4-1}, we can derive the duality between the upper and lower rough approximation operators $\overline{\mathscr{C}}_{3}$ and $\underline{\mathscr{C}}_{3}$.
\end{remark}

\begin{theorem}
	Let $\mathscr{C}$ be an $L$-fuzzy $\beta$-covering on $U$. Then, for any $X \in L^{U}$, it holds that $$ \underline{\mathscr{C}}_{3} (X)=\underset{b\in L}{\bigwedge}\left(\overline{\mathscr{C}_{3}}(X \rightarrow b) \rightarrow b\right).$$
\end{theorem}

\section{The matrix representations on three pairs of $L$-fuzzy $\beta$-covering-based rough approximation operators}\label{section5}

\begin{definition}\label{d5-1}
	Let $\mathscr{C}=\{C_{1}, C_{2}, \ldots, C_{m}\}$ be an $L$-fuzzy $\beta$-covering of
	$U=\{x_{1}, x_{2}, \ldots, x_{n}\}$ and $X\in L^{U}$.
	$M_{\mathscr{C}}=(C_{j}(x_{i}))_{n\times m}$ is said to be a matrix representation of $\mathscr{C}$,
	and $M_{X}=(X(x_{i}))_{n\times 1}$ a matrix representation of $X$.
\end{definition}

In general, if $M=(a_{ij})_{s\times t}$ be a matrix, then its transpose is $M^{T}=(a_{ji})_{t\times s}$.

\begin{example}\label{e5-1}
	Let $U=\{x_{1}, x_{2}, x_{3}, x_{4}, x_{5}, x_{6}\}$ and $L=([0, 1], \otimes, \rightarrow, \bigvee, \bigwedge, 0, 1)$
	be a G\"{o}del-residuated lattice, where
\begin{center}
	$x\otimes y=x\bigwedge y$ and $x\rightarrow y=\left\{
	\begin{array}{ll}
	1, & \hbox{$x\leq y$,} \\
	y, & \hbox{$x>y$.}
	\end{array}
	\right.$
\end{center}
	for any $x, y\in L$.
	A family $\mathscr{C}=\{C_{1}, C_{2}, C_{3}, C_{4}\}$ of
	$L$-fuzzy subsets of $U$ is listed below.
	\begin{align*}
	C_{1}&=\frac{0.7}{x_{1}}+\frac{0.1}{x_{2}}+\frac{0.3}{x_{3}}+\frac{0.5}{x_{4}}+\frac{0.3}{x_{5}}+\frac{0.3}{x_{6}},\
	C_{2}=\frac{0.5}{x_{1}}+\frac{0.7}{x_{2}}+\frac{0.8}{x_{3}}+\frac{0.6}{x_{4}}+\frac{0.4}{x_{5}}+\frac{0.6}{x_{6}},\\
	C_{3}&=\frac{0.6}{x_{1}}+\frac{0.7}{x_{2}}+\frac{0.2}{x_{3}}+\frac{0.3}{x_{4}}+\frac{0.2}{x_{5}}+\frac{0.1}{x_{6}},\
	C_{4}=\frac{0.4}{x_{1}}+\frac{0.3}{x_{2}}+\frac{0.6}{x_{3}}+\frac{0.6}{x_{4}}+\frac{0.7}{x_{5}}+\frac{0.4}{x_{6}}.
	\end{align*}
	According to Definition~\ref{d-covering}, $\mathscr{C}$ is an $L$-fuzzy $\beta$-covering of $U$ for any $\beta\in (0,0.6]$, and $M_{\mathscr{C}}$ and $M_{C_{3}}$ can be expressed as follows.
	\begin{center}
		$M_{\mathscr{C}}=\bordermatrix[{[]}]{
			& C_{1} & C_{2} & C_{3}  & C_{4}  \cr
			x_{1}&   0.7    & 0.5     & 0.6   & 0.4 \cr
			x_{2}&   0.1    & 0.7    & 0.7   & 0.3 \cr
			x_{3}&   0.3    & 0.8     & 0.2   & 0.6 \cr
			x_{4}&   0.5    & 0.6     & 0.3   & 0.6 \cr
			x_{5}&   0.3    & 0.4     & 0.2   & 0.7 \cr
			x_{6}&  0.3    & 0.6     & 0.1   & 0.4 \cr
		}$,\quad
		$M_{C_{3}}=\bordermatrix[{[]}]{
			&  C_{3}   \cr
			x_{1}&   0.6 \cr
			x_{2}&   0.7 \cr
			x_{3}&   0.2 \cr
			x_{4}&   0.3 \cr
			x_{5}&   0.2 \cr
			x_{6}&   0.1 \cr
		}$.
	\end{center}
\end{example}

\begin{definition}\label{d-new operators}
	Let $A=(a_{ik})_{n\times m}$ and $B=(b_{kj})_{m\times l}$ be two lattice valued matrices.
	We stipulate that $C= A\vartriangle B=(c_{ij})_{n\times l}$ and $D= A\blacktriangle B=(d_{ij})_{n\times l}$ as follows,
	\begin{align*}
	c_{ij}&=\underset{k=1}{\overset{m}\bigwedge}(a_{ik}\rightarrow b_{kj}),\ i=1, 2, \ldots, n, j=1, 2, \ldots, l,\\
	d_{ij}&=\underset{k=1}{\overset{m}\bigvee}(a_{ik}\otimes b_{kj}),\ i=1, 2, \ldots, n, j=1, 2, \ldots, l.
	\end{align*}
    When one of the matrices degenerates to a constant $\beta$,
	$C'= \beta\vartriangle B=(c_{ij})_{m\times l}$ and $D'= \beta\blacktriangle B=(d_{ij})_{m\times l}$ can be defined for any $\beta\in L$,
	\begin{align*}
	c_{ij}&=\beta\rightarrow b_{ij},\ i=1, 2, \ldots, m, j=1, 2, \ldots, l,\\
	d_{ij}&=\beta\otimes b_{ij},\ i=1, 2, \ldots, m, j=1, 2, \ldots, l.
	\end{align*}
	In particular, for any $\alpha,\beta\in L$, we have
	\begin{align*}
	\alpha\vartriangle\beta= \alpha\rightarrow \beta\quad \textnormal{and}\quad
	\alpha\blacktriangle\beta = \alpha\otimes\beta.
	\end{align*}
\end{definition}

\begin{proposition}\label{p5-1}
	Let $\mathscr{C}$ be an $L$-fuzzy $\beta$-covering on $U=\{x_{1}, x_{2}, \ldots, x_{n}\}$. Then, for any $A, B\in L^{U}$, the following hold.
	\begin{align*}
	N^{\beta}(A, B)&=N(A, B)\otimes\beta=[(M_{A})^{T}\blacktriangle M_{B}]\blacktriangle \beta=[(M_{B})^{T}\blacktriangle M_{A}]\blacktriangle \beta,\\
	S^{\beta}(A, B)&=\beta\rightarrow S(A, B)=\beta \triangle[(M_{A})^{T}\triangle M_{B}].
	\end{align*}
\end{proposition}

\begin{proof}
	Let $M_{A}=(a_{i})_{n\times 1}$, $M_{B}=(b_{i})_{n\times 1}$,
	$[(M_{A})^{T}\blacktriangle M_{B}]\blacktriangle \beta=c$ and $\beta \triangle [(M_{A})^{T}\triangle M_{B}]=d$. Then, one obtains that
	\begin{align*}
	c&=[(M_{A})^{T}\blacktriangle M_{B}]\blacktriangle \beta\\
	&=[\underset{i=1}{\overset{n}\bigvee}(a_{i}\otimes b_{i})]\blacktriangle\beta\\
	&=[\underset{x_{i}\in U}{\bigvee}A(x_{i})\otimes B(x_{i})]\blacktriangle\beta\\
	&=[\underset{x_{i}\in U}{\bigvee}A(x_{i})\otimes B(x_{i})]\otimes\beta\\
	&=N^\beta(A, B),
	\end{align*}
	and
	\begin{align*}
	d&=\beta \triangle [(M_{A})^{T}\triangle M_{B}]\\
	&=\beta \triangle[\underset{i=1}{\overset{n}\bigwedge}(a_{i}\rightarrow b_{i})]\\
	&=\beta \triangle[\underset{x_{i}\in U}{\bigwedge}(A(x_{i})\rightarrow B(x_{i}))]\\
	&=\beta \rightarrow[\underset{x_{i}\in U}{\bigwedge}(A(x_{i})\rightarrow B(x_{i}))]\\
	&=S^\beta(A, B).
	\end{align*}
	According to (N1) in Lemma~\ref{l2-1}, we get $N^\beta(A, B)=N^\beta(B, A)=[(M_{B})^{T}\blacktriangle M_{A}]\blacktriangle\beta$.
\end{proof}

\begin{example}\label{e5-2}
	Let $U=\{x_{1}, x_{2}, x_{3}, x_{4}, x_{5}, x_{6}\}$ and $L=\{0, a, 1\}$ be a complete lattice with $0<a<1$.
	Define that $\otimes$ be the G\"{o}del-$t$ norm (that is, $a\otimes b=a\bigwedge b$ for all $a, b\in L$)
	and $\rightarrow$ be defined by Table~\ref{table1}.
	It is easy to verify that $L=(L, \otimes, \rightarrow, \bigwedge, \bigvee, 0, 1)$ is a complete residuated lattice. In particular, take $\beta=a$ and
	\begin{align*}
	A&=\frac{0}{x_{1}}+\frac{a}{x_{2}}+\frac{1}{x_{3}}+\frac{0}{x_{4}}+\frac{a}{x_{5}}+\frac{1}{x_{6}},\
	B=\frac{0}{x_{1}}+\frac{1}{x_{2}}+\frac{a}{x_{3}}+\frac{a}{x_{4}}+\frac{1}{x_{5}}+\frac{a}{x_{6}}.
	\end{align*}
	Then
	\begin{align*}
	N^\beta(A, B)&=\left((M_{A})^{T}\blacktriangle M_{B}\right)\blacktriangle\beta\\
	&=\left(\left(\left[
	\begin{array}{c}
	0    \\
	a   \\
	1    \\
	0    \\
	a    \\
	1    \\
	\end{array}
	\right]\right)^{T}\blacktriangle
	\left[
	\begin{array}{c}
	0   \\
	1    \\
	a   \\
	a    \\
	1    \\
	a    \\
	\end{array}
	\right]\right)\blacktriangle\beta\\
	&=a\blacktriangle\beta\\
	&=a\bigwedge a\\
	&=a,
	\end{align*}
	and
	\begin{align*}
	S^\beta(A, B)&=\beta\triangle\left((M_{A})^{T}\triangle M_{B}\right)\\
	&=\beta\triangle\left(\left(\left[
	\begin{array}{c}
		0    \\
	a   \\
	1    \\
	0    \\
	a    \\
	1    \\
	\end{array}
	\right]\right)^{T}\triangle
	\left[
	\begin{array}{c}
	0   \\
	1    \\
	a   \\
	a    \\
	1    \\
	a    \\
	\end{array}
	\right]\right)\\
	&=\beta\rightarrow a\\
	&=a\rightarrow a\\
	&=1.
	\end{align*}
\end{example}

\subsection{The matrix representations on the first pair of $L$-fuzzy $\beta$-covering-based rough approximation operators}\label{first}

In the following, we represent the first pair of $L$-fuzzy $\beta$-covering-based rough approximation operators $\underline{\mathscr{C}}_{1}$
and $\overline{\mathscr{C}}_{1}$ with the help of matrix operations form.
\begin{proposition}
	Let $\mathscr{C}=\{C_{1}, C_{2}, \ldots, C_{m}\}$ be an $L$-fuzzy $\beta$-covering of
	$U=\{x_{1}, x_{2}, \ldots, x_{n}\}$ and $X\in L^{U}$. Then the following statements hold.
	\begin{align*}
	M_{\underline{\mathscr{C}}_{1}(X)}&=M_{\mathscr{C}}\blacktriangle[\beta\triangle((M_{\mathscr{C}})^{T}\triangle M_{X})],\\
	M_{\overline{\mathscr{C}}_{1}(X)}&=M_{\mathscr{C}}\triangle[((M_{\mathscr{C}})^{T}\blacktriangle M_{X})\blacktriangle\beta].
	\end{align*}
\end{proposition}
\begin{proof}
	Let $\ \beta\triangle[(M_{\mathscr{C}})^{T}\triangle M_{X}]=(a_{i})_{m\times 1}$
	and $M_{\mathscr{C}}\blacktriangle[\beta\triangle((M_{\mathscr{C}})^{T}\triangle M_{X})]=(b_{j})_{n\times 1}$.
	Then one concludes that
	\begin{align*}
	a_{i}&=\beta\triangle\left(\underset{l=1}{\overset{n}\bigwedge}[C_{i}(x_{l})\rightarrow X(x_{l})]\right)\\
	&=\beta\triangle\left(\underset{x\in U}{\bigwedge}[C_{i}(x)\rightarrow X(x)]\right)\\
	&=\beta\rightarrow\left(\underset{x\in U}{\bigwedge}[C_{i}(x)\rightarrow X(x)]\right)\\
	&=S^\beta(C_{i}, X),
	\end{align*}
	and
	\begin{align*}
	b_{j}&=\underset{s=1}{\overset{m}\bigvee}[C_{s}(x_{j})\otimes S^\beta(C_{s}, X)]\\
	&=\underset{C\in \mathscr{C}}{\bigvee}[C(x_{j})\otimes S^\beta(C, X)]\\
	&=\underline{\mathscr{C}}_{1}(X)(x_{j}).
	\end{align*}
	Hence, $M_{\underline{\mathscr{C}}_{1}(X)}=M_{\mathscr{C}}\blacktriangle[\beta\triangle ((M_{\mathscr{C}})^{T}\triangle M_{X})]$ holds for any $X\in L^{U}$.
	
	Similarly, denote $[(M_{\mathscr{C}})^{T}\blacktriangle M_{X}]\blacktriangle\beta=(c_{i})_{m\times 1}$
	and $M_{\mathscr{C}}\triangle[((M_{\mathscr{C}})^{T}\blacktriangle M_{X})\blacktriangle\beta]=(d_{j})_{n\times 1}$, then, one has that
	\begin{align*}
	c_{i}&=[\underset{l=1}{\overset{n}\bigvee}(C_{i}(x_{l})\otimes X(x_{l}))]\blacktriangle\beta\\
	&=[\underset{x\in U}{\bigvee}(C_{i}(x)\otimes X(x))]\blacktriangle\beta\\
	&=[\underset{x\in U}{\bigvee}(C_{i}(x)\otimes X(x))]\otimes\beta\\
	&=N^\beta(C_{i}, X),
	\end{align*}
	and
	\begin{align*}
	d_{j}&=\underset{s=1}{\overset{m}\bigwedge}[C_{s}(x_{j})\rightarrow N^\beta(C_{s}, X)]]\\
	&=\underset{C\in \mathscr{C}}{\bigwedge}[C(x_{j})\rightarrow N^\beta(C, X)]\\
	&=\overline{\mathscr{C}}_{1}(X)(x_{j}).
	\end{align*}
	Hence, $M_{\overline{\mathscr{C}}_{1}(X)}=M_{\mathscr{C}}\triangle[((M_{\mathscr{C}})^{T}\blacktriangle M_{X})\blacktriangle\beta]$
	holds for any $X\in L^{U}$.
\end{proof}

\begin{example}\label{e5-3}
	Let $U=\{x_{1}, x_{2}, x_{3}, x_{4}, x_{5}, x_{6}\}$,
	$L=([0, 1], \otimes, \rightarrow, \bigvee, \bigwedge, 0, 1)$
	be a G\"{o}del-residuated lattice,
	$\mathscr{C}=\{C_{1}, C_{2},$\\$ C_{3}, C_{4}\}$ be the $L$-fuzzy $\beta$-covering of $U$ in Example~\ref{e5-1} and
	$X=\frac{0.4}{x_{1}}+\frac{0.3}{x_{2}}+\frac{0.5}{x_{3}}+\frac{0.7}{x_{4}}+\frac{0.4}{x_{5}}+\frac{0.5}{x_{6}}$. Further, taking $\beta=0.6$, we can get the following results.
	\begin{align*}
	M_{\underline{\mathscr{C}}_{1}(X)}&=M_{\mathscr{C}}\blacktriangle\left(\beta\triangle((M_{\mathscr{C}})^{T}\triangle M_{X})\right)\\
	&=\left[\begin{array}{cccc}
	0.7    & 0.5     & 0.6    & 0.4 \\
	0.1    & 0.7    & 0.7   & 0.3 \\
    0.3    & 0.8     & 0.2   & 0.6 \\
    0.5    & 0.6     & 0.3   & 0.6 \\
	0.3    & 0.4     & 0.2   & 0.7\\
	0.3    & 0.6     & 0.1   & 0.4  \\
	\end{array}
	\right]\blacktriangle
	\left(\beta\triangle\left(\left[\begin{array}{cccccc}
	0.7    & 0.1     & 0.3   & 0.5  & 0.3 & 0.3\\
	0.5    & 0.7      & 0.8   & 0.6  & 0.4 & 0.6\\
	0.6      & 0.7     & 0.2   & 0.3  & 0.2 & 0.1\\
	0.4    & 0.3     & 0.6   & 0.6   & 0.7 & 0.4\\
	\end{array}
	\right]\triangle\left[
	\begin{array}{c}
	0.4    \\
	0.3   \\
	0.5   \\
	0.7   \\
	0.4    \\
	0.5   \\
	\end{array}
	\right]\right)\right)\\
	&=\left[
	\begin{array}{c}
	0.4    \\
	0.3    \\
	0.4   \\
	0.4    \\
	0.4  \\
	0.4    \\
	\end{array}
	\right],
	\end{align*}
	and
	\begin{align*}
	M_{\overline{\mathscr{C}}_{1}(X)}&=M_{\mathscr{C}}\triangle\left(((M_{\mathscr{C}})^{T}\blacktriangle M_{X})\blacktriangle\beta\right)\\
	&=\left[\begin{array}{cccc}
	0.7    & 0.5     & 0.6    & 0.4 \\
0.1    & 0.7    & 0.7   & 0.3 \\
0.3    & 0.8     & 0.2   & 0.6 \\
0.5    & 0.6     & 0.3   & 0.6 \\
0.3    & 0.4     & 0.2   & 0.7\\
0.3    & 0.6     & 0.1   & 0.4  \\
	\end{array}
	\right]\triangle
	\left(\left(\left[\begin{array}{cccccc}
	0.7    & 0.1     & 0.3   & 0.5  & 0.3 & 0.3\\
0.5    & 0.7      & 0.8   & 0.6  & 0.4 & 0.6\\
0.6      & 0.7     & 0.2   & 0.3  & 0.2 & 0.1\\
0.4    & 0.3     & 0.6   & 0.6   & 0.7 & 0.4\\
	\end{array}
	\right]\blacktriangle\left[
	\begin{array}{c}
0.4    \\
0.3   \\
0.5   \\
0.7   \\
0.4    \\
0.5   \\
	\end{array}
	\right]\right)\blacktriangle\beta\right)\\
	&=\left[
	\begin{array}{c}
	0.4    \\
	0.4    \\
	0.6   \\
	1    \\
	0.6   \\
	1    \\
	\end{array}
	\right].
	\end{align*}
	Therefore we have
	\begin{align*} \underline{\mathscr{C}}_{1}(X)&=\frac{0.4}{x_{1}}+\frac{0.3}{x_{2}}+\frac{0.4}{x_{3}}
	+\frac{0.4}{x_{4}}+\frac{0.4}{x_{5}}+\frac{0.4}{x_{6}},\quad
	\overline{\mathscr{C}}_{1}(X)=\frac{0.4}{x_{1}}+\frac{0.4}{x_{2}}+\frac{0.6}{x_{3}}
	+\frac{1}{x_{4}}+\frac{0.6}{x_{5}}+\frac{1}{x_{6}}.
	\end{align*}
\end{example}

\subsection{The matrix representations on the second pair of $L$-fuzzy $\beta$-covering-based rough approximation operators}\label{second}
In the following, we represent the second pair of $L$-fuzzy $\beta$-covering-based rough approximation operators $\underline{\mathscr{C}}_{2}$
and $\overline{\mathscr{C}}_{2}$ with the help of matrix operations form.
\begin{proposition}\label{p3-3}
	Let $\mathscr{C}=\{C_{1}, C_{2}, \ldots, C_{m}\}$ be an $L$-fuzzy $\beta$-covering of
	$U=\{x_{1}, x_{2}, \ldots, x_{n}\}$ and $X\in L^{U}$. Then the following statements hold.
	\begin{align*}
	M_{\underline{\mathscr{C}}_{2}(X)}&=M_{\mathscr{C}}\triangle[\beta\triangle((M_{\mathscr{C}})^{T}\triangle M_{X})],\\
	M_{\overline{\mathscr{C}}_{2}(X)}&=M_{\mathscr{C}}\blacktriangle[((M_{\mathscr{C}})^{T}\blacktriangle M_{X})\blacktriangle\beta].
	\end{align*}
\end{proposition}

\begin{proof}
	Let $\beta\triangle[(M_{\mathscr{C}})^{T}\triangle M_{X}]=(a_{i})_{m\times 1}$
	and $M_{\mathscr{C}}\triangle[\beta\triangle((M_{\mathscr{C}})^{T}\triangle M_{X})]=(b_{j})_{n\times 1}$.
	Then one concludes that
	\begin{align*}
	a_{i}&=\beta\triangle[\underset{l=1}{\overset{n}\bigwedge}(C_{i}(x_{l})\rightarrow X(x_{l}))]\\
	&=\beta\triangle[\underset{x\in U}{\bigwedge}(C_{i}(x)\rightarrow X(x))]\\
	&=\beta\rightarrow[\underset{x\in U}{\bigwedge}(C_{i}(x)\rightarrow X(x))]\\
	&=S^\beta(C_{i}, X),
	\end{align*}
	and
	\begin{align*}
	b_{j}&=\underset{s=1}{\overset{m}\bigwedge}[C_{s}(x_{j})\rightarrow S^\beta(C_{s}, X)]\\
	&=\underset{C\in \mathscr{C}}{\bigwedge}[C(x_{j})\rightarrow S^\beta(C, X)]\\
	&=\underline{\mathscr{C}}_{2}(X)(x_{j}).
	\end{align*}
	Hence,
	$M_{\underline{\mathscr{C}}_{2}(X)}=M_{\mathscr{C}}\triangle[\beta\triangle((M_{\mathscr{C}})^{T}\triangle M_{X})]$ holds for any $X\in L^{U}$.
	
	Similarly, denote $[(M_{\mathscr{C}})^{T}\blacktriangle M_{X}]\blacktriangle\beta=(c_{i})_{m\times 1}$
	and $M_{\mathscr{C}}\blacktriangle[((M_{\mathscr{C}})^{T}\blacktriangle M_{X})\blacktriangle\beta]=(d_{j})_{n\times 1}$. Then, one has that
	\begin{align*}
	c_{i}&=[\underset{l=1}{\overset{n}\bigvee}(C_{i}(x_{l})\otimes X(x_{l}))]\blacktriangle\beta\\
	&=[\underset{x\in U}{\bigvee}(C_{i}(x)\otimes X(x))\blacktriangle\beta]\\
	&=N(C_{i}, X)\otimes\beta\\
	&=N^\beta(C_{i}, X),
	\end{align*}
	and
	\begin{align*}
	d_{j}&=\underset{s=1}{\overset{m}\bigvee}[C_{s}(x_{j})\otimes N^\beta(C_{s}, X)]\\
	&=\underset{C\in \mathscr{C}}{\bigvee}[C(x_{j})\otimes N^\beta(C, X)]\\
	&=\overline{\mathscr{C}}_{2}(X)(x_{j}).
	\end{align*}
	Hence, $M_{\overline{\mathscr{C}}_{2}(X)}=M_{\mathscr{C}}\blacktriangle[((M_{\mathscr{C}})^{T}\blacktriangle M_{X})\blacktriangle\beta]$ holds for any $X\in L^{U}$.
\end{proof}

\begin{example}\label{e5-4}	
	Let $U=\{x_{1}, x_{2}, x_{3}, x_{4}, x_{5}, x_{6}\}$,
	$L=([0, 1], \otimes, \rightarrow, \bigvee, \bigwedge, 0, 1)$
	be a G\"{o}del-residuated lattice,
	$\mathscr{C}=\{C_{1}, C_{2},$\\$ C_{3}, C_{4}\}$ be the $L$-fuzzy $\beta$-covering of $U$ in
	Example~\ref{e5-1} and
	$X=\frac{0.4}{x_{1}}+\frac{0.3}{x_{2}}+\frac{0.5}{x_{3}}+\frac{0.7}{x_{4}}+\frac{0.4}{x_{5}}+\frac{0.5}{x_{6}}$.
	Further, taking $\beta=0.6$, we can get the following results.
	\begin{align*}
	M_{\underline{\mathscr{C}}_{2}(X)}&=M_{\mathscr{C}}\triangle\left(\beta\triangle((M_{\mathscr{C}})^{T}\triangle M_{X})\right)\\
	&=\left[\begin{array}{cccc}
	0.7    & 0.5     & 0.6   & 0.4 \\
	0.1    & 0.7     & 0.7   & 0.3 \\
	0.3    & 0.8     & 0.2   & 0.6 \\
	0.5    & 0.6     & 0.3   & 0.6 \\
	0.3    & 0.4     & 0.2   & 0.7 \\
	0.3    & 0.6     & 0.1   & 0.4 \\
	\end{array}
	\right]\triangle
	\left(\beta\triangle\left(\left[\begin{array}{cccccc}
0.7    & 0.1     & 0.3   & 0.5  & 0.3  & 0.3\\
0.5    & 0.7     & 0.8   & 0.6  & 0.4  & 0.6\\
0.6    & 0.7     & 0.2   & 0.3  & 0.2  & 0.1\\
0.4    & 0.3     & 0.6   & 0.6  & 0.7  & 0.4\\
	\end{array}
	\right]\triangle\left[
	\begin{array}{c}
0.4   \\
0.3   \\
0.5   \\
0.7   \\
0.4   \\
0.5   \\
	\end{array}
	\right]\right)\right)\\
	&=\left[
	\begin{array}{c}
	0.3    \\
	0.3    \\
	0.3   \\
	0.3    \\
	0.3   \\
	0.3    \\
	\end{array}
	\right],
	\end{align*}
	and
	\begin{align*}
	M_{\overline{\mathscr{C}}_{2}(X)}&=M_{\mathscr{C}}\blacktriangle\left(((M_{\mathscr{C}})^{T}\blacktriangle M_{X})\blacktriangle\beta\right)\\
	&=\left[\begin{array}{cccc}
	0.7    & 0.5     & 0.6   & 0.4 \\
	0.1    & 0.7     & 0.7   & 0.3 \\
	0.3    & 0.8     & 0.2   & 0.6 \\
	0.5    & 0.6     & 0.3   & 0.6 \\
	0.3    & 0.4     & 0.2   & 0.7 \\
	0.3    & 0.6     & 0.1   & 0.4 \\
	\end{array}
	\right]\blacktriangle
	\left(\left(\left[\begin{array}{cccccc}
   0.7    & 0.1     & 0.3   & 0.5  & 0.3  & 0.3\\
   0.5    & 0.7     & 0.8   & 0.6  & 0.4  & 0.6\\
   0.6    & 0.7     & 0.2   & 0.3  & 0.2  & 0.1\\
   0.4    & 0.3     & 0.6   & 0.6  & 0.7  & 0.4\\
	\end{array}
	\right]\blacktriangle\left[
	\begin{array}{c}
   0.4   \\
   0.3   \\
   0.5   \\
   0.7   \\
   0.4   \\
   0.5   \\
	\end{array}
	\right]\right)\blacktriangle\beta\right)\\
	&=\left[
	\begin{array}{c}
	0.5    \\
	0.6    \\
	0.6    \\
	0.6    \\
	0.6    \\
	0.6    \\
	\end{array}
	\right].
	\end{align*}
	Therefore we have
	\begin{align*} \underline{\mathscr{C}}_{2}(X)&=\frac{0.3}{x_{1}}+\frac{0.3}{x_{2}}+\frac{0.3}{x_{3}}
	+\frac{0.3}{x_{4}}+\frac{0.3}{x_{5}}+\frac{0.3}{x_{6}},\
	\overline{\mathscr{C}}_{2}(X)=\frac{0.5}{x_{1}}+\frac{0.6}{x_{2}}+\frac{0.6}{x_{3}}
	+\frac{0.6}{x_{4}}+\frac{0.6}{x_{5}}+\frac{0.6}{x_{6}}.
	\end{align*}
\end{example}

\subsection{The matrix representations on the third pair of $L$-fuzzy $\beta$-covering-based rough approximation operators}\label{third}
In the following, we represent the third pair of $L$-fuzzy $\beta$-covering-based rough approximation operators $\underline{\mathscr{C}}_{3}$
and $\overline{\mathscr{C}}_{3}$ with the help of matrix operations form.
\begin{proposition}\label{p5-4}
	Let $\mathscr{C}=\{C_{1}, C_{2}, \ldots, C_{m}\}$ be an $L$-fuzzy $\beta$-covering of
	$U=\{x_{1}, x_{2}, \ldots, x_{n}\}$ and $M_{R}=(R_{\mathscr{C}}(x_{i}, x_{j}))_{n\times n}$. Then $M_{R}=M_{\mathscr{C}}\triangle(M_{\mathscr{C}})^{T}$.
\end{proposition}
\begin{proof}
	Let $M_{\mathscr{C}}=(a_{ij})_{n\times m}$, $(M_{\mathscr{C}})^{T}=(a_{ij})_{m\times n}$ and $M_{\mathscr{C}}\triangle(M_{\mathscr{C}})^{T}=(r_{ij})_{n\times n}$. For any $i,j\in\{1, 2, \ldots, n\}$, we get that
	\begin{align*}
	r_{ij}&=\underset{k=1}{\overset{m}\bigwedge}(a_{ik}\rightarrow b_{kj})\\
	&=\underset{k=1}{\overset{m}\bigwedge}[C_{k}(x_{i})\rightarrow C_{k}(x_{j})]\\
	&=\underset{C\in \mathscr{C}}{\bigwedge}[C(x_{i})\rightarrow C(x_{j})]\\
	&=R_{\mathscr{C}}(x_{i}, x_{j}).
	\end{align*}
	Hence, $M_{R}=M_{\mathscr{C}}\triangle(M_{\mathscr{C}})^{T}$ holds.
\end{proof}

\begin{example}\label{e5-5}
	Let $U=\{x_{1}, x_{2}, x_{3}, x_{4}, x_{5}, x_{6}\}$,
	$L=([0, 1], \otimes, \rightarrow, \bigvee, \bigwedge, 0, 1)$
	be a G\"{o}del-residuated lattice and
	$\mathscr{C}=\{C_{1}, C_{2}, C_{3}, C_{4}\}$ be the $L$-fuzzy $\beta$-covering of $U$ in
	Example~\ref{e5-1}.
	Then we conclude that
	\begin{align*}
	M_{R}&=M_{\mathscr{C}}\triangle(M_{\mathscr{C}})^{T}\\
	&=\left[\begin{array}{cccc}
    0.7    & 0.5     & 0.6   & 0.4 \\
    0.1    & 0.7     & 0.7   & 0.3 \\
    0.3    & 0.8     & 0.2   & 0.6 \\
    0.5    & 0.6     & 0.3   & 0.6 \\
    0.3    & 0.4     & 0.2   & 0.7 \\
    0.3    & 0.6     & 0.1   & 0.4 \\
	\end{array}
	\right]\triangle
	\left[\begin{array}{cccccc}
    0.7    & 0.1     & 0.3   & 0.5  & 0.3  & 0.3\\
    0.5    & 0.7     & 0.8   & 0.6  & 0.4  & 0.6\\
    0.6    & 0.7     & 0.2   & 0.3  & 0.2  & 0.1\\
    0.4    & 0.3     & 0.6   & 0.6  & 0.7  & 0.4\\
	\end{array}
	\right]\\
	&=\left[\begin{array}{cccccc}
	1      & 0.1     & 0.2   & 0.3  & 0.2 & 0.1\\
	0.5    & 1       & 0.2   & 0.3  & 0.2 & 0.1\\
	0.4    & 0.1     & 1     & 0.6  & 0.4 & 0.1\\
	0.4    & 0.1     & 0.2   & 1    & 0.2 & 0.1\\
	0.4    & 0.1     & 0.6   & 0.6  & 1   & 0.1\\
	0.5    & 0.1     & 1     & 1    & 0.4 & 1  \\
	\end{array}
	\right].
	\end{align*}
\end{example}
In the definition of $\triangle$ operator, $x\rightarrow y\neq y\rightarrow x$ for some $x, y\in L$. That is the reason why $M_{R}\neq(M_{R})^{T}$ in the example above.

\begin{proposition}\label{p3-5}
	Let $\mathscr{C}=\{C_{1}, C_{2}, \ldots, C_{m}\}$ be an $L$-fuzzy $\beta$-covering of
	$U=\{x_{1}, x_{2}, \ldots, x_{n}\}$ and $X\in L^{U}$. Then the following statements hold.
	\begin{align*}
	M_{\underline{\mathscr{C}}_{3}(X)}
	&=\beta\triangle\left([M_{\mathscr{C}}\triangle(M_{\mathscr{C}})^{T}]^{T}\triangle M_{X}\right),\\
	M_{\overline{\mathscr{C}}_{3}(X)}
	&=\left([M_{\mathscr{C}}\triangle(M_{\mathscr{C}})^{T}]^{T}\blacktriangle M_{X}\right)\blacktriangle\beta.
	\end{align*}
\end{proposition}
\begin{proof}
	Following the notation of Proposition~\ref{p5-4}, note that $M_{\mathscr{C}}\triangle(M_{\mathscr{C}})^{T}=M_{R}$.
	Let $(M_{R})^{T}=(r_{ij})_{n\times n}$, i.e., $r_{ij}=R_{\mathscr{C}}(x_{j}, x_{i})$, $M_{X}=(X(x_{i}))_{n\times 1}$, $\beta\triangle[(M_{R})^{T}\triangle M_{X}]=(a_{i})_{n\times 1}$ and $[(M_{R})^{T}\blacktriangle M_{X}]\blacktriangle\beta=(b_{i})_{n\times 1}$. Then one can conclude that
	\begin{align*}
	a_{i}&=\beta\triangle\left(\underset{j=1}{\overset{n}\bigwedge}[r_{ij}\rightarrow X(x_{j})]\right)\\
	&=\beta\triangle\left(\underset{j=1}{\overset{n}\bigwedge}[R_{\mathscr{C}}(x_{j}, x_{i})\rightarrow X(x_{j})]\right)\\
	&=\beta\triangle\left(\underset{x\in U}{\bigwedge}[R_{\mathscr{C}}(x, x_{i})\rightarrow X(x)]\right)\\
	&=\beta\triangle\left(S(R_{\mathscr{C}}(-, x_{i}), X)\right)\\
	&=\beta\rightarrow\left(S(R_{\mathscr{C}}(-, x_{i}), X)\right)\\
	&=S^\beta(R_{\mathscr{C}}(-, x_{i}), X)\\
	&=\underline{\mathscr{C}}_{3}(X)(x_{i}),
	\end{align*}
	and
	\begin{align*}
	b_{i}&=\left(\underset{j=1}{\overset{n}\bigvee}[r_{ij}\otimes X(x_{j})]\right)\blacktriangle\beta\\
	&=\left(\underset{j=1}{\overset{n}\bigvee}[R_{\mathscr{C}}(x_{j}, x_{i})\otimes X(x_{j})]\right)\blacktriangle\beta\\
	&=\left(\underset{x\in U}{\bigvee}[R_{\mathscr{C}}(x, x_{i})\otimes X(x)]\right)\blacktriangle\beta\\
	&=N(R_{\mathscr{C}}(-, x_{i}), X)\otimes\beta\\
	&=N^\beta(R_{\mathscr{C}}(-, x_{i}), X)\\
	&=\overline{\mathscr{C}}_{3}(X)(x_{i}).
	\end{align*}
	Hence, we obtain that 	$M_{\underline{\mathscr{C}}_{3}(X)}
	=\beta\triangle\left([M_{\mathscr{C}}\triangle(M_{\mathscr{C}})^{T}]^{T}\triangle M_{X}\right)$ and
	$M_{\overline{\mathscr{C}}_{3}(X)}
	=\left([M_{\mathscr{C}}\triangle(M_{\mathscr{C}})^{T}]^{T}\blacktriangle M_{X}\right)\blacktriangle\beta$.
\end{proof}

\begin{example}\label{e3-6}
	Let $U=\{x_{1}, x_{2}, x_{3}, x_{4}, x_{5}, x_{6}\}$,
	$L=([0, 1], \otimes, \rightarrow, \bigvee, \bigwedge, 0, 1)$
	be a G\"{o}del-residuated lattice,
	$\mathscr{C}=\{C_{1}, C_{2},$\\$ C_{3}, C_{4}\}$ be the $L$-fuzzy $\beta$-covering of $U$ in
	Example~\ref{e5-1} and
	$X=\frac{0.4}{x_{1}}+\frac{0.3}{x_{2}}+\frac{0.5}{x_{3}}+\frac{0.7}{x_{4}}+\frac{0.4}{x_{5}}+\frac{0.5}{x_{6}}$. Further, taking $\beta=0.6$, we can get the following results.
	\begin{align*}
	M_{\underline{\mathscr{C}}_{3}(X)}&=\beta\triangle\left((M_{\mathscr{C}}
\triangle(M_{\mathscr{C}})^{T})^{T}\triangle M_{X}\right)\\
	&=\beta\triangle\left(\left(\left[\begin{array}{cccc}
	0.7    & 0.5     & 0.6   & 0.4 \\
	0.1    & 0.7     & 0.7   & 0.3 \\
	0.3    & 0.8     & 0.2   & 0.6 \\
	0.5    & 0.6     & 0.3   & 0.6 \\
	0.3    & 0.4     & 0.2   & 0.7 \\
	0.3    & 0.6     & 0.1   & 0.4 \\
	\end{array}
	\right]\triangle
	\left[\begin{array}{cccccc}
    0.7    & 0.1     & 0.3   & 0.5  & 0.3  & 0.3\\
    0.5    & 0.7     & 0.8   & 0.6  & 0.4  & 0.6\\
    0.6    & 0.7     & 0.2   & 0.3  & 0.2  & 0.1\\
    0.4    & 0.3    & 0.6   & 0.6  & 0.7  & 0.4\\
	\end{array}
	\right]\right)^{T}\triangle
	\left[
	\begin{array}{c}
	0.4   \\
	0.3   \\
	0.5   \\
	0.7   \\
	0.4   \\
	0.5   \\
	\end{array}
	\right]\right)\\
	&=\left[
	\begin{array}{c}
	0.3    \\
	0.3    \\
	0.4   \\
	0.4    \\
	0.4   \\
	0.5   \\
	\end{array}
	\right],
	\end{align*}
	and
	\begin{align*}
	M_{\overline{\mathscr{C}}_{3}(X)}&=\left((M_{\mathscr{C}}\triangle(M_{\mathscr{C}})^{T})^{T}\blacktriangle M_{X}\right)\blacktriangle\beta\\
	&=\left(\left(\left[\begin{array}{cccc}
    0.7    & 0.5     & 0.6   & 0.4 \\
    0.1    & 0.7     & 0.7   & 0.3 \\
    0.3    & 0.8     & 0.2   & 0.6 \\
    0.5    & 0.6     & 0.3   & 0.6 \\
    0.3    & 0.4     & 0.2   & 0.7 \\
    0.3    & 0.6     & 0.1   & 0.4 \\
	\end{array}
	\right]\triangle
	\left[\begin{array}{cccccc}
	0.7    & 0.1     & 0.3   & 0.5  & 0.3  & 0.3\\
	0.5    & 0.7     & 0.8   & 0.6  & 0.4  & 0.6\\
	0.6    & 0.7     & 0.2   & 0.3  & 0.2  & 0.1\\
	0.4    & 0.3     & 0.6   & 0.6  & 0.7  & 0.4\\
	\end{array}
	\right]\right)^{T}\blacktriangle
	\left[
	\begin{array}{c}
	0.4   \\
	0.3   \\
	0.5   \\
	0.7   \\
	0.4   \\
	0.5   \\
	\end{array}
	\right]\right)\blacktriangle\beta\\
	&=\left[
	\begin{array}{c}
   0.5   \\
   0.3   \\
   0.5   \\
   0.6   \\
   0.4   \\
   0.5   \\
	\end{array}
	\right].
	\end{align*}
	Hence, we have
	\begin{align*} \underline{\mathscr{C}}_{3}(X)&=\frac{0.3}{x_{1}}+\frac{0.3}{x_{2}}+\frac{0.4}{x_{3}}
	+\frac{0.4}{x_{4}}+\frac{0.4}{x_{5}}+\frac{0.5}{x_{6}},\
	\overline{\mathscr{C}}_{3}(X)=\frac{0.5}{x_{1}}+\frac{0.3}{x_{2}}+\frac{0.5}{x_{3}}
	+\frac{0.6}{x_{4}}+\frac{0.4}{x_{5}}+\frac{0.5}{x_{6}}.
	\end{align*}
\end{example}

\section{The interdependency on three pairs of $L$-fuzzy $\beta$-covering-based rough approximation operators}\label{section6}
In this section, we are devoted to exploring the conditions of two $L$-fuzzy $\beta$-coverings generate the same $L$-fuzzy $\beta$-covering-based lower and upper rough approximation operators.
In 2003, Zhu and Wang in~\cite{ZhuWang03Reduction} firstly introduced the notion of reducible element to expound under what conditions two coverings generate the same covering-based lower and upper rough approximation operators. Next, we further expand the notion of reducible element to get the condition under which two $L$-fuzzy $\beta$-coverings generate the same fuzzy covering-based lower and upper rough approximation operators.

\begin{definition}\label{d6-1}
	Let $\mathscr{C}$ be an $L$-fuzzy $\beta$-covering of
	$U$ and $C\in \mathscr{C}$.
	If there exist $C_{1}, C_{2}, \ldots, C_{w}\in \mathscr{C}-\{C\}$
	such that $C=\underset{i=1}{\overset{w}\bigvee}C_{i}$, then
	$C$ is called a reducible element of $\mathscr{C}$; if not, $C$ is an irreducible element of $\mathscr{C}$.
\end{definition}
\begin{example}\label{e6-1}
	Let $U=\{x_{1}, x_{2}, x_{3}, x_{4}, x_{5}\}$ and $L=([0, 1], \otimes, \rightarrow, \bigvee,
	\bigwedge, 0, 1)$  be the {\L}ukasiewicz residuated lattice, where
	\begin{align*}
		x\otimes y=\max\{x+y-1, 0\}\ \mbox{and}\ x\rightarrow y=\min\{1-x+y, 1\}.\quad(\forall \ x, y\in L)
	\end{align*}	
	A family $\mathscr{C}=\{C_{1}, C_{2}, C_{3}, C_{4}\}$ of
	$L$-fuzzy subsets of $U$ is listed as follows.
	\begin{align*}
	C_{1}&=\frac{0.9}{x_{1}}+\frac{0.7}{x_{2}}+\frac{0.2}{x_{3}}+\frac{0.8}{x_{4}}+\frac{0.3}{x_{5}},\
	C_{2}=\frac{0.7}{x_{1}}+\frac{0.9}{x_{2}}+\frac{0.3}{x_{3}}+\frac{0.6}{x_{4}}+\frac{0.8}{x_{5}},\\
	C_{3}&=\frac{1}{x_{1}}+\frac{0.4}{x_{2}}+\frac{0.9}{x_{3}}+\frac{0.9}{x_{4}}+\frac{0.9}{x_{5}},\
	C_{4}=\frac{0.9}{x_{1}}+\frac{0.9}{x_{2}}+\frac{0.3}{x_{3}}+\frac{0.8}{x_{4}}+\frac{0.8}{x_{5}}.
	\end{align*}
	It is easy to see that $\mathscr{C}$ is an $L$-fuzzy $\beta$-covering of $U$ for any $\beta\in(0,0.9]$.
	Since $C_{4}=C_{1}\bigvee C_{2}$, then $C_{4}$ is a reducible
	element of $\mathscr{C}$ and further get
	$\mathscr{C}-\{C_{4}\}$ is also an $L$-fuzzy $\beta$-covering of $U$ for any $\beta\in(0,0.9]$.
\end{example}

\begin{proposition}\label{p6-1}
	Let $\mathscr{C}$ be an $L$-fuzzy $\beta$-covering of
	$U$. If $C$ is a reducible element of $\mathscr{C}$,
	then $\mathscr{C}-\{C\}$ is still an $L$-fuzzy $\beta$-covering of
	$U$.
\end{proposition}
\begin{proof}
	Assume that $\mathscr{C}=\{C, C_{1}, C_{2}, \ldots, C_{n}\}$ is an $L$-fuzzy $\beta$-covering,
	where $C, C_{i}\in L^{U}~(i=1, 2, \ldots, n)$.
	Since $C$ is a reducible element of $\mathscr{C}$,
	we conclude that
	$(\underset{i=1}{\overset{n}\bigvee}C_{i})(x)=((\underset{i=1}{\overset{n}\bigvee}C_{i})\bigvee
	C)(x)\ge\beta$ holds for any $x\in U$. Hence, $\mathscr{C}-\{C\}$ is an $L$-fuzzy $\beta$-covering of
	$U$.
\end{proof}

\begin{proposition}\label{p6-2}
	Let $\mathscr{C}$ be an $L$-fuzzy $\beta$-covering of
	$U$, $C$ be a reducible element of $\mathscr{C}$,
	and $C_{1}\in\mathscr{C}-\{C\}$.
	Then $C_{1}$ is a reducible element of $\mathscr{C}$
	if and only if it is a reducible element of $\mathscr{C}-\{C\}$.
\end{proposition}
\begin{proof}
	Assume that $\mathscr{C}=\{C, C_{1}, C_{2}, \ldots, C_{n}\}$ is an $L$-fuzzy $\beta$-covering,
	where $C, C_{i}\in L^{U}~(i=1, 2, \ldots, n)$ and
	$C$ is a reducible element of $\mathscr{C}$. Then
	there exist $C'_{1}, C'_{2}, \ldots, C'_{t}\in\mathscr{C}-\{C\}~(1<t\leq n)$
	such that $C=\underset{k=1}{\overset{t}\bigvee}C'_{k}$.
	
	$\Longrightarrow)$: If $C_{1}$ is a reducible element of $\mathscr{C}$, then there exist $C''_{1}, C''_{2}, \ldots, C''_{s}\in \mathscr{C}-\{C_{1}\}~(1<s\leq n)$ such that $C_{1}=\underset{l=1}{\overset{s}\bigvee}C''_{l}$.
	\begin{itemize}
		\item
		If $C\notin \{C''_{1}, C''_{2}, \ldots, C''_{s}\}$, then $C_{1}$ is a reducible element of $\mathscr{C}-\{C\}$.
		\item
		If $C\in \{C''_{1}, C''_{2}, \ldots, C''_{s}\}$, then there
		exists $r\in\{1, 2, \ldots, s\}$ such that $C=C''_{r}$. Further, it holds that
		\begin{align*}
			C_{1}=(\underset{((l=1)\bigwedge(l\neq r))}{\overset{s}\bigvee}C''_{l})\bigvee C=(\underset{((l=1)\bigwedge(l\neq r))}{\overset{s}\bigvee}C''_{l})\bigvee (\underset{k=1}{\overset{t}\bigvee}C'_{k}).
		\end{align*}
		 Therefore, $C_{1}$ is a reducible element of $\mathscr{C}-\{C\}$ due to $1<s, t\leq n$.
	\end{itemize}

	$\Longleftarrow)$: It is obvious.
\end{proof}

Proposition ~\ref{p6-1} ensures that after a reducible element is removed from an $L$-fuzzy $\beta$-covering, it remains an $L$-fuzzy $\beta$-covering.
And Proposition ~\ref{p6-2} emphasizes that deleting a reducible element from an $L$-fuzzy $\beta$-covering will not affect the existence of other reducible elements, that is, it will not create any new reducible elements, and will not make the original reducible elements into irreducible elements.
Therefore, the reduct of an $L$-fuzzy $\beta$-covering is calculated by directly deleting all the a reducible elements simultaneously or one by one.

\begin{definition}\label{d6-2}
	Let $\mathscr{C}$ be an $L$-fuzzy $\beta$-covering of $U$ and $\mathscr{D}$ be a subset of $\mathscr{C}$.
	If $\mathscr{C}-\mathscr{D}$ is the set of all reducible elements of $\mathscr{C}$, then $\mathscr{D}$
	is called the reduct of $\mathscr{C}$, and is denoted as $\mathscr{R}(\mathscr{C})$.
\end{definition}

\begin{example}\label{e6-2}
	Let $\mathscr{C}$ be the $L$-fuzzy $\beta$-covering in Example~\ref{e6-1}. Then
	$\mathscr{R}(\mathscr{C})=\{C_{1}, C_{2}, C_{3}\}$.
\end{example}

In 2018, Yang and Hu \cite{yang2018matrix} discussed the condition under which two $L$-fuzzy coverings can generate the same three pairs of fuzzy covering-based lower or upper rough approximation operators. Based on their work, we further extend the conclusion to the case of $L$-fuzzy $\beta$-coverings and then give the condition under which two $L$-fuzzy $\beta$-coverings can generate the same three pairs of fuzzy covering-based lower or upper rough approximation operators.

\begin{lemma}\label{l6-1}
	Let $\mathscr{C}$ be an $L$-fuzzy $\beta$-covering of $U$ and $C$ be a reducible element of $\mathscr{C}$.
	Then, for any $X\in L^{U}$ and $x\in U$, the following statements hold.
	\begin{align*}
	\underset{C'\in\mathscr{C}}{\bigvee}[C'(x)\otimes S^\beta(C', X)]&=\underset{C''\in\mathscr{C}-\{C\}}{\bigvee}[C''(x)\otimes S^\beta(C'', X)],\\
	\underset{C'\in\mathscr{C}}{\bigwedge}[C'(x)\rightarrow N^\beta(C', X)]&=\underset{C''\in\mathscr{C}-\{C\}}{\bigwedge}[C''(x)\rightarrow N^\beta(C'', X)].
	\end{align*}
\end{lemma}
\begin{proof}
	Assume $\mathscr{C}=\{C, C_{1}, C_{2}, \ldots, C_{n}\}$,
	where $C, C_{i}\in L^{U}~(i=1, 2, \ldots, n)$ and
	$C$ is a reducible element of $\mathscr{C}$. Then
	there exist $C'_{1}, C'_{2}, \ldots, C'_{t}\in\mathscr{C}-\{C\}~(1<t\leq n)$
	such that $C=\underset{k=1}{\overset{t}\bigvee}C'_{k}$.
	For any $X\in L^{U}$ and $x\in U$, we obtain that
	\begin{align*}
	\underset{C'\in\mathscr{C}}{\bigvee}[C'(x)\otimes S^\beta(C', X)]&=\left(\underset{i=1}{\overset{n}\bigvee}[C_{i}(x)\otimes S^\beta(C_{i}, X)]\right)\bigvee[C(x)\otimes S^\beta(C, X)]\\
	&=\left(\underset{i=1}{\overset{m}\bigvee}[C_{i}(x)\otimes S^\beta(C_{i}, X)]\right)\bigvee\left(\underset{k=1}{\overset{t}\bigvee}[C'_{k}(x)\otimes S^\beta(C, X)]\right)\\
	&=\left(\underset{i=1}{\overset{m}\bigvee}[C_{i}(x)\otimes S^\beta(C_{i}, X)]\right)\bigvee\left(\underset{k=1}{\overset{t}\bigvee}[C'_{k}(x)\otimes S^\beta(C, X)]\right)\\
	&=\underset{i=1}{\overset{m}\bigvee}[C_{i}(x)\otimes S^\beta(C_{i}, X)],
	\end{align*}
	and
	\begin{align*}
	\underset{C'\in\mathscr{C}}{\bigwedge}[C'(x)\rightarrow S(C', X)]&=\left(\underset{i=1}{\overset{n}\bigwedge}[C_{i}(x)\rightarrow N^\beta(C_{i}, X)]\right)\bigwedge\left([C(x)\rightarrow N^\beta(C, X)]\right)\\
	&=\left(\underset{i=1}{\overset{n}\bigwedge}[C_{i}(x)\rightarrow N^\beta(C_{i}, X)]\right)\bigwedge\left([\underset{k=1}{\overset{t}\bigvee}C'_{k}(x)\rightarrow N^\beta(C, X)]\right)\\
	&=\left(\underset{i=1}{\overset{n}\bigwedge}[C_{i}(x)\rightarrow N^\beta(C_{i}, X)]\right)\bigwedge\left(\underset{k=1}{\overset{t}\bigwedge}[C'_{k}(x)\rightarrow N^\beta(C, X)]\right)\\
	&=\underset{i=1}{\overset{n}\bigwedge}[C_{i}(x)\rightarrow N^\beta(C_{i}, X)].
	\end{align*}
	Hence, we conclude that
	$$\underset{C'\in\mathscr{C}}{\bigvee}[C'(x)\otimes S^\beta(C', X)]=\underset{C''\in\mathscr{C}-\{C\}}{\bigvee}[C''(x)\otimes S^\beta(C'', X)]$$ and
	$$\underset{C'\in\mathscr{C}}{\bigwedge}[C'(x)\rightarrow N^\beta(C', X)]=\underset{C''\in\mathscr{C}-\{C\}}{\bigwedge}[C''(x)\rightarrow N^\beta(C'', X)].$$
\end{proof}

\begin{lemma}\label{l6-2}
	Let $\mathscr{C}$ be an $L$-fuzzy $\beta$-covering of $U$.
	Then, for any $X\in L^{U}$ and $x\in U$, the following statements hold.
	\begin{align*}
	\underset{C\in\mathscr{C}}{\bigvee}[C(x)\otimes S^\beta(C, X)]&=\underset{C'\in\mathscr{R}(\mathscr{C})}{\bigvee}[C'(x)\otimes S^\beta(C', X)],\\
	\underset{C\in\mathscr{C}}{\bigwedge}[C(x)\rightarrow N^\beta(C, X)]&=\underset{C'\in\mathscr{R}(\mathscr{C})}{\bigwedge}[C'(x)\rightarrow N^\beta(C', X)].
	\end{align*}
\end{lemma}
\begin{proof}
	It follows immediately from Definition~\ref{d4-2} and Lemma~\ref{l6-1}.
\end{proof}
According to Definition~\ref{d3-4}, Lemmas~\ref{l6-1} and~\ref{l6-2}, the following two propositions can be established

\begin{proposition}\label{p6-3}
	Let $\mathscr{C}_{1}, \mathscr{C}_{2}$ be two $L$-fuzzy $\beta$-coverings of $U$. For any
	$X\in L^{U}$,
	$\mathscr{C}_{1}, \mathscr{C}_{2}$ generate the same first type of fuzzy covering-based
	lower rough approximation of $X$ if and only if $\mathscr{R}(\mathscr{C}_{1})=\mathscr{R}(\mathscr{C}_{2})$.
\end{proposition}

\begin{proposition}\label{p6-4}
	Let $\mathscr{C}_{1}, \mathscr{C}_{2}$ be two $L$-fuzzy $\beta$-coverings of $U$. For each
	$X\in L^{U}$,
	$\mathscr{C}_{1}, \mathscr{C}_{2}$ generate the same first type of fuzzy covering-based
	upper rough approximation of $X$ if and only if $\mathscr{R}(\mathscr{C}_{1})=\mathscr{R}(\mathscr{C}_{2})$.
\end{proposition}

Proposition~\ref{p6-3} and Proposition~\ref{p6-4} give that a necessary and sufficient condition under which two $L$-fuzzy $\beta$-coverings can generate
the same first type of fuzzy covering-based lower and upper rough approximation of an $L$-fuzzy subset is that their reducts are equal, respectively. Further, we obtain the following corollary, which indicates that the first type of $L$-fuzzy $\beta$-covering-based lower and
upper rough approximation of an $L$-fuzzy subset can determine each other.

\begin{corollary}\label{c6-1}
	Let $\mathscr{C}_{1}, \mathscr{C}_{2}$ be two $L$-fuzzy $\beta$-coverings of $U$. For any
	$X\in L^{U}$,
	$\mathscr{C}_{1}, \mathscr{C}_{2}$ generate the same first type of fuzzy covering-based lower rough approximation of $X$ if and only if $\mathscr{C}_{1}, \mathscr{C}_{2}$ generate the same first type of fuzzy covering-based upper rough approximation of $X$.
\end{corollary}

\begin{example}\label{e6-3}
	Let $\mathscr{C}$ be the $L$-fuzzy $\beta$-covering in Example~\ref{e6-1} for any $\beta\in(0,0.9]$ and
	\begin{align*}
	X&=\frac{0.6}{x_{1}}+\frac{0.2}{x_{2}}+\frac{0.5}{x_{3}}+\frac{1}{x_{4}}+\frac{0.1}{x_{5}}.
	\end{align*}
	Let $\beta$=0.9, then we can get that
	\begin{align*}
	M_{\underline{\mathscr{C}}_{1}(X)}&=M_{\mathscr{C}}\blacktriangle[\beta\triangle((M_{\mathscr{C}})^{T}\triangle M_{X})]\\
	&=\left[\begin{array}{cccc}
	0.9    & 0.7     & 1     & 0.9\\
	0.7    & 0.9     & 0.4     & 0.9 \\
	0.2    & 0.3     & 0.9     & 0.3 \\
	0.8    & 0.6     & 0.9     & 0.8 \\
	0.3    & 0.8     & 0.9     & 0.8 \\
	\end{array}
	\right]\blacktriangle
	\left(\beta\triangle\left(\left[\begin{array}{cccccc}
	0.9    & 0.7     & 0.2   & 0.8    & 0.3\\
	0.7    & 0.9     & 0.3   & 0.6    & 0.8 \\
	1      & 0.4     & 0.9   & 0.9    & 0.9 \\
	0.9    & 0.9     & 0.3   & 0.8    & 0.8 \\
	\end{array}
	\right]\triangle\left[
	\begin{array}{c}
	0.6    \\
	0.2    \\
	0.5   \\
	1     \\
	0.1    \\
	\end{array}
	\right]\right)\right)\\
	&=\left[
	\begin{array}{c}
	0.5   \\
	0.3    \\
	0.2   \\
	0.4    \\
	0.2   \\
	\end{array}
	\right],\\
	M_{\overline{\mathscr{C}}_{1}(X)}&=M_{\mathscr{C}}\triangle[((M_{\mathscr{C}})^{T}\blacktriangle M_{X})\blacktriangle\beta]\\
	&=\left[\begin{array}{cccc}
	0.9    & 0.7     & 1     & 0.9\\
	0.7    & 0.9     & 0.4     & 0.9 \\
	0.2    & 0.3     & 0.9     & 0.3 \\
	0.8    & 0.6     & 0.9     & 0.8 \\
	0.3    & 0.8     & 0.9     & 0.8 \\
	\end{array}
	\right]\triangle
	\left(\left(\left[\begin{array}{cccccc}
	0.9    & 0.7     & 0.2   & 0.8    & 0.3\\
	0.7    & 0.9     & 0.3   & 0.6    & 0.8 \\
	1      & 0.4     & 0.9   & 0.9    & 0.9 \\
	0.9    & 0.9     & 0.3   & 0.8    & 0.8 \\
	\end{array}
	\right]\blacktriangle\left[
	\begin{array}{c}
	0.6    \\
	0.2    \\
	0.5   \\
	1     \\
	0.1    \\
	\end{array}
	\right]\right)\blacktriangle\beta\right)\\
	&=\left[
	\begin{array}{c}
	0.8    \\
	0.6    \\
	0.9   \\
	0.9   \\
	0.7   \\
	\end{array}
	\right],\\
	M_{\underline{\mathscr{R}(\mathscr{C})}_{1}(X)}&=M_{\mathscr{R}(\mathscr{C})}\blacktriangle[\beta\triangle((M_{\mathscr{R}(\mathscr{C})})^{T}\triangle M_{X})]\\
	&=\left[\begin{array}{cccc}
	0.9    & 0.7     & 1     \\
	0.7    & 0.9     & 0.4   \\
	0.2    & 0.3     & 0.9   \\
	0.8    & 0.6     & 0.9   \\
	0.3    & 0.8     & 0.9   \\
	\end{array}
	\right]\blacktriangle
	\left(\beta\triangle\left(\left[\begin{array}{cccccc}
	0.9    & 0.7     & 0.2   & 0.8    & 0.3\\
	0.7    & 0.9     & 0.3   & 0.6    & 0.8 \\
	1      & 0.4     & 0.9   & 0.9    & 0.9 \\
	\end{array}
	\right]\triangle\left[
	\begin{array}{c}
	0.6    \\
	0.2    \\
	0.5   \\
	1     \\
	0.1    \\
	\end{array}
	\right]\right)\right)\\
	&=\left[
	\begin{array}{c}
	0.5    \\
	0.3    \\
	0.2    \\
	0.4    \\
	0.2    \\
	\end{array}
	\right],\\
	M_{\overline{\mathscr{R}(\mathscr{C})}_{1}(X)}&=M_{\mathscr{R}(\mathscr{C})}\triangle[((M_{\mathscr{R}(\mathscr{C})})^{T}\blacktriangle M_{X})\blacktriangle\beta]\\
	&=\left[\begin{array}{cccc}
	0.9    & 0.7     & 1     \\
	0.7    & 0.9     & 0.4   \\
	0.2    & 0.3     & 0.9   \\
	0.8    & 0.6     & 0.9   \\
	0.3    & 0.8     & 0.9   \\
	\end{array}
	\right]\triangle
	\left(\left(\left[\begin{array}{cccccc}
	0.9    & 0.7     & 0.2   & 0.8    & 0.3\\
	0.7    & 0.9     & 0.3   & 0.6    & 0.8 \\
	1      & 0.4     & 0.9   & 0.9    & 0.9 \\
	\end{array}
	\right]\blacktriangle\left[
	\begin{array}{c}
	0.6    \\
	0.2    \\
	0.5   \\
	1     \\
	0.1    \\
	\end{array}
	\right]\right)\blacktriangle\beta\right)\\
	&=\left[
	\begin{array}{c}
	0.8   \\
	0.6    \\
	0.9   \\
	0.9  \\
	0.7   \\
	\end{array}
	\right].
	\end{align*}
	Therefore, we have
	\begin{align*} \underline{\mathscr{C}}_{1}(X)&=\underline{\mathscr{R}(\mathscr{C})}_{1}(X)=
	\frac{0.5}{x_{1}}+\frac{0.3}{x_{2}}+\frac{0.2}{x_{3}}
	+\frac{0.4}{x_{4}}+\frac{0.2}{x_{5}},\quad	\overline{\mathscr{C}}_{1}(X)=\overline{\mathscr{R}(\mathscr{C})}_{1}(X)=
	\frac{0.8}{x_{1}}+\frac{0.6}{x_{2}}+\frac{0.9}{x_{3}}
	+\frac{0.9}{x_{4}}+\frac{0.7}{x_{5}}.
	\end{align*}
\end{example}

\begin{definition}\label{d6-3}
	Let $\mathscr{C}$ be an $L$-fuzzy $\beta$-covering of
	$U$ and $C\in \mathscr{C}$.
	If there exist $C_{1}, C_{2}, \ldots, C_{w}\in \mathscr{C}-\{C\}$
	such that $C=\underset{i=1}{\overset{w}\bigwedge}C_{i}$, then $C$ is called an independent element of $\mathscr{C}$;
	if not, $C$ is a dependent element of $\mathscr{C}$.
\end{definition}

\begin{example}\label{e6-4}
	Let $U=\{x_{1}, x_{2}, x_{3}, x_{4}, x_{5}\}$ and $L=([0, 1], \otimes, \rightarrow, \bigvee,
	\bigwedge, 0, 1)$  be the {\L}ukasiewicz residuated lattice.
	A family $\mathscr{C}=\{C_{1}, C_{2}, C_{3}, C_{4}\}$ of
	$L$-fuzzy subsets of $U$ is listed as follows.
	\begin{align*}
	C_{1}&=\frac{0.9}{x_{1}}+\frac{0.9}{x_{2}}+\frac{0.2}{x_{3}}+\frac{0.8}{x_{4}}+\frac{0.3}{x_{5}},\
	C_{2}=\frac{0.7}{x_{1}}+\frac{0.9}{x_{2}}+\frac{0.3}{x_{3}}+\frac{0.6}{x_{4}}+\frac{0.8}{x_{5}},\\
	C_{3}&=\frac{1}{x_{1}}+\frac{0.4}{x_{2}}+\frac{0.9}{x_{3}}+\frac{0.9}{x_{4}}+\frac{0.9}{x_{5}},\
	C_{4}=\frac{0.7}{x_{1}}+\frac{0.9}{x_{2}}+\frac{0.2}{x_{3}}+\frac{0.6}{x_{4}}+\frac{0.3}{x_{5}}.
	\end{align*}
	It is easy to see that $\mathscr{C}$ is an $L$-fuzzy $\beta$-covering on $U$ for any $\beta\in(0,0.9]$. Further,
	 $C_{4}$ is an independent
	element of $\mathscr{C}$ due to $C_{4}=C_{1}\bigwedge C_{2}$, and then
	$\mathscr{C}-\{C_{4}\}$ is also an $L$-fuzzy $\beta$-covering on $U$ for any $\beta\in(0,0.9]$.
\end{example}

\begin{proposition}\label{p6-5}
	Let $\mathscr{C}$ be an $L$-fuzzy $\beta$-covering of
	$U$. If $C$ is an independent element of $\mathscr{C}$,
	then $\mathscr{C}-\{C\}$ is still an $L$-fuzzy $\beta$-covering of $U$.
\end{proposition}
\begin{proof}
	Assume $\mathscr{C}=\{C, C_{1}, C_{2}, \ldots, C_{n}\}$,
	where $C, C_{i}\in L^{U}~(i=1, 2, \ldots, n)$.
	Since $C$ is an independent element of $\mathscr{C}$,
	we have that $((\underset{i=1}{\overset{n}\bigvee}C_{i})\bigvee
	C)(x)=(\underset{i=1}{\overset{n}\bigvee}C_{i})(x)\geq\beta$ holds for each $x\in U$. Hence, $\mathscr{C}-\{C\}$ is an $L$-fuzzy $\beta$-covering of $U$.
\end{proof}

\begin{proposition}\label{p6-6}
	Let $\mathscr{C}$ be an $L$-fuzzy $\beta$-covering of
	$U$, $C$ be an independent element of $\mathscr{C}$,
	and $C_{1}\in\mathscr{C}-\{C\}$.
	Then $C_{1}$ is an independent element of $\mathscr{C}$
	if and only if it is an independent element of $\mathscr{C}-\{C\}$.
\end{proposition}
\begin{proof}
	Assume $\mathscr{C}=\{C, C_{1}, C_{2}, \ldots, C_{n}\}$,
	where $C, C_{i}\in L^{U}~(i=1, 2, \ldots, n)$ and
	$C$ is an independent element of $\mathscr{C}$. Then
	there exist $C'_{1}, C'_{2}, \ldots, C'_{t}\in\mathscr{C}-\{C\}~(1<t\leq n)$
	such that $C=\underset{k=1}{\overset{t}\bigwedge}C'_{k}$.
	
	$\Longrightarrow)$: If $C_{1}$ is an independent element of $\mathscr{C}$, then there exist $C''_{1}, C''_{2}, \ldots, C''_{s}\in \mathscr{C}-\{C_{1}\}~(1<s\leq n)$ such that $C_{1}=\underset{l=1}{\overset{s}\bigwedge}C''_{l}$.

    \begin{itemize}
      \item If $C\notin \{C''_{1}, C''_{2}, \ldots, C''_{s}\}$, then $C_{1}$ is an independent element of $\mathscr{C}-\{C\}$.
      \item If $C\in \{C''_{1}, C''_{2}, \ldots, C''_{s}\}$, then there exists $r\in\{1, 2, \ldots, s\}$ such that $C=C''_{r}$. Further, it holds that \begin{center}
          $C_{1}=(\underset{((l=1)\bigwedge(l\neq r))}{\overset{s}\bigwedge}C''_{l})\bigwedge C=(\underset{((l=1)\bigwedge(l\neq r))}{\overset{s}\bigwedge}C''_{l})\bigwedge(\underset{k=1}
          {\overset{t}\bigwedge}C'_{k})$.
          \end{center}
           Therefore, $C_{1}$ is an independent element of $\mathscr{C}-\{C\}$ due to $1<s, t\leq n$.
    \end{itemize}

	$(\Longleftarrow)$: It is obvious.
\end{proof}

Proposition ~\ref{p6-5} ensures that after an independent element is removed from an $L$-fuzzy $\beta$-covering, it remains an $L$-fuzzy $\beta$-covering.
And Proposition ~\ref{p6-6} emphasizes that deleting an independent element from an $L$-fuzzy $\beta$-covering will not affect the existence of other independent elements, that is, it will not create any new independent elements, and will not make the original independent elements into dependent elements.
Therefore, the core of an $L$-fuzzy $\beta$-covering is calculated by directly deleting all the independent elements simultaneously or one by one.

\begin{definition}\label{d6-4}
	Let $\mathscr{C}$ be an $L$-fuzzy $\beta$-covering of $U$ and $\mathscr{E}$ be a subset of $\mathscr{C}$.
	If $\mathscr{C}-\mathscr{E}$ is the set of all independent elements of $\mathscr{C}$, then $\mathscr{E}$
	is called the core of $\mathscr{C}$, and is denoted as $\mathscr{E}(\mathscr{C})$.
\end{definition}

\begin{example}\label{e6-5}
	Let $\mathscr{C}$ be the $L$-fuzzy $\beta$-covering in Example~\ref{e6-4}. Then
	$\mathscr{E}(\mathscr{C})=\{C_{1}, C_{2}, C_{3}\}$.
\end{example}


\begin{lemma}\label{l6-3}
	Let $\mathscr{C}$ be an $L$-fuzzy $\beta$-covering of $U$ and $C$ be an independent element of $\mathscr{C}$.
	Then, for any $X\in L^{U}$ and $x\in U$, the following statements hold.
	\begin{align*}
	\underset{C'\in\mathscr{C}}{\bigwedge}[C'(x)\rightarrow S^\beta(C', X)]&=\underset{C''\in\mathscr{C}-\{C\}}{\bigwedge}[C''(x)\rightarrow S^\beta(C'', X)],\\
	\underset{C'\in\mathscr{C}}{\bigvee}[C'(x)\otimes N^\beta(C', X)]&=\underset{C''\in\mathscr{C}-\{C\}}{\bigvee}[C''(x)\otimes N^\beta(C'', X)].
	\end{align*}
\end{lemma}
\begin{proof}
	Amusse $\mathscr{C}=\{C, C_{1}, C_{2}, \ldots, C_{n}\}$,
	where $C, C_{i}\in L^{U}~(i=1, 2, \ldots, n)$ and
	$C$ is an independent element of $\mathscr{C}$. Then
	there exist $C'_{1}, C'_{2}, \ldots, C'_{t}\in\mathscr{C}-\{C\}~(1<t\leq n)$
	such that $C=\underset{k=1}{\overset{t}\bigwedge}C'_{k}$. Further, it follows from (I4$'$), (I5$'$), (I6$'$) of Theorem~\ref{t2-1} and Lemma~\ref{l4-1} that, for any $X\in L^{U}$ and $x\in U$, we obtain that
	\begin{align*}
	\underset{C'\in\mathscr{C}}{\bigwedge}[C'(x)\rightarrow S^\beta(C', X)]
    &=\underset{i=1}{\overset{n}\bigwedge}[C_{i}(x)\rightarrow S^\beta(C_{i}, X)]\bigwedge[C(x)\rightarrow S^\beta(C, X)]\\
	&=\underset{i=1}{\overset{n}\bigwedge}[C_{i}(x)\rightarrow S^\beta(C_{i}, X)]\bigwedge[\underset{k=1}{\overset{t}\bigwedge}C'_{k}(x)\rightarrow S^\beta(C, X)]\\
	&=\underset{i=1}{\overset{n}\bigwedge}[C_{i}(x)\rightarrow S^\beta(C_{i}, X)],
	\end{align*}
	and
	\begin{align*}
	\underset{C'\in\mathscr{C}}{\bigvee}[C'(x)\otimes N^\beta(C', X)]
    &=\underset{i=1}{\overset{n}\bigvee}[C_{i}(x)\otimes N^\beta(C_{i}, X)]\bigvee[C(x)\otimes N^\beta(C, X)]\\
	&=\underset{i=1}{\overset{n}\bigvee}[C_{i}(x)\otimes N(C_{i}, X)]\bigvee[\underset{k=1}{\overset{t}\bigwedge}C'_{k}(x)\otimes N^\beta(C, X)]\\
	&=\underset{i=1}{\overset{n}\bigvee}[C_{i}(x)\otimes N^\beta(C_{i}, X)].
	\end{align*}
	Hence, we conclude that $$\underset{C'\in\mathscr{C}}{\bigwedge}[C'(x)\rightarrow S^\beta(C', X)]=\underset{C''\in\mathscr{C}-\{C\}}{\bigwedge}[C''(x)\rightarrow S^\beta(C'', X)]$$ and
	$$\underset{C'\in\mathscr{C}}{\bigvee}[C'(x)\otimes N^\beta(C', X)]=\underset{C''\in\mathscr{C}-\{C\}}{\bigvee}[C''(x)\otimes N^\beta(C'', X)].$$
\end{proof}

\begin{lemma}\label{l6-4}
	Let $\mathscr{C}$ be an $L$-fuzzy $\beta$-covering of $U$.
	Then, for any $X\in L^{U}$ and $x\in U$, the following statements hold.
	\begin{align*}	\underset{C\in\mathscr{C}}{\bigwedge}[C(x)\rightarrow S^\beta(C, X)]&=\underset{C'\in\mathscr{E}(\mathscr{C})}{\bigwedge}[C'(x)\rightarrow S^\beta(C', X)],\\
	\underset{C\in\mathscr{C}}{\bigvee}[C(x)\otimes N^\beta(C, X)]&=\underset{C'\in\mathscr{E}(\mathscr{C})}{\bigvee}[C'(x)\otimes N^\beta(C', X)].
	\end{align*}
\end{lemma}
\begin{proof}
	It follows immediately from Definition~\ref{d6-4} and Lemma~\ref{l6-3}.
\end{proof}
 According to Definition~\ref{d3-5}, Lemmas~\ref{l6-3} and~\ref{l6-4}, the following two propositions can be established.

\begin{proposition}\label{p6-7}
	Let $\mathscr{C}_{1}, \mathscr{C}_{2}$ be two $L$-fuzzy $\beta$-coverings of $U$. For any
	$X\in L^{U}$,
	$\mathscr{C}_{1}, \mathscr{C}_{2}$ generate the same second type of fuzzy covering-based
	lower rough approximation of $X$ if and only if $\mathscr{E}(\mathscr{C}_{1})=\mathscr{E}(\mathscr{C}_{2})$.
\end{proposition}

\begin{proposition}\label{p6-8}
	Let $\mathscr{C}_{1}, \mathscr{C}_{2}$ be two $L$-fuzzy $\beta$-coverings of $U$. For each $X\in L^{U}$,
	$\mathscr{C}_{1}, \mathscr{C}_{2}$ generate the same second type of fuzzy covering-based upper rough approximation of $X$ if and only if $\mathscr{E}(\mathscr{C}_{1})=\mathscr{E}(\mathscr{C}_{2})$.
\end{proposition}

Proposition~\ref{p6-7} and Proposition~\ref{p6-8} give that a necessary and sufficient condition under which two $L$-fuzzy $\beta$-coverings to generate
the same second type of fuzzy covering-based lower and upper  rough approximation of an $L$-fuzzy subset is that their cores are equal, respectively. Further, we obtain the following corollary, which indicates that the second type of $L$-fuzzy $\beta$-covering-based lower and
upper rough approximation of an $L$-fuzzy subset can be determined each other.

\begin{corollary}\label{c4-2}
	Let $\mathscr{C}_{1}, \mathscr{C}_{2}$ be two $L$-fuzzy $\beta$-coverings of $U$. For each $X\in L^{U}$,
	$\mathscr{C}_{1}, \mathscr{C}_{2}$ generate the same second type of fuzzy covering-based lower rough approximation of $X$ if and only if $\mathscr{C}_{1}, \mathscr{C}_{2}$ generate the same second type of fuzzy covering-based
	upper rough approximation of $X$.
\end{corollary}

\begin{example}\label{e6-6}
	Let $\mathscr{C}$ be the $L$-fuzzy $\beta$-covering in Example~\ref{e6-4} for any $\beta\in(0,0.9]$ and
	\begin{align*}
	X&=\frac{0.6}{x_{1}}+\frac{0.2}{x_{2}}+\frac{0.5}{x_{3}}+\frac{1}{x_{4}}+\frac{0.1}{x_{5}}.
	\end{align*}
	Then we have
	\begin{align*}
	M_{\underline{\mathscr{C}}_{2}(X)}&=M_{\mathscr{C}}\triangle[\beta\triangle((M_{\mathscr{C}})^{T}\triangle M_{X})]\\
	&=\left[\begin{array}{cccc}
	0.9    & 0.7     & 1       & 0.7 \\
	0.9    & 0.9     & 0.4     & 0.9 \\
	0.2    & 0.3     & 0.9     & 0.2 \\
	0.8    & 0.6     & 0.9     & 0.6 \\
	0.3    & 0.8     & 0.9     & 0.3 \\
	\end{array}
	\right]\triangle
	\left(\beta\triangle\left(\left[\begin{array}{cccccc}
	0.9    & 0.9     & 0.2   & 0.8    & 0.3\\
	0.7    & 0.9     & 0.3   & 0.6    & 0.8 \\
	1      & 0.4     & 0.9   & 0.9    & 0.9 \\
	0.7    & 0.9     & 0.2   & 0.6    & 0.3 \\
	\end{array}
	\right]\triangle\left[
	\begin{array}{c}
	0.6    \\
	0.2    \\
	0.5   \\
	1     \\
	0.1    \\
	\end{array}
	\right]\right)\right)\\
	&=\left[
	\begin{array}{c}
	0.3    \\
	0.5   \\
	0.4  \\
	0.4   \\
	0.4  \\
	\end{array}
	\right],\\
	M_{\overline{\mathscr{C}}_{2}(X)}&=M_{\mathscr{C}}\blacktriangle[((M_{\mathscr{C}})^{T}\blacktriangle M_{X})\blacktriangle\beta]\\
	&=\left[\begin{array}{cccc}
	0.9    & 0.7     & 1       & 0.7 \\
	0.9    & 0.9     & 0.4     & 0.9 \\
	0.2    & 0.3     & 0.9     & 0.2 \\
	0.8    & 0.6     & 0.9     & 0.6 \\
	0.3    & 0.8     & 0.9     & 0.3 \\
	\end{array}
	\right]\blacktriangle
	\left(\beta\blacktriangle\left(\left[\begin{array}{cccccc}
	0.9    & 0.9     & 0.2   & 0.8    & 0.3\\
	0.7    & 0.9     & 0.3   & 0.6    & 0.8 \\
	1      & 0.4     & 0.9   & 0.9    & 0.9 \\
	0.7    & 0.9     & 0.2   & 0.6    & 0.3 \\
	\end{array}
	\right]\blacktriangle\left[
	\begin{array}{c}
	0.6    \\
	0.2    \\
	0.5    \\
	1      \\
	0.1    \\
	\end{array}
	\right]\right)\right)\\
	&=\left[
	\begin{array}{c}
	0.8    \\
	0.6  \\
	0.7    \\
	0.7    \\
	0.7   \\
	\end{array}
	\right],\\
	M_{\underline{\mathscr{E}(\mathscr{C})}_{2}(X)}&=M_{\mathscr{E}(\mathscr{C})}\triangle[\beta\triangle((M_{\mathscr{E}(\mathscr{C})})^{T}\triangle M_{X})]\\
	&=\left[\begin{array}{cccc}
	0.9    & 0.7     & 1      \\
	0.9    & 0.9     & 0.4    \\
	0.2    & 0.3     & 0.9    \\
	0.8    & 0.6     & 0.9    \\
	0.3    & 0.8     & 0.9    \\
	\end{array}
	\right]\triangle
	\left(\beta\triangle\left(\left[\begin{array}{cccccc}
	0.9    & 0.9     & 0.2   & 0.8    & 0.3\\
	0.7    & 0.9     & 0.3   & 0.6    & 0.8 \\
	1      & 0.4     & 0.9   & 0.9    & 0.9 \\
	\end{array}
	\right]\triangle\left[
	\begin{array}{c}
	0.6    \\
	0.2    \\
	0.5   \\
	1     \\
	0.1    \\
	\end{array}
	\right]\right)\right)\\
	&=\left[
	\begin{array}{c}
	0.3    \\
	0.5    \\
	0.4  \\
	0.4    \\
	0.4   \\
	\end{array}
	\right],\\
	M_{\overline{\mathscr{E}(\mathscr{C})}_{2}(X)}&=M_{\mathscr{E}(\mathscr{C})}\blacktriangle[(
(M_{\mathscr{E}(\mathscr{C})})^{T}\blacktriangle M_{X})\blacktriangle\beta]\\
	&=\left[\begin{array}{cccc}
	0.9    & 0.7     & 1      \\
	0.9    & 0.9     & 0.4    \\
	0.2    & 0.3     & 0.9    \\
	0.8    & 0.6     & 0.9    \\
	0.3    & 0.8     & 0.9    \\
	\end{array}
	\right]\blacktriangle
	\left(\left(\left[\begin{array}{cccccc}
	0.9    & 0.9     & 0.2   & 0.8    & 0.3\\
	0.7    & 0.9     & 0.3   & 0.6    & 0.8 \\
	1      & 0.4     & 0.9   & 0.9    & 0.9 \\
	\end{array}
	\right]\blacktriangle\left[
	\begin{array}{c}
	0.6    \\
	0.2    \\
	0.5   \\
	1     \\
	0.1    \\
	\end{array}
	\right]\right)\blacktriangle\beta\right)\\
	&=\left[
	\begin{array}{c}
	0.8    \\
	0.6  \\
	0.7   \\
	0.7   \\
	0.7    \\
	\end{array}
	\right].
	\end{align*}
	Therefore, we have
	\begin{align*} \underline{\mathscr{C}}_{2}(X)&=\underline{\mathscr{E}(\mathscr{C})}_{2}(X)=
	\frac{0.3}{x_{1}}+\frac{0.5}{x_{2}}+\frac{0.4}{x_{3}}
	+\frac{0.4}{x_{4}}+\frac{0.4}{x_{5}},
	\overline{\mathscr{C}}_{2}(X)=\overline{\mathscr{E}(\mathscr{C})}_{2}(X)=
	\frac{0.8}{x_{1}}+\frac{0.6}{x_{2}}+\frac{0.7}{x_{3}}
	+\frac{0.7}{x_{4}}+\frac{0.7}{x_{5}}.
	\end{align*}
\end{example}

The focus of the following propositions is on the condition under which two $L$-fuzzy $\beta$-coverings to generate the same third pair of fuzzy covering-based lower or upper rough approximation operators. By Definition~\ref{d3-6},
Lemmas~\ref{l3-2} and~\ref{l3-3}, we have the following two lemmas.
\begin{lemma}\label{l6-5}
	Let $\mathscr{C}$ be an $L$-fuzzy $\beta$-covering of $U$ and $C$ be a reducible element of $\mathscr{C}$.
	Then $R_{\mathscr{C}}=R_{\mathscr{C}-\{C\}}$.
\end{lemma}
\begin{proof}
	Amusse $\mathscr{C}=\{C, C_{1}, C_{2}, \ldots, C_{n}\}$,
	where $C, C_{i}\in L^{U}~(i=1, 2, \ldots, n)$ and
	$C$ is a reducible element of $\mathscr{C}$. Then
	there exist $C'_{1}, C'_{2}, \ldots, C'_{t}\in\mathscr{C}-\{C\}~(1<t\leq n)$
	such that $C=\underset{k=1}{\overset{t}\bigvee}C'_{k}$.
	Further, for any $x, y\in U$, we obtain that
	\begin{align*}
	R_{\mathscr{C}}(x, y)&=\underset{C'\in \mathscr{C}}{\bigwedge}[C'(x)\rightarrow C'(y)]\\
	&=\underset{i=1}{\overset{n}\bigwedge}[C_{i}(x)\rightarrow C_{i}(y)]\bigwedge[C(x)\rightarrow C(y)]\\
	&=\underset{i=1}{\overset{n}\bigwedge}[C_{i}(x)\rightarrow C_{i}(y)]\bigwedge[\underset{k=1}{\overset{t}\bigvee}C'_{k}(x)\rightarrow C(y)]\\
	&=\underset{i=1}{\overset{n}\bigwedge}[C_{i}(x)\rightarrow C_{i}(y)]\bigwedge\underset{k=1}{\overset{t}\bigwedge}[C'_{k}(x)\rightarrow C(y)]\\
	&=\underset{i=1}{\overset{n}\bigwedge}[C_{i}(x)\rightarrow C_{i}(y)]\\
	&=R_{\mathscr{C}-\{C\}}(x, y).
	\end{align*}
	Hence, $R_{\mathscr{C}}=R_{\mathscr{C}-\{C\}}$ can be followed.
\end{proof}

\begin{lemma}\label{l6-6}
	Let $\mathscr{C}$ be an $L$-fuzzy $\beta$-covering of $U$.
	Then $R_{\mathscr{C}}=R_{\mathscr{R}(\mathscr{C})}$.
\end{lemma}
\begin{proof}
	It follows immediately from Definition~\ref{d6-2} and Lemma~\ref{l6-5}.
\end{proof}

By Definition~\ref{d3-6}, Lemmas~\ref{l6-5} and~\ref{l6-6},
we have the following two propositions.
\begin{proposition}\label{p6-9}
	Let $\mathscr{C}_{1}, \mathscr{C}_{2}$ be two $L$-fuzzy $\beta$-coverings of $U$. For any $X\in L^{U}$,
	$\mathscr{C}_{1}, \mathscr{C}_{2}$ generate the same third type of fuzzy covering-based lower rough approximation of $X$ if and only if $\mathscr{R}(\mathscr{C}_{1})=\mathscr{R}(\mathscr{C}_{2})$.
\end{proposition}

\begin{proposition}\label{p6-10}
	Let $\mathscr{C}_{1}, \mathscr{C}_{2}$ be two $L$-fuzzy $\beta$-coverings of $U$. For any $X\in L^{U}$,
	$\mathscr{C}_{1}, \mathscr{C}_{2}$ generate the same third type of fuzzy covering-based upper rough approximation of $X$ if and only if $\mathscr{R}(\mathscr{C}_{1})=\mathscr{R}(\mathscr{C}_{2})$.
\end{proposition}

Proposition~\ref{p6-9} and Proposition~\ref{p6-10} give that a necessary and sufficient condition under which two $L$-fuzzy $\beta$-coverings to generate the same third type of fuzzy covering-based lower and upper rough approximation of an $L$-fuzzy subset is that their reducts are equal, respectively. Further, we obtain the following corollary, which indicates that the third type of $L$-fuzzy $\beta$-covering-based lower and
upper rough approximation of an $L$-fuzzy subset can be determined each other.

\begin{corollary}\label{c6-3}
	Let $\mathscr{C}_{1}, \mathscr{C}_{2}$ be two $L$-fuzzy $\beta$-coverings of $U$. For any
	$X\in L^{U}$,
	$\mathscr{C}_{1}, \mathscr{C}_{2}$ generate the same third type of fuzzy covering-based
	lower rough approximation of $X$ if and only if $\mathscr{C}_{1}, \mathscr{C}_{2}$ generate the same third type of fuzzy covering-based
	upper rough approximation of $X$.
\end{corollary}

\begin{example}\label{e6-7}
	Let $\mathscr{C}$ be the $L$-fuzzy $\beta$-covering in Example~\ref{e6-1} and
	\begin{align*}
	X&=\frac{0.6}{x_{1}}+\frac{0.2}{x_{2}}+\frac{0.5}{x_{3}}+\frac{1}{x_{4}}+\frac{0.1}{x_{5}}.
	\end{align*}
	Then we have
	\begin{align*}
	M_{R_{\mathscr{C}}}&=M_{\mathscr{C}}\triangle(M_{\mathscr{C}})^{T}\\
	&=\left[\begin{array}{cccc}
	0.9    & 0.7     & 1     & 0.9\\
	0.7    & 0.9     & 0.4     & 0.9 \\
	0.2    & 0.3     & 0.9     & 0.3 \\
	0.8    & 0.6     & 0.9     & 0.8 \\
	0.3    & 0.8     & 0.9     & 0.8 \\
	\end{array}
	\right]\triangle
	\left[\begin{array}{cccccc}
	0.9    & 0.7     & 0.2   & 0.8    & 0.3\\
	0.7    & 0.9     & 0.3   & 0.6    & 0.8 \\
	1      & 0.4     & 0.9   & 0.9    & 0.9 \\
	0.9    & 0.9     & 0.3   & 0.8    & 0.8 \\
	\end{array}
	\right]\\
	&=\left[\begin{array}{cccccc}
	1      & 0.4   & 0.3   & 0.9  & 0.4 \\
	0.8    & 1     & 0.4   & 0.7  & 0.6 \\
	1      & 0.5   & 1     & 1    & 1   \\
	1      & 0.5   & 0.4   & 1    & 0.5 \\
	0.9    & 0.5   & 0.5   & 0.8  & 1   \\
	\end{array}
	\right],\\
	M_{R_{\mathscr{R}(\mathscr{C}})}&=M_{\mathscr{R}(\mathscr{C})}\triangle(M_{\mathscr{R}(\mathscr{C})})^{T}\\
	&=\left[\begin{array}{cccc}
	0.9    & 0.7     & 1     \\
	0.7    & 0.9     & 0.4   \\
	0.2    & 0.3     & 0.9   \\
	0.8    & 0.6     & 0.9   \\
	0.3    & 0.8     & 0.9   \\
	\end{array}
	\right]\triangle
	\left[\begin{array}{cccccc}
	0.9    & 0.7     & 0.2   & 0.8    & 0.3 \\
	0.7    & 0.9     & 0.3   & 0.6    & 0.8 \\
	1      & 0.4     & 0.9   & 0.9    & 0.9 \\
	\end{array}
	\right]\\
	&=\left[\begin{array}{cccccc}
	1      & 0.4   & 0.3   & 0.9  & 0.4 \\
	0.8    & 1     & 0.4   & 0.7  & 0.6 \\
	1      & 0.5   & 1     & 1    & 1   \\
	1      & 0.5   & 0.4   & 1    & 0.5 \\
	0.9    & 0.5   & 0.5   & 0.8  & 1   \\
	\end{array}
	\right],\\
	M_{\underline{\mathscr{C}}_{3}(X)}&=[M_{\mathscr{C}}\triangle(M_{\mathscr{C}})^{T}]^{T}\triangle M_{X}\\
	&=(M_{R_{\mathscr{C}}})^{T}\triangle M_{X}\\
	&=(M_{R_{\mathscr{R}(\mathscr{C})}})^{T}\triangle M_{X}\\
	&=M_{\underline{\mathscr{R}(\mathscr{C})}_{3}(X)}\\
	&=\left[\begin{array}{cccccc}
	1      & 0.4   & 0.3   & 0.9  & 0.4 \\
	0.8    & 1     & 0.4   & 0.7  & 0.6 \\
	1      & 0.5   & 1     & 1    & 1   \\
	1      & 0.5   & 0.4   & 1    & 0.5 \\
	0.9    & 0.5   & 0.5   & 0.8  & 1   \\
	\end{array}
	\right]\triangle
	\left[
	\begin{array}{c}
	0.6    \\
	0.2    \\
	0.5   \\
	1     \\
	0.1    \\
	\end{array}
	\right]\\
	&=\left[
	\begin{array}{c}
	0.6    \\
	0.2    \\
	0.1  \\
	0.6     \\
	0.1    \\
	\end{array}
	\right],\\
	M_{\overline{\mathscr{C}}_{3}(X)}&=[M_{\mathscr{C}}\triangle(M_{\mathscr{C}})^{T}]^{T}\blacktriangle M_{X}\\
	&=(M_{R_{\mathscr{C}}})^{T}\blacktriangle M_{X}\\
	&=(M_{R_{\mathscr{R}(\mathscr{C})}})^{T}\blacktriangle M_{X}\\
	&=M_{\overline{\mathscr{R}(\mathscr{C})}_{3}(X)}\\
	&=\left[\begin{array}{cccccc}
	1      & 0.4   & 0.3   & 0.9  & 0.4 \\
	0.8    & 1     & 0.4   & 0.7  & 0.6 \\
	1      & 0.5   & 1     & 1    & 1   \\
	1      & 0.5   & 0.4   & 1    & 0.5 \\
	0.9    & 0.5   & 0.5   & 0.8  & 1   \\
	\end{array}
	\right]\blacktriangle
	\left[
	\begin{array}{c}
	0.6    \\
	0.2    \\
	0.5   \\
	1     \\
	0.1    \\
	\end{array}
	\right]\\
	&=\left[
	\begin{array}{c}
	0.9    \\
	0.7    \\
	1  \\
	1     \\
	0.8    \\
	\end{array}
	\right].
	\end{align*}
	Therefore, we have
	\begin{align*} \underline{\mathscr{C}}_{3}(X)&=\underline{\mathscr{R}(\mathscr{C})}_{3}(X)=
	\frac{0.6}{x_{1}}+\frac{0.2}{x_{2}}+\frac{0.1}{x_{3}}
	+\frac{0.6}{x_{4}}+\frac{0.1}{x_{5}},
	\overline{\mathscr{C}}_{3}(X)=\overline{\mathscr{R}(\mathscr{C})}_{3}(X)=
	\frac{0.9}{x_{1}}+\frac{0.7}{x_{2}}+\frac{1}{x_{3}}
	+\frac{1}{x_{4}}+\frac{0.8}{x_{5}}.
	\end{align*}
\end{example}
\section{Conclusions}\label{section7}
     Fuzzy $\beta$-covering can be seen as a bridge between covering-based rough set theory and fuzzy set theory. In this paper, we further construct three types of $L$-fuzzy $\beta$-covering-based rough set models by introducing the notions such as $\beta$-degree of intersection and $\beta$-subsethood degree. Main conclusions in this paper and continuous work to do are listed as follows.
     \begin{enumerate}[(1)]
       \item
        We showed an axiom set for each of three pairs of $L$-fuzzy $\beta$-covering-based rough approximation operators and further verify their irrelevance. In addition, with the idea of \cite{bvelohlavek2002fuzzy} and \cite{belohlavek2011closure}, some interesting dual theorems can be obtained between $L$-fuzzy $\beta$-covering-based lower and upper rough approximation operators.
       \item
       We introduced the matrix representations of the three pairs of $L$-fuzzy $\beta$-covering-based lower and upper rough approximations, which make calculations more valid through operations on matrices.
       \item
       The interdependency of the three pairs of $L$-fuzzy $\beta$-covering-based rough approximation operators is proposed. Using the concept of reducible element (independent element) of the fuzzy $\beta$-covering, the necessary and sufficient conditions under which the two $L$-fuzzy $\beta$-covering can generate the same first and third (second) $L$-fuzzy $\beta$-covering-based rough approximation operations are given.
        \item
       The axiomatic characterizations of the second and third pairs of $L$-fuzzy $\beta$-covering-based rough approximation operators are based on complete Heyting algebra, further weakening the conditions and obtaining more general conclusions.
     \end{enumerate}

\section*{Acknowledgements}
    This research was supported by the National Natural Science Foundation of China (Grant nos. 12101500, 62166037 and 11901465), the Chinese Universities Scientific Fund (Grant no. 2452018054), the Science and Technology Program of Gansu Province (20JR10RA101), the China Postdoctoral Science Foundation (2021M692561), the Scientific Research Fund for Young Teachers of Northwest Normal University (NWNU-LKQN-18-28) and the Doctoral Research Fund of Northwest Normal University (6014/0002020202).

\section*{Reference}

%

\end{document}